\pdfoutput=1
\def\ArxivPreprint{1}

\documentclass{article}
\usepackage{iclr2027_conference,times} %

\def\ICLRFullPaperBuild{1}
\newif\ifincludeappendix
\ifdefined\ICLRFullPaperBuild
  \includeappendixtrue
\else
  \includeappendixfalse
\fi

\usepackage{amsmath,amssymb,amsthm}

\usepackage{booktabs}
\usepackage{array}
\usepackage{enumitem}
\usepackage{graphicx}
\usepackage{placeins}
\newcolumntype{L}[1]{>{\raggedright\arraybackslash}p{#1}}
\usepackage{tikz}
\usetikzlibrary{arrows.meta,decorations.pathreplacing}
\usepackage{wrapfig}
\usepackage{needspace}

\usepackage{microtype} %

\usepackage{hyperref}
\usepackage{url}

\definecolor{linknavy}{HTML}{1A3D6D}
\definecolor{keyblue}{HTML}{0072B2}
\definecolor{keyorange}{HTML}{E69F00}
\definecolor{keypurple}{HTML}{CC79A7}
\hypersetup{
  colorlinks=true,
  linkcolor=linknavy,
  citecolor=linknavy,
  urlcolor=linknavy,
}

\newcommand{\RR}{\mathbb{R}}
\newcommand{\EE}{\mathbb{E}}

\newcommand{\dconv}{d_{\mathrm{conv}}}

\newtheorem{theorem}{Theorem}
\newtheorem{proposition}{Proposition}
\newtheorem{lemma}{Lemma}
\newtheorem{corollary}{Corollary}
\newtheorem{remark}{Remark}

\title{A Comparative Analysis of Attention versus State-Space Models for In-Context Learning}

\author{Anonymous authors}

\ifdefined\ArxivPreprint
  \iclrfinalcopy
  \author{Enes Arda \\ The Ohio State University \\ \texttt{arda.2@osu.edu} \And
    Semih Cayci \\ RWTH Aachen University \\ \texttt{cayci@mathc.rwth-aachen.de} \And
    Atilla Eryilmaz \\ The Ohio State University \\ \texttt{eryilmaz.2@osu.edu}}
\fi

\begin{document}

\ifdefined\ArxivPreprint\edef\savedtabcolsep{\the\tabcolsep}\setlength{\tabcolsep}{4.5pt}\fi
\maketitle
\ifdefined\ArxivPreprint\setlength{\tabcolsep}{\savedtabcolsep}\lhead{Preprint}\vspace{-11pt}\fi

\begin{abstract}
Transformers and state-space models (SSMs) are two prominent sequential learning architectures, yet their comparison remains largely empirical and existing theoretical analyses are typically task-specific or architecturally restricted.
In this paper, we develop \emph{belief geometry}, a unified analytical framework for comparing the representational capabilities of broad classes of attention and SSMs.
Starting from a generalized formulation of in-context linear regression and using cumulative Bayes regret as our measure, we abstract three capabilities required by many sequential learning problems in our belief geometry: evidence assembly, belief maintenance, and addressing. 
We then study three cases of our formulation that isolate these capabilities and yield sharp architectural lessons:
For belief maintenance, SSMs attain the optimal regret over stationary aggregation kernels; for positional assembly, SSMs have a memory advantage; and for content addressing, softmax attention has an exponential width advantage over sigmoid-selective SSMs.
Experiments with LLaMA-type Transformers and Mamba-2 show that these architectural insights extend beyond our analytically tractable classes and linear-regression testbed.
\end{abstract}
\vspace{-0.325em}
\section{Introduction}\label{sec:intro}
\vspace{-0.325em}
Transformers \citep{vaswaniAttentionAllYou2017} and state-space models (SSMs) \citep{guMambaLinearTimeSequence2024, beckXLSTMExtendedLong2024, yangGatedLinearAttention2024} have emerged as two competitive architecture families for sequential learning. They organize history in fundamentally different ways: Transformers retain past tokens within a context window for direct access, whereas SSMs compress the past into a recurrent state. A principled understanding of when either strategy is advantageous, and which architectural mechanisms give rise to those advantages, has therefore become essential for designing sequential learning architectures.

Comparisons between Transformers and SSMs, however, remain largely empirical \citep[e.g.,][]{aroraZoologyMeasuringImproving2023, parkCanMambaLearn2024}. Existing theoretical work typically studies specific tasks \citep[e.g.,][]{sanfordRepresentationalStrengthsLimitations2023, sanfordTransformersParallelComputation2024, jelassiRepeatMeTransformers2024, wenRNNsAreNot2024}, analyzes restricted forms of these architectural classes \citep[e.g.,][]{coleInContextLearningLinear2025, zhangTrainedTransformersLearn2023, tianLearningRecurrentState2026}, or unifies their mathematical descriptions without explaining when their architectural differences matter \citep[e.g.,][]{sieberUnderstandingDifferencesFoundation2024}. We therefore lack a general framework for identifying which architectural capabilities a sequential learning task demands, how different architectures supply them, and at what resource cost.

\vspace{-0.2em}
\paragraph{Contributions.}
In this paper, we generalize in-context linear regression (ICLR) through routing and evaluate architectures by their cumulative Bayes regret rather than by their loss at a fixed prediction step. This extension reveals three capabilities required by many sequential learning problems: (i) \emph{assembling evidence} from inputs, (ii) \emph{maintaining a belief} by aggregating that evidence, and (iii) \emph{addressing} the relevant part of history to which that belief is applied. We formalize these capabilities in our \emph{belief geometry}, which enables a clean, unified representational analysis of architecturally rich and practically relevant classes of attention (Transformers' backbone) and SSMs (see Remark~\ref{rem:architecture_scope}). 

Our analysis proceeds through controlled instantiations of the generalized ICLR setup: in each case, a different capability becomes the bottleneck, and we analyze it through its corresponding element in belief geometry.
Starting from the simplest versions of each family, we use linear regret lower bounds to identify the failures and sublinear regret constructions to show which architectural components overcome them. The resulting sharp comparisons of achievable regret and of the depth, width, and memory required by each family explain \emph{where} and \emph{why} one family prevails over the other and yield the following architectural lessons:
\begingroup
\setlength{\leftmargini}{0.9em}
\begin{itemize}

\item \textbf{Belief maintenance: SSMs attain the stationary optimum.}
Case~1 isolates belief maintenance, captured in belief geometry by the \emph{aggregation kernel}. Our key observation is that one-layer linear attention and one-layer fixed-transition SSMs can at best realize stationary kernels. Within this restriction, their best achievable regrets differ: the SSM's exponential kernel attains the stationary optimum, whereas attention's uniform kernel cannot. This separation predicts which kernel practical models in each family favor---an inductive-bias difference that we validate empirically.

\item \textbf{Positional assembly: SSMs have a memory advantage.}
Case~2.a introduces positional routing and makes evidence assembly the bottleneck, captured in belief geometry by \emph{belief alignment}. The Case~1 architectures incur linear regret at any width or context length because retaining history cannot substitute for assembly. Attention therefore needs depth and uses its context length for both reach and belief maintenance, whereas the SSM separates them through convolution and recurrence, attaining the same regret rate with \emph{strictly less memory}.

\item \textbf{Content addressing: attention has an exponential width advantage.}
Case~2.b introduces content routing and makes addressing the bottleneck, quantified in belief geometry by \emph{addressing capacity}. The Case~2.a mechanisms cannot sharply address the required historical input: at any fixed depth, linear attention and affine-gated SSMs incur linear regret regardless of width or context length. The families escape this floor differently: attention uses softmax to select among retained tokens, whereas the SSM uses sigmoid-selective transitions to preserve candidates separately in its recurrent state and requires \emph{exponentially larger width} for the same regret rate.

\end{itemize}
\endgroup

Finally, experiments with LLaMA-type Transformers and Mamba-2 show that these lessons extend beyond our analytically tractable classes and linear-regression testbed. Appendix~\ref{app:experiments_analytical} additionally provides direct experimental validation of our theoretical results, which are summarized in Table~\ref{tab:results_summary}.

\paragraph{Related work.}
\label{par:intro_related_work}
Within the theory of in-context learning, work closest to our setup characterizes which regression algorithms attention and recurrent models can realize in ICLR, typically studying one family at a time and evaluating loss at a fixed prediction step \citep{oswaldTransformersLearnIncontext2023,mahankaliOneStepGradient2023,baiTransformersStatisticiansProvable2023,tianLearningRecurrentState2026,coleInContextLearningLinear2025}.
Our belief-maintenance analysis (Case~1, \S\ref{sec:case1_lb}) instead compares \emph{both families} under \emph{cumulative Bayes regret}. 
For positional assembly (Case~2.a, \S\ref{sec:case2a}), prior work studies attention depth \citep{sanfordOnelayerTransformersFail2024,ekboteWhatOneCannot2025} and convolution in Mamba \citep{bondaschiMarkovLaplaceHow2026}; we prove \emph{class-level regret floors} and an \emph{SSM memory advantage} at the same sublinear regret rate.
The literature closest to content addressing (Case~2.b, \S\ref{sec:case2b}) establishes attention--recurrence separations for associative recall, induction heads, and copying under computational, memory, or statistical criteria \citep{sanfordRepresentationalStrengthsLimitations2023,sanfordTransformersParallelComputation2024,jelassiRepeatMeTransformers2024,wenRNNsAreNot2024,mousavi-hosseiniWhenTransformersOutperform2025}.
These tasks are largely standalone and discrete, and their lower bounds measure state size in bits under a finite-precision assumption. Our addressing problem is instead embedded in noisy Bayesian inference alongside evidence assembly and belief maintenance; its lower bound concerns \emph{real-valued state dimension} without a precision restriction.
At a broader theoretical level, prior work unifies the mathematical representation of attention and SSMs \citep{sieberUnderstandingDifferencesFoundation2024,wangTesttimeRegressionUnifying2025}. Our belief geometry instead \emph{unifies their analysis}, translating architectural differences into task-dependent regret and resource comparisons. 
Appendix~\ref{app:related} compares our results with these works in greater detail and surveys the broader literature.

\section{Setup}\label{sec:setup}

\paragraph{Data-generating process.}
We generalize Bayesian in-context linear regression (ICLR) by introducing routing into the setup studied in prior work \citep{gargWhatCanTransformers2023,akyurekWhatLearningAlgorithm2023,
oswaldTransformersLearnIncontext2023,
zhangTrainedTransformersLearn2023}.
A partial injective \emph{routing map} $\rho:\mathbb N\rightharpoonup\mathbb N$ specifies which regressor generates the response $y_s \in \RR$ for each time step $s$ in the domain of $\rho$, subject to the causal constraint $\rho(s) \leq s$. Whenever $\rho(s)$ is defined, step $s$ is called \emph{matched}, and at that step the data stream obeys
\[
  y_s=x_{\rho(s)}^\top\theta+\epsilon_s,\qquad
  \theta\sim\mathcal N(0,\Sigma_\theta),\quad
  x_s\stackrel{\mathrm{iid}}{\sim}\mathcal N(0,\Sigma_x),\quad
  \epsilon_s\stackrel{\mathrm{iid}}{\sim}\mathcal N(0,\sigma^2),
\]
where $\theta,x_s\in\RR^p$ for a fixed $p\in\mathbb N$, $\Sigma_\theta,\Sigma_x\succ0$, $\sigma^2>0$, and $\theta,(x_s)_s,(\epsilon_s)_s$ are mutually independent.
All remaining steps are unmatched, with $y_s=0$, and therefore carry no information about $\theta$. Throughout, expressions involving $\rho(s)$ are understood only for matched $s$. 
At step $t$, the task is to predict $y_{t+1}$ from the history $z_{1:t}$, where $z_s = (x_{s+1}, y_s, x_s) \in \RR^{2p+1}$.
So at every step, the history contains the regressor \(x_{\rho(t+1)}\) that generates \(y_{t+1}\), but not the response itself. This lets us evaluate cumulative Bayes regret rather than only the loss at the end of the context.

\paragraph{Bayes regret.} Let $\mathcal{H}_t := L^2(\sigma(z_{1:t}))$, with its usual inner product $\langle\cdot,\cdot\rangle_t$ and norm $\|\cdot\|_t$. A fixed predictor $g$ maps each history $z_{1:t}$, of any length, to $g_t := g(z_{1:t}) \in \mathcal{H}_t$ and incurs the squared loss $\ell_t(g) := \EE[(g_t - y_{t+1})^2]$. The Bayes predictor $g^{\mathrm{Bayes}}$ minimizes this loss with \mbox{$g_t^{\mathrm{Bayes}} := \EE[y_{t+1} \mid z_{1:t}]$}. Taking it as the comparator, the cumulative regret of a predictor $g$, and of a class $\mathcal{F}$, over horizon $T$ is
\[
  \mathrm{Reg}_T(g) \,:=\, \sum_{t=1}^{T} \big[\ell_t(g) - \ell_t(g^{\mathrm{Bayes}})\big]
  \,=\, \sum_{t=1}^{T}\|g_t-g_t^{\mathrm{Bayes}}\|_t^2,
  \qquad
  \mathrm{Reg}_T(\mathcal{F}) \,:=\, \inf_{g \in \mathcal{F}} \mathrm{Reg}_T(g).
\]
Unlike fixed-context ICLR formulations that evaluate a single loss $\ell_T(g)$ after observing the entire context \citep{baiTransformersStatisticiansProvable2023, mahankaliOneStepGradient2023, ahnTransformersLearnImplement2023, wuHowManyPretraining2024}, our online formulation evaluates the \emph{cumulative excess loss} as the context grows. This captures an essential aspect of sequential learning that terminal-only evaluation misses: performance throughout the entire trajectory, not merely at its endpoint.

\paragraph{Belief geometry.} 
Our belief geometry is motivated by the following structure of the Bayes predictor under routing: it factorizes as \(g_t^{\mathrm{Bayes}}=x_{\rho(t+1)}^\top\hat\theta_t^{\mathrm{Bayes}}\), where \(\hat\theta_t^{\mathrm{Bayes}}:=\EE[\theta\mid z_{1:t}]\) is the Bayes belief, which in turn can be written as an aggregation of evidence atoms through a kernel:
\begin{equation}\label{eq:bayes_kernel}
  \hat\theta_t^{\mathrm{Bayes}}
  =
  \sum_{s\leq t}K_{t,s}^{\mathrm{Bayes}}x_{\rho(s)}y_s,
  \qquad
  K_{t,s}^{\mathrm{Bayes}}
  :=
  \bigg(
    \sigma^2\Sigma_\theta^{-1}
    +\sum_{r\leq t}x_{\rho(r)}x_{\rho(r)}^\top
  \bigg)^{-1}.
\end{equation}
Together, the factorization and kernel form reveal three operations required to solve many sequential-learning problems:
\emph{evidence assembly}, \emph{belief maintenance}, and \emph{addressing}.
The Bayes predictor performs all three: 
it assembles evidence atoms $x_{\rho(s)}y_s$, maintains a belief by weighting these atoms with the aggregation kernel $K_{t,s}^{\mathrm{Bayes}}$, and addresses the relevant part of history, $x_{\rho(t+1)}$, to which that belief is applied. 
Appendix~\ref{app:conjugate_belief_geometry} shows how these operations extend to broader sequential learning problems, including conjugate and nonconjugate setups, and explains how our regret-analysis framework carries over.

To relate an arbitrary predictor \(g\) to the Bayes structure and measure how well it performs these operations, let $\mathcal G_t^-$ denote the history with $x_{\rho(t+1)}$ held out and define the \emph{belief subspace}\hfill\break
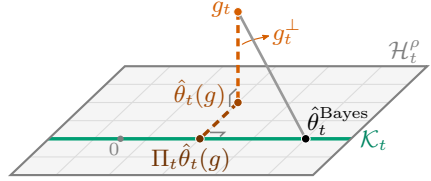
\begin{wrapfigure}[11]{r}{0.42\textwidth}
  \vspace{-1.4\baselineskip}
  \centering
\begingroup
\definecolor{bgBayes}{HTML}{000000}
\definecolor{bgModel}{HTML}{D55E00}
\definecolor{bgKernel}{HTML}{009E73}
\definecolor{bgInk}{HTML}{333333}
\definecolor{bgPlaneFill}{HTML}{F5F5F5}
\definecolor{bgPlaneEdge}{HTML}{808080}
\definecolor{bgGrid}{HTML}{DCDCDC}
\begin{tikzpicture}[
  x={(1cm,0cm)}, y={(0.42cm,0.40cm)}, z={(0cm,1cm)},
  lab/.style={font=\footnotesize,inner sep=1.4pt},
  plab/.style={font=\footnotesize,inner sep=1.4pt},
  note/.style={font=\scriptsize,inner sep=1.2pt},
  dot/.style={circle,inner sep=0pt,minimum size=3.2pt,fill=#1,draw=white,line width=.35pt},
  sdot/.style={circle,inner sep=0pt,minimum size=3.2pt,fill=#1,draw=white,line width=.35pt},
  ra/.style={line width=.6pt,draw=black!50},
  co/.style={-{Stealth[length=3.1pt,width=2.3pt,inset=0.5pt,round]},line width=.42pt,line cap=round,shorten <=0.6pt,shorten >=0.8pt},
]
  \fill[bgPlaneFill] (-0.5,-1.35,0) -- (3.5,-1.35,0) -- (3.5,2.25,0) -- (-0.5,2.25,0) -- cycle;
  \begin{scope}
    \clip (-0.5,-1.35,0) -- (3.5,-1.35,0) -- (3.5,2.25,0) -- (-0.5,2.25,0) -- cycle;
    \foreach \gx in {0.1,0.8,1.5,2.2,2.9}
      \draw[bgGrid,line width=.4pt] (\gx,-1.35,0) -- (\gx,2.25,0);
    \foreach \gy in {-0.75,0.45,1.05,1.65}
      \draw[bgGrid,line width=.4pt] (-1.0,\gy,0) -- (4.7,\gy,0);
    \draw[bgKernel,line width=1.3pt] (-1.1,-0.15,0) -- (4.9,-0.15,0);
  \end{scope}
  \draw[bgPlaneEdge,line width=.8pt]
    (-0.5,-1.35,0) -- (3.5,-1.35,0) -- (3.5,2.25,0) -- (-0.5,2.25,0) -- cycle;
  \draw[bgModel,line width=1.3pt,dash pattern=on 3pt off 2pt] (1.5,1.05,1.2) -- (1.5,1.05,0);
  \node[dot=bgBayes]  at (2.9,-0.15,0) {};
  \node[sdot=bgModel!75!black] at (1.5,1.05,0) {};
  \node[dot=bgModel]  at (1.5,1.05,1.2) {};
  \node[lab,text=bgPlaneEdge,anchor=south east] at (4.42,0,0.88) {$\mathcal{H}_t^{\rho}$};
  \node[lab,text=bgKernel!85!black,anchor=west] at (3.50,0,-0.08) {$\mathcal{K}_t$};
  \node[plab,text=bgBayes,anchor=south west] at (2.8,0,-0.07) {$\hat\theta_t^{\mathrm{Bayes}}$};
  \node[plab,text=bgModel,anchor=east] at (1.92,0,1.66) {$g_t$};

\draw[co,bgModel]
  (1.98,0,1.18) to[out=30,in=195] (2.38,0,1.31);
\draw[black!40,line width=1.0pt] (1.5,1.05,1.2) -- (2.9,-0.15,0);   %
\draw[bgModel!75!black,line width=1.3pt,dash pattern=on 3pt off 2pt] (1.5,1.05,0) -- (1.5,-0.15,0);
\draw[ra] (1.5,0.79,0.045) -- (1.5,0.79,0.20) -- (1.5,1.005,0.20);   %
\draw[ra] (1.735,-0.075,0) -- (1.735,0.06,0) -- (1.545,0.06,0);       %
\node[sdot=bgModel!52!black] at (1.5,-0.15,0) {};
\node[sdot=bgModel!75!black] at (1.5,1.05,0) {};
\node[dot=bgModel] at (1.5,1.05,1.2) {};
\node[dot=bgBayes] at (2.9,-0.15,0) {};
\node[circle,inner sep=0pt,minimum size=2.4pt,fill=black!50] at (0.45,-0.15,0) {};
\node[note,text=black!50,anchor=north east] at (0.42,0,-0.07) {$0$};
\node[plab,text=bgModel!52!black,anchor=north east] at (1.9,0,-0.06) {$\Pi_t\hat\theta_t(g)$};
\node[plab,text=bgModel!75!black,anchor=east] at (1.84,0,0.54) {$\hat\theta_t(g)$};
\node[plab,text=bgModel,anchor=west] at (2.34,0,1.36) {$g_t^\perp$};
\end{tikzpicture}
\endgroup
   \vspace{-0.5\baselineskip}
  \caption{Belief geometry in $\mathcal H_t$.}
  \label{fig:belief_geometry}
\end{wrapfigure}
\vspace{-10pt}
\[
  \mathcal H_t^\rho
  :=
  \left\{x_{\rho(t+1)}^\top a:
  a\in L^2(\mathcal G_t^-;\RR^p)\right\}
  \subseteq\mathcal H_t.
\]
Projecting $g_t$ onto this subspace gives
\[g_t=x_{\rho(t+1)}^\top\hat\theta_t(g)+g_t^\perp,\]
which separates $g$'s maintained belief $\hat\theta_t(g)$ from the orthogonal component $g_t^\perp\perp\mathcal H_t^\rho$, which cannot arise from applying a belief to the addressed regressor (Figure~\ref{fig:belief_geometry}). 

Within $\mathcal H_t^\rho$, 
the Bayes kernel form motivates a further refinement: 
the subspace of beliefs expressible as weighted sums of evidence atoms.
Let $\mathcal X_t^-\subseteq\mathcal G_t^-$ be the $\sigma$-field generated by the regressors in $\mathcal G_t^-$, and define the \emph{kernel-form subspace}
\[
  \mathcal K_t
  :=
  \overline{\Bigl\{
    \textstyle\sum_{s\leq t}K_{t,s}x_{\rho(s)}y_s
    :
    K_{t,s}\in L^\infty(\mathcal X_t^-;\RR^{p\times p})
  \Bigr\}}
  \subseteq L^2(\mathcal G_t^-;\RR^p).
\]
Let $\Pi_t:L^2(\mathcal G_t^-;\RR^p)\to\mathcal K_t$ denote orthogonal projection under the inner product $\langle a,b\rangle_{\Sigma_x}:=\EE[a^\top\Sigma_xb]$ and the induced norm $\|a\|_{\Sigma_x}^2:=\EE[a^\top\Sigma_xa]$. Since Equation~\eqref{eq:bayes_kernel} shows that $\hat\theta_t^{\mathrm{Bayes}}\in\mathcal K_t$, the Pythagorean decomposition visualized in Figure~\ref{fig:belief_geometry} gives
\begin{equation}\label{eq:belief_pythagoras}
  \|g_t-g_t^{\mathrm{Bayes}}\|_t^2
  =
  \underbrace{\|g_t^\perp\|_t^2}_{\text{outside belief form}}
  +\underbrace{\|\hat\theta_t(g)-\Pi_t\hat\theta_t(g)\|_{\Sigma_x}^2}_{\text{outside kernel form}}
  +\underbrace{\|\Pi_t\hat\theta_t(g)-\hat\theta_t^{\mathrm{Bayes}}\|_{\Sigma_x}^2}_{\text{kernel mismatch}}.
\end{equation}
When an architecture applies a kernel-form belief to the addressed regressor, the first two terms vanish and only kernel mismatch remains.
Case~1 (\S\ref{sec:case1_lb}) begins directly in this regime: trivial routing isolates belief maintenance, and our projection arguments reduce the analysis to the optimal aggregation kernels each architecture can realize.
Case~2 (\S\ref{sec:case2_lb}) then introduces nontrivial routing, bringing the first two terms into play and motivating new elements of belief geometry that capture failures of evidence assembly and addressing.
Before turning to these cases, we define the attention and SSM classes.

\subsection{Architecture classes and their components}\label{sec:architectures}

We study rich classes of attention and SSMs whose components either appear in practice or serve as minimal tractable surrogates (see Appendix~\ref{app:arch} for practical correspondence). 
Throughout, \(D\), \(H\), \(d\), and \(d_e\) denote depth, head count, per-head width, and residual-stream width, respectively; \(d\) is attention's query/key/value width and an SSM head's recurrent-state width. Layer and head indices are \(r=1,\ldots,D\) and \(b=1,\ldots,H\), respectively. Because \(d_e\) primarily follows the input-token dimension and is common to both families, we suppress it from the class notation. Lower bounds are uniform over it, while upper bound constructions assume sufficient width. The displayed arguments in class notations fix architectural hyperparameters, and each class ranges over all remaining parameters, making our analysis representational. All projections are affine, with biases suppressed from notation. Both families initially embed $z_s$ as \(e_s^{(0)}=W_{\mathrm{emb}}z_s\in\RR^{d_e}\) and finally read out with $W_{\mathrm{pred}}\in\RR^{1\times d_e}$. 
\paragraph{Attention $(\mathcal F_{\mathrm{att}}(D,d,H,L))$.}
At time $t$, attention accesses the causal window $J_t:=\{\max(1,t-L+1),\ldots,t\}$. Its parameters at layer $r$ and head $b$ are $W_Q^{(r,b)},W_K^{(r,b)},W_V^{(r,b)}\in\RR^{d\times d_e}$, $W_O^{(r,b)}\in\RR^{d_e\times d}$, and $p_\ell^{(r,b)}\in\RR$ for $0\leq\ell<L$. It computes
\[
\begin{gathered}
  s_{t,s}^{(r,b)}=\left\langle W_Q^{(r,b)}e_t^{(r-1)},W_K^{(r,b)}e_s^{(r-1)}\right\rangle+p_{t-s}^{(r,b)},\quad s\in J_t,
  \qquad a_t^{(r,b)}\in\left\{s_t^{(r,b)},\operatorname{softmax}\bigl(s_t^{(r,b)}\bigr)\right\},\\[-1pt]
  o_t^{(r,b)}=\sum_{s\in J_t}a_{t,s}^{(r,b)}W_V^{(r,b)}e_s^{(r-1)},
  \qquad e_t^{(r)}=e_t^{(r-1)}+\sum_b W_O^{(r,b)}o_t^{(r,b)},
  \qquad g_t=W_{\mathrm{pred}}e_t^{(D)}.
\end{gathered}
\]
Choosing $a_t^{(r,b)}=s_t^{(r,b)}$ gives linear attention, whereas choosing $a_t^{(r,b)}=\operatorname{softmax}(s_t^{(r,b)})$ gives softmax attention; $\mathcal F_{\mathrm{att}}$ uses the former throughout, while $\mathcal F_{\mathrm{att}}^{\mathrm{sm}}$ permits either. Depth composes successive attention operations, $L$ bounds the accessible history, and $p_{t-s}$ is a tractable additive analogue of the relative-lag dependence induced by RoPE in LLaMA-type Transformers.

\paragraph{State-space models $(\mathcal F_{\mathrm{SSM}}(D,d,H,d_{\mathrm{conv}}))$.}
At time $t$, each SSM layer causally convolves its input, uses the result to write to a recurrent state, and reads that state back to the residual stream. 
At layer $r$, head $b$ has recurrent state $h_t^{(r,b)}\in\RR^d$, convolution coefficients $C_j^{(r,b)}\in\RR^{d_e}$ ($0\leq j<d_{\mathrm{conv}}$), projections $W_K^{(r,b)},W_V^{(r,b)},W_Q^{(r,b)}\in\RR^{d\times d_e}$ and $W_O^{(r,b)}\in\RR^{d_e\times d}$, and transition $\Phi_t^{(r,b)}\in\RR^{d\times d}$. Initial states and token padding are zero. With $\odot$ denoting coordinatewise multiplication, it computes:
\[
\resizebox{\textwidth}{!}{$
\begin{gathered}
  \widetilde e_t^{(r,b)}
  =\sum_{j=0}^{d_{\mathrm{conv}}-1}C_j^{(r,b)}\odot e_{t-j}^{(r-1)},\;
  u_t^{(r,b)}=(W_K^{(r,b)}\widetilde e_t^{(r,b)})\odot(W_V^{(r,b)}\widetilde e_t^{(r,b)}),\;
  h_t^{(r,b)}=\Phi_t^{(r,b)}h_{t-1}^{(r,b)}+u_t^{(r,b)},\\[-1pt]
  o_t^{(r,b)}=(W_Q^{(r,b)}\widetilde e_t^{(r,b)})\odot h_t^{(r,b)},\qquad
  e_t^{(r)}=e_t^{(r-1)}+\sum_b W_O^{(r,b)}o_t^{(r,b)},\qquad
  g_t=W_{\mathrm{pred}}e_t^{(D)}.
\end{gathered}
$}
\]
This coordinatewise formulation includes outer-product writes and matrix--vector reads without extra recurrent state.
The transition is fixed when $\Phi_t=\Phi$ and selective when it depends on $\widetilde e_t$. We allow
\[
  \Phi_t\in\Big\{
    \underbrace{\Phi}_{\text{fixed}},\;
    \underbrace{\alpha_tI}_{\text{scalar gate}},\;
    \underbrace{\operatorname{diag}(\boldsymbol\alpha_t)}_{\text{diagonal gate}},\;
    \underbrace{I-\beta_t k_tk_t^\top}_{\text{delta rule}},\;
    \underbrace{\operatorname{diag}(\boldsymbol\alpha_t)+r_t^{\mathrm{out}}(r_t^{\mathrm{in}})^\top}_{\text{generalized delta}}
  \Big\}.
\]
Here $\Phi$ is learned but input-independent. The scalar gates $\alpha_t,\beta_t$ and the coordinates of $\boldsymbol\alpha_t$ are sigmoids of affine functions of $\widetilde e_t$, while $k_t,r_t^{\mathrm{out}},r_t^{\mathrm{in}}$ are affine in $\widetilde e_t$. \(\mathcal F_{\mathrm{SSM}}\) uses fixed transitions throughout; \(\mathcal F_{\mathrm{SSM}}^{\mathrm{sel}}\) permits any displayed transition in its first layer and uses fixed transitions thereafter; and \(\mathcal F_{\mathrm{SSM}}^{\mathrm{aff}}\) permits any displayed transition in every layer with each sigmoid replaced by the identity.

\paragraph{Practical counterparts.}
For attention, our main-text experiments use LLaMA-type Transformers \citep{touvronLLaMAOpenEfficient2023} with causal softmax attention, RoPE, and SwiGLU MLPs. For SSMs, we use Mamba-2 \citep{daoTransformersAreSSMs2024} with scalar-gated selective transitions and additional write and read nonlinearities.  Both include RMSNorm and residual connections.

\begin{remark}[Architectural scope]\label{rem:architecture_scope}
One-layer linear attention $\mathcal F_{\mathrm{att}}(1,d,H,L)$ is standard in theoretical ICL \citep[e.g.,][]{oswaldTransformersLearnIncontext2023}, with one-layer fixed-transition SSMs $\mathcal F_{\mathrm{SSM}}(1,d,H,1)$ providing its recurrent counterpart \citep[e.g.,][]{katharopoulosTransformersAreRNNs2020}. 
Case~1 begins with these models, and as new demands arise, later cases expand the comparison to our substantially richer full classes: 
both families vary depth and width; attention also allows variable context length, positional bias, and linear or softmax attention, while SSMs allow convolution and selective transitions.
\end{remark}

\section{Case 1: belief maintenance}\label{sec:case1_lb}
Case~1 recovers the canonical Bayesian ICLR setup with the routing map $\rho(s)=s$. Under this route, each past token $z_s$ contains both factors of its evidence atom $x_sy_s$, and the current token $z_t$ contains the next target's regressor $x_{t+1}$. Evidence assembly and addressing therefore pose no additional demands, leaving belief maintenance as the focus of Case~1. Accordingly, the \emph{aggregation kernel} introduced by our belief geometry becomes the central tool: we use the kernel-mismatch term in \eqref{eq:belief_pythagoras} to compare the two families by how well their kernels approximate the Bayes kernel.

\paragraph{Stationary kernels.}
Applying the kernel-form projection in Figure~\ref{fig:belief_geometry} to the simplest classes of both families---one-layer linear attention $\mathcal F_{\mathrm{att}}(1,d,H,L)$ and one-layer fixed-transition SSMs $\mathcal F_{\mathrm{SSM}}(1,d,H,1)$---reveals our key structural observation: they can at best realize \emph{stationary aggregation kernels}, $K_{t,s}=\kappa_{t-s}$, whose weights depend only on lag and not on the current time or realized history (Appendix~\ref{app:case1_vector_stationary}). Larger width or context length can enrich the stationary profiles they realize, but cannot remove this restriction. We therefore first optimize over all stationary kernels to obtain an architecture-independent benchmark before returning to comparing the two families.

Because our performance measure is cumulative regret, a stationary kernel's inability to adapt to time creates a \emph{bias--variance trade-off}. Spreading weight across larger lags aggregates more evidence atoms and lowers variance but increases cold-start bias before those lags become available; concentrating weight on small lags reduces cold-start bias but aggregates fewer evidence atoms and raises variance. To quantify this trade-off, write $\Sigma_x^{1/2}\Sigma_\theta\Sigma_x^{1/2}=U\operatorname{diag}(\lambda_1,\ldots,\lambda_p)U^\top$ for the eigendecomposition of the prior covariance in the $\Sigma_x$-geometry. For each eigendirection, set $v_i:=\lambda_i+\sum_{j=1}^p\lambda_j+\sigma^2$, the corresponding evidence-innovation variance, and summarize the resulting bias--variance scales by $\zeta:=\sum_{i=1}^p\sqrt{\lambda_i v_i}$ and $\bar\zeta:=\sqrt{(\sum_{i=1}^p\lambda_i)(\sum_{i=1}^p v_i)}$. Cauchy--Schwarz gives $\bar\zeta\geq\zeta$. 
Optimizing the bias--variance trade-off over all stationary kernels yields the following universal regret floor.

\begin{theorem}[Stationary floor]\label{thm:stationary_floor}
Every $g$ whose kernel is stationary obeys $\mathrm{Reg}_T(g)\geq\zeta\sqrt T+o(\sqrt T)$.
\end{theorem}
Therefore, neither class can beat this floor by increasing width or context length; these resources only affect how closely each class approaches it. 
Rather surprisingly, the proposition below shows that the SSM attains the floor in Thm.~\ref{thm:stationary_floor}, whereas attention necessarily pays a larger leading constant.
\begin{proposition}[Stationary pair]\label{prop:stationary_pair}
When \(Hd<p\), both classes incur \(\Omega(T)\) regret (uniformly in \(L\) for attention). When \(Hd\ge p\), their leading-order optimal kernels are:
\begin{enumerate}[label=\emph{(\alph*)},wide=0pt,labelsep=.5em,
                  topsep=0pt,itemsep=0pt,parsep=0pt,partopsep=0pt]
\item $\mathcal F_{\mathrm{att}}(1,d,H,L)$: the \emph{uniform kernel} $\kappa_u=\tfrac{1}{L}\Sigma_x^{-1}+O(\tfrac{1}{L^2}+\tfrac{1}{T})$ for $0\leq u<L$ and $0$ otherwise.
\item $\mathcal F_{\mathrm{SSM}}(1,d,H,1)$: the \emph{exponential kernel} $\kappa_u=\Sigma_x^{-1/2}U\operatorname{diag}((1-\alpha_i^\star)(\alpha_i^\star)^u)_{i=1}^pU^\top\Sigma_x^{-1/2}$ for $u\geq0$, where $\alpha_i^\star=(2\sqrt{v_iT/\lambda_i}-1)/(2\sqrt{v_iT/\lambda_i}+1)$.
\end{enumerate}
For fixed $L$, $\mathrm{Reg}_T(\mathcal F_{\mathrm{att}}(1,d,H,L))=\Omega(T)$. The optimal $L^\star(T):=\sqrt{3T(\sum_i v_i)/(\sum_i\lambda_i)}$ gives
\[
  \mathrm{Reg}_T\bigl(\mathcal F_{\mathrm{att}}(1,d,H,L^\star(T))\bigr)
  =\tfrac{2}{\sqrt3}\bar\zeta\sqrt T+o(\sqrt T),\;\;\;
  \mathrm{Reg}_T\bigl(\mathcal F_{\mathrm{SSM}}(1,d,H,1)\bigr)
  =\zeta\sqrt T+o(\sqrt T).
\]
\end{proposition}

\paragraph{Architectural lessons.}
The optimal kernels reveal how the two families maintain beliefs: attention averages evidence over a window set by its context length, whereas the SSM discounts past evidence at rates set by its recurrent transition.
At first glance, one might expect full-context attention ($L=T$) to be optimal and the SSM's exponentially decaying kernel to forget useful evidence, since $\theta$ is fixed and no observation becomes stale. But linear attention's optimal kernel cannot adapt to elapsed time: at $L=T$, it weights every evidence atom by $\Theta(1/T)$, severely underweighting the early evidence and causing large cold-start bias. Balancing this against the variance from a shorter window then yields $L^\star(T)=\Theta(\sqrt T)$, not $T$. 
This also provides a theoretical analogue of empirical findings that longer context windows need not improve LLMs' effective context use \citep{liuLostMiddleHow2024,hsiehRULERWhatsReal2024}.
Even at the optimal context length, attention pays a shape factor \(2/\sqrt3\): the SSM's exponential kernel achieves a better bias--variance trade-off than attention's uniform kernel.
Beyond this shape factor, the architectural comparison also depends on the task's statistical structure, namely how the signal variance is spread across eigendirections (the \(\lambda_i\)).
The SSM tunes a separate decay rate to each eigendirection, whereas attention's single context length must compromise across all of them.
Thus, at fixed noise and total signal variance, concentrating the signal in fewer eigendirections lowers \(\zeta\) but leaves \(\bar\zeta\) unchanged, increasing the SSM's advantage.
App.~\hyperref[app:case1_statistical_effects]{\ref*{app:case1_stationary_pair}} further examines these statistical effects and also the loss curvature at the optimal weights; App.~\ref{app:experiments_case1_analytical} experimentally validates Prop.~\ref{prop:stationary_pair}.

\paragraph{Stationary optima as inductive biases.}
Practical attention and SSMs are not confined to stationary kernels. Indeed, normalization and additional depth let both families escape the stationary regret floor (Appendix~\ref{app:case1_adaptive}).
However, the stationary comparison remains consequential: it isolates the temporal weighting supplied by each family's core mechanism and predicts which kernels richer architectures in each family learn most readily under finite optimization---their inductive biases.

\FloatBarrier
\begin{wrapfigure}[20]{r}{0.56\textwidth}
  \vspace{-1.5\baselineskip}
  \centering
  \includegraphics[width=\linewidth]{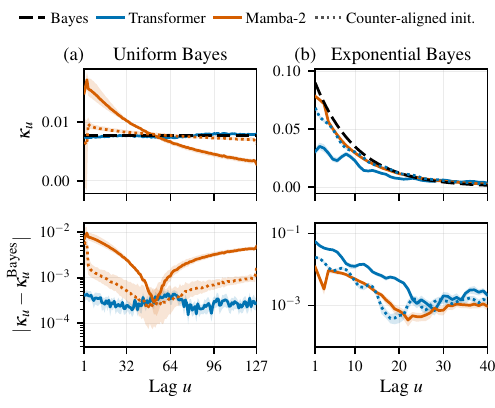}
  \begin{minipage}{0.75\linewidth}
    \vspace{-0.6\baselineskip}
    \caption{Kernel alignment in practical models under finite optimization. Shading shows 95\% bootstrap CIs over 20 seeds.}
    \label{fig:case1_kernel_alignment}
  \end{minipage}
\end{wrapfigure}

We test this hypothesis using single-layer instances of the practical architectures described in \S\ref{sec:architectures}. 
We use two scalar Gaussian filtering tasks: one with a uniform and another with an exponential Bayes kernel.
Fig.~\ref{fig:case1_kernel_alignment} compares each regime's Bayes kernel with the learned aggregation kernels estimated from model predictions.
App.~\ref{app:experiments_case1_practical} gives full experimental details.

Regardless of the regime, the Transformer learns flatter kernels, whereas Mamba-2 learns exponentially decaying ones, mirroring the in-class optimal kernels of Prop.~\ref{prop:stationary_pair}.
Correspondingly, the Transformer performs better in the uniform regime and Mamba-2 in the exponential regime.
The mismatch is asymmetric under random initialization: Mamba-2 incurs $1.9\times$ the Transformer's regret in the uniform regime, whereas the Transformer incurs $4.3\times$ Mamba-2's in the exponential regime, echoing Proposition~\ref{prop:stationary_pair}, where the exponential kernel attains the stationary optimum. 
To distinguish representational limitations from inductive bias, we also design a counter-aligned initialization (dotted lines) that biases each model toward the other Bayes kernel---uniform for Mamba-2 and exponential for the Transformer. 
Without changing either architecture, this initialization largely closes both performance gaps, showing that both models can closely approximate either kernel but more readily learn the one favored by their stationary core.

\section{Case 2: routing}\label{sec:case2_lb}
With a nontrivial routing map $\rho(s)<s$, each label $y_s$ is separated from the regressor $x_{\rho(s)}$ that generated it. This creates two architectural demands. First, the architecture must assemble each evidence atom $x_{\rho(s)}y_s$ from factors appearing in different tokens. Second, it must address the past regressor $x_{\rho(t+1)}$ to which the maintained belief is applied. Once an architecture satisfies these demands, the
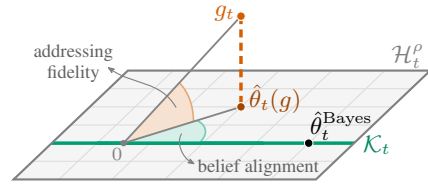
\begin{wrapfigure}[8]{r}{0.42\textwidth}
  \vspace{-1.2\baselineskip}
  \centering
\begingroup
\definecolor{rgBayes}{HTML}{000000}
\definecolor{rgModel}{HTML}{D55E00}
\definecolor{rgKernel}{HTML}{009E73}
\definecolor{rgInk}{HTML}{333333}
\definecolor{rgPlaneFill}{HTML}{F5F5F5}
\definecolor{rgPlaneEdge}{HTML}{808080}
\definecolor{rgGrid}{HTML}{DCDCDC}
\begin{tikzpicture}[
  x={(1cm,0cm)}, y={(0.42cm,0.40cm)}, z={(0cm,1cm)},
  lab/.style={font=\footnotesize,inner sep=1.4pt},
  plab/.style={font=\footnotesize,inner sep=1.4pt},
  note/.style={font=\scriptsize,inner sep=1.2pt},
  dot/.style={circle,inner sep=0pt,minimum size=3.2pt,fill=#1,draw=white,line width=.35pt},
  sdot/.style={circle,inner sep=0pt,minimum size=3.2pt,fill=#1,draw=white,line width=.35pt},
  ra/.style={line width=.6pt,draw=black!50},
  co/.style={-{Stealth[length=3.1pt,width=2.3pt,inset=0.5pt,round]},line width=.42pt,line cap=round,shorten <=0.6pt,shorten >=0.8pt},
]
  \fill[rgPlaneFill] (-0.5,-1.35,0) -- (3.5,-1.35,0) -- (3.5,2.25,0) -- (-0.5,2.25,0) -- cycle;
  \begin{scope}
    \clip (-0.5,-1.35,0) -- (3.5,-1.35,0) -- (3.5,2.25,0) -- (-0.5,2.25,0) -- cycle;
    \foreach \gx in {0.1,0.8,1.5,2.2,2.9}
      \draw[rgGrid,line width=.4pt] (\gx,-1.35,0) -- (\gx,2.25,0);
    \foreach \gy in {-0.75,0.45,1.05,1.65}
      \draw[rgGrid,line width=.4pt] (-1.0,\gy,0) -- (4.7,\gy,0);
    \draw[rgKernel,line width=1.3pt] (-1.1,-0.15,0) -- (4.9,-0.15,0);
  \end{scope}
  \draw[rgPlaneEdge,line width=.8pt]
    (-0.5,-1.35,0) -- (3.5,-1.35,0) -- (3.5,2.25,0) -- (-0.5,2.25,0) -- cycle;
  \draw[rgModel,line width=1.3pt,dash pattern=on 3pt off 2pt] (1.5,1.05,1.2) -- (1.5,1.05,0);
  \node[dot=rgBayes]  at (2.9,-0.15,0) {};
  \node[sdot=rgModel!75!black] at (1.5,1.05,0) {};
  \node[dot=rgModel]  at (1.5,1.05,1.2) {};
  \node[lab,text=rgPlaneEdge,anchor=south east] at (4.42,0,0.88) {$\mathcal{H}_t^{\rho}$};
  \node[lab,text=rgKernel!85!black,anchor=west] at (3.50,0,-0.08) {$\mathcal{K}_t$};
  \node[plab,text=rgBayes,anchor=south west] at (2.8,0,-0.07) {$\hat\theta_t^{\mathrm{Bayes}}$};
  \node[plab,text=rgModel,anchor=east] at (1.92,0,1.66) {$g_t$};

\fill[rgKernel!16] (0.45,-0.15,0) -- plot[domain=0:48.8,samples=22]
  ({0.45+1.0*cos(\x)},{-0.15+1.0*sin(\x)},0) -- cycle;
\draw[rgKernel!55,line width=.6pt] plot[domain=0:48.8,samples=22]
  ({0.45+1.0*cos(\x)},{-0.15+1.0*sin(\x)},0);
\fill[rgModel!18] (0.45,-0.15,0) -- plot[domain=0:37.0,samples=22]
  ({0.45+0.95*cos(\x)*0.6585},{-0.15+0.95*cos(\x)*0.7526},{0.95*sin(\x)}) -- cycle;
\draw[rgModel!55,line width=.6pt] plot[domain=0:37.0,samples=22]
  ({0.45+0.95*cos(\x)*0.6585},{-0.15+0.95*cos(\x)*0.7526},{0.95*sin(\x)});
\draw[rgKernel,line width=1.3pt] (0.1,-0.15,0) -- (2.3,-0.15,0);
\draw[black!48,line width=.7pt] (0.45,-0.15,0) -- (1.5,1.05,0);
\draw[black!48,line width=.7pt] (0.45,-0.15,0) -- (1.5,1.05,1.2);
\node[circle,inner sep=0pt,minimum size=2.4pt,fill=black!50] at (0.45,-0.15,0) {};
\node[note,text=black!50,anchor=north east] at (0.42,0,-0.07) {$0$};
\node[sdot=rgModel!75!black] at (1.5,1.05,0) {};
\node[dot=rgModel] at (1.5,1.05,1.2) {};
\draw[co,black!50]
  (1.08,0.2,0) to[out=-120,in=150] (1.35,0,-0.35);
\node[note,text=black!62,align=center,anchor=north west] at (1.32,0,-0.22)
  {belief alignment};
\draw[co,black!50]
  (0.92,0.409,0.232) to[out=120,in=-35] (0.2,-0.1,1.0);
\node[note,text=black!62,align=center,anchor=east] at (0.30,0,1.00)
  {addressing\\fidelity};
\node[plab,text=rgModel!75!black,anchor=west] at (1.91,0.2,0.45) {$\hat\theta_t(g)$};
\end{tikzpicture}
\endgroup
   \vspace{-0.5\baselineskip}
  \caption{Floor factors in belief geometry.}
  \label{fig:routing_geometry}
\end{wrapfigure}
remaining problem is the Case~1 kernel comparison. 
Accordingly, two elements of belief geometry become central in Case 2:
\emph{belief alignment}, which measures agreement between the implicit and Bayes beliefs and therefore captures failures of evidence assembly, and \emph{addressing fidelity}, which measures whether the prediction applies its implicit belief to the correct regressor. 
Together, they yield the following lower bound.

\begin{theorem}[Two-factor floor]\label{thm:routing_floor}
For every $g$, with $\mathrm{corr}_{\Sigma_x}^2$ denoting the squared cosine under $\langle\cdot,\cdot\rangle_{\Sigma_x}$,
\[
  \mathrm{Reg}_T(g)
  \;\geq\;
  \sum_{\substack{t<T:\\t+1\ \mathrm{matched}}}
  \|\hat\theta_t^{\mathrm{Bayes}}\|_{\Sigma_x}^2
  \biggl(
    1
    -
    \underbrace{
      \mathrm{corr}_{\Sigma_x}^2\!\left(
        \hat\theta_t(g),\hat\theta_t^{\mathrm{Bayes}}
      \right)
    }_{\text{belief alignment}}
    \cdot
    \underbrace{
      \frac{\|\hat\theta_t(g)\|_{\Sigma_x}^2}
           {\|\hat\theta_t(g)\|_{\Sigma_x}^2+\|g_t^\perp\|_t^2}
    }_{\text{addressing fidelity}}
  \biggr).
\]
\end{theorem}

Both factors lie in $[0,1]$, so if either remains uniformly below $1$ on a constant fraction of matched steps, then $\mathrm{Reg}_T(g)=\Omega(T_\rho)$, where $T_\rho:=|\{s\leq T:s\ \text{matched}\}|$ is the \emph{effective horizon} and we assume $T_\rho=\omega(\log T)$. The two routing rules below expose different sources of this floor: positional routing makes evidence assembly the bottleneck (Case~2.a, \S\ref{sec:case2a}), whereas content routing makes addressing the bottleneck (Case~2.b, \S\ref{sec:case2b}).

\subsection{Case 2.a: positional routing}\label{sec:case2a}

Case~2.a introduces the positional routing map \(\rho(s)=s-K\) for \(s>K\), where \(1\leq K<T\). The first $K$ steps are unmatched, so $T_\rho=T-K$. 
Under this route, assembling each evidence atom \(x_{s-K}y_s\) requires reaching \(K\) steps into the past. The same reach also supplies the regressor \(x_{t+1-K}\) to which the belief is applied. Thus, solving evidence assembly also solves addressing.

One might expect retaining the relevant history to suffice for positional assembly: attention keeps past tokens in its context, while an SSM compresses them into its recurrent state. But neither alone pairs each $y_s$ with $x_{s-K}$: Case 2.a requires a compositional mechanism that first supplies positional reach and then uses it for assembly. Without this composition, increasing width or context length cannot make belief alignment in Thm.~\ref{thm:routing_floor} approach $1$, yielding the following floors for Case~1 architectures.

\begin{corollary}[Assembly floors]\label{cor:case2a_floors}
\raggedright
Under the positional-routing setup above:
\begin{enumerate}[label=\emph{(\alph*)},wide=0pt,labelsep=.5em,
                  topsep=-2pt,itemsep=3pt,parsep=0pt,partopsep=0pt]
\item $\mathrm{Reg}_T(\mathcal F_{\mathrm{att}}(1,d,H,L))=\Omega(T_\rho)$ uniformly in $d,H,L$.
\item $\mathrm{Reg}_T(\mathcal F_{\mathrm{SSM}}(1,d,H,d_{\mathrm{conv}}))=\Omega(T_\rho)$ uniformly in $d,H\ge1$ and $1\le d_{\mathrm{conv}}\le K$ with $K$ fixed.
\end{enumerate}
\end{corollary}
\vspace{-0.35em}
Both families escape these floors by the same two-stage composition. First, they form the lag-shifted token $(x_{s+1-K},y_s,x_{s-K})$, reducing Case~2.a to Case~1: two-layer attention uses positional encoding in its first layer, whereas the one-layer SSM uses a width-$(K+1)$ convolution. Then, they reuse their respective kernels from Prop.~\ref{prop:stationary_pair} over the effective horizon $T_\rho$, yielding the following regret rates.

\begin{proposition}[Positional-routing pair]\label{prop:case2a_pair}
Under the positional-routing setup above, 
when $H_{\mathrm s}d_{\mathrm s}<p$, $\mathrm{Reg}_T(\mathcal F_{\mathrm{SSM}}(1,d_{\mathrm s},H_{\mathrm s},d_{\mathrm{conv}}))=\Omega(T_\rho)$ uniformly in $d_{\mathrm{conv}}$.
When $H_{\mathrm a}d_{\mathrm a}\geq 2p$ and $H_{\mathrm s}d_{\mathrm s}\geq p$, setting $\widetilde L^\star:=(K+1)\vee L^\star(T_\rho)$ gives
\[
  \mathrm{Reg}_T\bigl(\mathcal{F}_{\mathrm{att}}(2,d_{\mathrm a},H_{\mathrm a},\widetilde L^\star)\bigr)
  = O\bigl(K \vee \sqrt{T_\rho}\bigr),
  \qquad
  \mathrm{Reg}_T\bigl(\mathcal{F}_{\mathrm{SSM}}(1,d_{\mathrm s},H_{\mathrm s},K{+}1)\bigr)
  = \Theta\bigl(\sqrt{T_\rho}\bigr).
\]
When $K=O(\sqrt{T_\rho})$, the attention bound is also sharp: $\mathrm{Reg}_T(\mathcal{F}_{\mathrm{att}}(2,d_{\mathrm a},H_{\mathrm a},\widetilde L^\star)) = \Theta(\sqrt{T_\rho})$.
\end{proposition}
\vspace{-0.88em}
\paragraph{Architectural lessons.}
\looseness=-1
Together, Cor.~\ref{cor:case2a_floors} and Prop.~\ref{prop:case2a_pair} add a lesson beyond Case~1's leading constants: the two families allocate resources for positional reach and belief maintenance differently. Attention uses the context length $L$ for both jobs, so its construction takes $\widetilde L^\star$ to span both the routing lag $K$ and Prop.~\ref{prop:stationary_pair}'s optimal window $L^\star(T_\rho)$. The SSM instead separates them: $\dconv=K+1$ supplies only the positional reach, while its recurrent decay is tuned independently for belief maintenance as in Prop.~\ref{prop:stationary_pair}. This has a memory consequence: the SSM stores only a \(p\)-dimensional recurrent state and \(K\) residual-stream vectors for convolution, whereas attention retains a growing history of \(\widetilde L^\star=\Theta(K\vee\sqrt{T_\rho})\) such vectors. The growing context requirement is not specific to our two-layer construction:
\begin{proposition}[Attention context requirement]
\label{prop:attention_context_requirement}
Under the positional-routing setup above, for every fixed $D$,
$\mathrm{Reg}_T(\mathcal F_{\mathrm{att}}^{\mathrm{sm}}(D,d,H,L))
=O(\sqrt{T_\rho})$ requires $L=\Omega(\sqrt{T_\rho})$,
uniformly in $d,H$.
\end{proposition}
\vspace{-0.3em}
Props.~\ref{prop:case2a_pair} and~\ref{prop:attention_context_requirement} show that, when $K=o(\sqrt{T_\rho})$, the SSM reaches the same regret rate with strictly less memory than fixed-depth attention with standard caching, regardless of width or softmax use (App.~\ref{app:case2a_memory_comparison}).
This identifies a task regime that favors SSMs: short-range positional assembly coupled with long-term belief maintenance, echoing Mamba's empirical advantage in \citet{bondaschiMarkovLaplaceHow2026}.

\begin{wrapfigure}[12]{r}{221.5bp}
  \vspace{-1.52\baselineskip}
  \centering
  \includegraphics{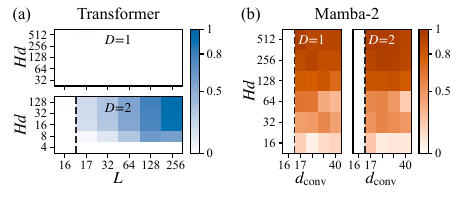}
  \begin{minipage}{0.91\linewidth}
    \vspace{-1.2\baselineskip}
    \caption{Practical validation of positional assembly mechanisms. Darker means closer to Bayes; dashed lines mark the $K+1=17$ threshold.}
  \label{fig:case2a_bandit_resources}
  \end{minipage}
\end{wrapfigure}
\vspace{-0.57em}
\paragraph{Practical validation.}
App.~\ref{app:experiments_case2a_analytical} validates these results directly within our analytically tractable classes, and
App.~\ref{app:experiments_bernoulli} confirms these architectural lessons for the practical models of \S\ref{sec:architectures} on a conjugate Beta--Bernoulli bandit. Here we test whether these lessons also extend to nonconjugate setups (App.~\hyperref[app:nonconjugate_logistic]{\ref*{app:conjugate_belief_geometry}}), using a delayed-feedback reward prediction task for a 30-arm contextual logistic bandit (\(p=30\), \(K=16\)).
Fig.~\ref{fig:case2a_bandit_resources} reports a Bayes-normalized performance score computed from cumulative excess log loss: $0$ corresponds to a prior-only predictor that ignores the history, and $1$ to Bayes. Throughout, we use \(0.8\) as the score threshold for high performance. 
Because our focus is representational capability, each cell reports the score of the best validation-selected run from a fixed sweep over training seeds and learning rates.
App.~\ref{app:experiments_common_evaluation} details this choice and the full protocol.

Consistent with Cor.~\ref{cor:case2a_floors}, the one-layer Transformer and one-layer Mamba-2 with \(d_{\mathrm{conv}}\le K\) remain near baseline. Although depth extends the overall positional reach, both two-layer models improve only once their first layer can span the lag ($L,d_{\mathrm{conv}}\ge K+1=17$), matching the positional-assembly mechanism in Prop.~\ref{prop:case2a_pair}. 
For the Transformer, $Hd\leq8$ falls short, but once $Hd$ reaches $16<p=30$, performance varies almost entirely with $L$, exceeding $0.8$ at $L=256$. This matches the architectural lesson that $L$, not $Hd$, governs belief maintenance while also supplying positional reach.
For Mamba-2, $D=2$ notably leaves the $d_{\mathrm{conv}}=17$ boundary unchanged, although two layers enable shorter convolutions for fixed-transition SSMs on the Gaussian task (App.~\ref{app:experiments_case2a_analytical}).
At either depth, Mamba-2's performance improves primarily with \(Hd\) and first exceeds \(0.8\) at \(Hd=128\), well above $p=30$. 
This fits the complementary lesson for SSMs: convolution supplies positional assembly, whereas recurrent-state capacity governs belief maintenance.
Finally, the resource differences are reflected in total memory: the lowest-memory tested Transformer exceeding \(0.8\) uses \(5.2\times\) as much as its Mamba-2 counterpart (App.~\hyperref[app:experiments_case2a_memory]{\ref*{app:experiments_case2a_practical}}), echoing the theoretical memory comparison.
\FloatBarrier
\vspace{-0.25em}
\subsection{Case 2.b: content routing}\label{sec:case2b}
\vspace{-0.25em}
Case~2.b introduces a fixed codebook of $R>0$ unit vectors $\{\psi^{(1)},\ldots,\psi^{(R)}\}\subset\RR^{d_{\mathrm{key}}}$. 
Writing $\psi_s$ for the key observed at step $s$, we define the content routing map $\rho(s):=\max\{r<s:\psi_r=\psi_s\}$ for $s>R$.
Following the layout in \S\ref{sec:setup}, each token additionally includes the current and next keys: \(z_t=(\psi_{t+1},x_{t+1},y_t,\psi_t,x_t)\).
Keys arrive in \(R\)-step cycles, each an independent uniform permutation of the codebook. The first cycle is unmatched, so \(T_\rho=T-R\).
Under this route, content rather than a fixed positional offset identifies the relevant past regressor, yielding a continuous-valued variant of associative recall, which is a widely studied sequence-model benchmark \citep{aroraZoologyMeasuringImproving2023,huangUnderstandingInputSelectivity2025}. In our case, however, it is embedded in Bayesian linear regression and must be solved alongside evidence assembly and belief maintenance.

\begin{wrapfigure}[3]{r}{.51\textwidth}
\vspace{-1.9\intextsep}
\centering
\begin{tikzpicture}[
  scale=.93,
  transform shape,
  x=1.14cm,
  key/.style={draw=black!60,rounded corners=1pt,minimum width=1.02cm,minimum height=.44cm,inner sep=1pt,font=\small},
  route/.style={densely dashed,-{Stealth[length=2.5pt,width=2.5pt]},draw=black!55,line width=.4pt},
  route label/.style={font=\footnotesize,text=black!70,fill=white,fill opacity=.82,text opacity=1,rounded corners=1.5pt,inner xsep=2pt,inner ysep=.7pt},
  group/.style={decorate,decoration={brace,mirror,amplitude=2pt},draw=black!55,line width=.4pt}
]
  \node[key,fill=keyblue!14]   (k1) at (0,0) {$\psi_1=a$};
  \node[key,fill=keyorange!16] (k2) at (1,0) {$\psi_2=b$};
  \node[key,fill=keypurple!14] (k3) at (2,0) {$\psi_3=c$};
  \node[key,fill=keypurple!14] (k4) at (3.25,0) {$\psi_4=c$};
  \node[key,fill=keyorange!16] (k5) at (4.25,0) {$\psi_5=b$};
  \node[key,fill=keyblue!14]   (k6) at (5.25,0) {$\psi_6=a$};
  \draw[route] (k4.north) .. controls +(0,.13) and +(0,.13) .. (k3.north);
  \draw[route] (k5.north) .. controls +(0,.27) and +(0,.27) .. (k2.north);
  \draw[route] (k6.north) .. controls +(0,.41) and +(0,.41) ..
    node[pos=0,left=.08pt,xshift=-75pt,yshift=9pt,route label] {$\rho(s)$} (k1.north);
  \draw[group] (-.47,-.28) -- (2.47,-.28)
    node[midway,below=1.5pt,font=\scriptsize,text=black!75] {unmatched};
  \draw[group] (2.78,-.28) -- (5.72,-.28)
    node[midway,below=1.5pt,font=\scriptsize,text=black!75] {matched};
\end{tikzpicture}
\vspace{-\intextsep}
\end{wrapfigure}
The diagram to the right shows one realization of the first two cycles for $R=3$, with codebook $\{a, b, c\}$. Its unequal spans illustrate the random lags $s-\rho(s)\in\{1,\ldots,2R-1\}$, so no fixed offset can locate all the previous occurrences. To avoid the positional-reach issue encountered in Case~2.a, we take \(R=o(\sqrt{T_\rho})\), so \(L^\star(T_\rho)\) spans every lag. At the same time, we keep content addressing nontrivial with a compressed codebook satisfying $d_{\mathrm{key}}=\Theta(\log R)$ and $\max_{i\neq j}|\langle\psi^{(i)},\psi^{(j)}\rangle|<1/3$.

Here, unlike in Case 2.a, solving evidence assembly does not automatically solve addressing. Belief maintenance allows evidence assembly to tolerate coarse key matching: for $r\neq s$, the regressor $x_{\rho(r)}$ is independent of $y_s$ and has zero mean, so spurious terms $x_{\rho(r)}y_s$ average out in the belief. 
But addressing cannot rely on this averaging: it requires the particular past regressor named by \(\psi_{t+1}\).
Thus, the architectures that suffice for Case 2.a can still make belief alignment in Theorem~\ref{thm:routing_floor} approach $1$, but they leave addressing fidelity bounded away from $1$, forcing linear regret. Accordingly, we introduce \emph{addressing capacity}, the central belief-geometry tool for quantifying this remaining failure.
\vspace{-1.7em}
\paragraph{Addressing capacity.} Fix $t\geq R$. Let $N_t$ be the number of keys not yet visited in the cycle containing $t+1$, and index them by $q=1,\ldots,N_t$. For each candidate $q$, let $X_q\in\RR^p$ be the most recent regressor carrying that key, and let $g_{t,q}$ be the predictor's time-$t$ output if the revealed key is $q$. The key idea behind addressing capacity is to decompose each candidate regressor into its components parallel and orthogonal to the Bayes-belief direction $\hat\theta_t^{\mathrm{Bayes}}$.
Let \(\mathcal C_t\) collect the history with each candidate regressor replaced by its orthogonal component, leaving only their components along the Bayes-belief direction random
(formalized in Appendix~\ref{app:case2}). 
Crucially, conditioning on \(\mathcal C_t\) makes the candidate Bayes predictions $(X_q^\top\hat\theta_t^{\mathrm{Bayes}})_{q=1}^{N_t}$ independent Gaussians while leaving each \(g_{t,q}\) free to depend arbitrarily on all of them through the candidate regressors' earlier occurrences.
This allows us to decompose the analysis across Bayes-prediction directions and define
\[
  \mathrm{Cap}_t(g)
  :=\sum_{q=1}^{N_t}
  \mathrm{corr}_{\mathcal C_t}^2
  \Big(g_{t,q},\,X_q^\top\hat\theta_t^{\mathrm{Bayes}}\Big),
\]
which is an analytically tractable measure of how many candidates are addressed. Because the arriving key is uniform, the average two-factor product in Theorem~\ref{thm:routing_floor} is at most $\mathrm{Cap}_t(g)/N_t$. 
Together with \(N_t\ge R/2\) on half of every cycle, this bound implies that $\mathrm{Cap}_t(g)=o(R)$ forces $\Omega(T_\rho)$ regret.
Thus, deriving addressing floors reduces to upper-bounding addressing capacity, which we do next.
\begin{theorem}[Addressing floors]\label{thm:case2b_floors}
Under the content-routing setup above, fix $D\geq1$ and $t\ge R$:
\begin{enumerate}[label=\emph{(\alph*)},wide=0pt,labelsep=.5em,
                  topsep=-2pt,itemsep=2pt,parsep=0pt,partopsep=0pt]
\item
Every $g\in\mathcal F_{\mathrm{att}}(D,d,H,L)$ satisfies
$\mathrm{Cap}_t(g)=o(R)$, uniformly in $d,H,L$.
\item
Every $g\in\mathcal F_{\mathrm{SSM}}^{\mathrm{aff}}(D,d,H,d_{\mathrm{conv}})$ satisfies $\mathrm{Cap}_t(g)=o(R)$,
uniformly in $d,H,d_{\mathrm{conv}}$.
\item
Every $g\in\mathcal F_{\mathrm{SSM}}^{\mathrm{sel}}(D,d,H,d_{\mathrm{conv}})$
satisfies $\mathrm{Cap}_t(g)=O(DHd\log R)$, uniformly in $d,H,d_{\mathrm{conv}}$.
\end{enumerate}
Thus, all predictors in (a)--(b) incur \(\Omega(T_\rho)\) regret, as do those in (c) with \(DHd=o(R/\log R)\).
\end{theorem}

Unlike prior bit-based lower bounds for associative recall and copying \citep{jelassiRepeatMeTransformers2024,wenRNNsAreNot2024,aroraSimpleLinearAttention2025}, Theorem~\ref{thm:case2b_floors} requires no finite-precision assumption and treats continuous, noisy prediction. The result also clarifies why softmax matters: linear attention is often motivated as a more efficient replacement for softmax attention \citep{katharopoulosTransformersAreRNNs2020,schlagLinearTransformersAre2021}, yet here that replacement makes sublinear regret impossible at any fixed depth.

The following constructions escape the floors of Thm.~\ref{thm:case2b_floors} by replacing Case~2.a's fixed positional offset with \emph{sharp nonlinear selection}.
Attention selects from retained tokens with separate softmax heads for evidence assembly and addressing.
The SSM instead encodes each candidate in two recurrent-state coordinates per head using a diagonal sigmoid gate for addressing and assembles evidence with a width-\(2\) convolution.
Both then reuse their Case~1 kernels from Prop.~\ref{prop:stationary_pair} 
for belief maintenance.

\begin{proposition}[Content-routing pair]\label{prop:case2b_pair}
Under the content-routing setup above,
\[
  \mathrm{Reg}_T\bigl(\mathcal{F}_{\mathrm{att}}^{\mathrm{sm}}(2,d_{\mathrm{key}}\vee p,2,L^\star(T_\rho))\bigr)
  = O(\sqrt{T_\rho}),
  \qquad
  \mathrm{Reg}_T\bigl(\mathcal{F}_{\mathrm{SSM}}^{\mathrm{sel}}(2,2R,p,2)\bigr)
  = O(\sqrt{T_\rho}).
\]
\end{proposition}

\paragraph{Architectural lessons.}\looseness=-1
Thm.~\ref{thm:case2b_floors} and Prop.~\ref{prop:case2b_pair} reveal a key architectural difference in how sharp selectivity is employed. Softmax attention selects among retained tokens, so it can address $R$ candidate tokens with only $O(\log R)$-dimensional representations; the SSM instead preserves candidates separately in its recurrent state, so the first-layer-selective SSM class $\mathcal{F}_{\mathrm{SSM}}^{\mathrm{sel}}$ requires $\widetilde\Omega(R)$ recurrent-state dimensions, which our $O(R)$ construction matches up to logarithmic factors. Together, Cases~2.a--b yield a clean division of labor: convolution gives SSMs cheap positional reach, while softmax gives attention cheap content addressing.

\begin{wrapfigure}[11]{r}{221.5bp}
  \vspace{-1.6\baselineskip}
  \centering
  \includegraphics{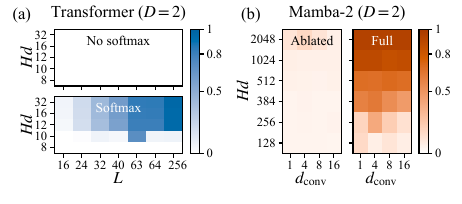}
  \begin{minipage}{0.91\linewidth}
  \vspace{-1.2\baselineskip}
  \caption{Practical validation of content addressing mechanisms. Darker means closer to Bayes.}
  \label{fig:case2b_mechanisms_scaling}
  \end{minipage}
\end{wrapfigure}
\paragraph{Practical validation.}
We test two-layer practical models on Case~2.a's delayed-bandit task, replacing its fixed delay with content routing over a compressed codebook of size \(R=32\). Consistent with Theorem~\ref{thm:case2b_floors}, Figure~\ref{fig:case2b_mechanisms_scaling} shows that the no-softmax Transformer---despite retaining its SwiGLU MLPs---and ablated Mamba-2 with fixed transitions and no additional write nonlinearities---despite retaining nonlinear reads---remain far below Bayes. 
By contrast, the softmax Transformer and full Mamba-2 achieve much higher scores.
For the Transformer, performance rises sharply when \(Hd\) reaches \(12<R=32\) and then changes little with further width. Improvement with \(L\) is instead smooth, with no sudden jump from \(L=2R-1=63\) to \(L=2R=64\), reflecting Case~2.a's belief-maintenance trend without its positional-reach threshold.
Full Mamba-2 instead depends mainly on \(Hd\), with no systematic benefit from increasing \(d_{\mathrm{conv}}\).

\begin{wrapfigure}[6]{r}{120.2pt}
  \vspace{-1.6\baselineskip}
  \centering
  \includegraphics{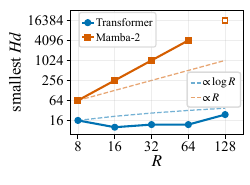}
  \makebox[\linewidth][r]{
  \begin{minipage}[t]{133pt}
    \setlength{\abovecaptionskip}{-15pt}
    \caption{Width scaling (log--log).}
    \label{fig:case2b_resource_scaling}
  \end{minipage}%
  }
\end{wrapfigure}

Fig.~\ref{fig:case2b_resource_scaling} tracks the smallest tested \(Hd\) reaching \(0.8\) as \(R\) increases from \(8\) to \(128\). It stays within \(10\)--\(24\) for the Transformer but grows rapidly for Mamba-2 (hollow marker means \(0.8\) not reached).
This qualitatively matches the exponential width separation jointly established by the lower bound in Thm.~\ref{thm:case2b_floors}(c) and the constructions in Prop.~\ref{prop:case2b_pair}. 
App.~\ref{app:experiments_case2b_practical} gives full experimental details.
\FloatBarrier

\section{Conclusion}\label{sec:conclusion}
In this paper, we developed belief geometry to compare the representational capabilities of attention and SSMs in ICLR. Because its elements abstract common computational demands, we expect them to aid analyses across broader settings. Another promising direction is to extend this representational framework to optimization, thereby helping explain why training favors some solutions over others.

\subsection*{AI use statement}

We used generative AI tools for refining theoretical claims and assisting with their proofs, implementing experiments, surveying related work, and revising the manuscript. We checked the mathematical arguments, experimental implementations, results, and cited sources, and take full responsibility for the final content of the paper. We did not use generative AI tools to generate research ideas or to propose hypotheses.

\subsection*{Ethics statement}

This work combines theoretical analysis of attention and SSMs with experiments on synthetic data. It involves no human participants or personal data. We identify no ethical concerns specific to the work presented.

\subsection*{Reproducibility statement}

Section~\ref{sec:architectures} defines our architecture classes, with practical connections and extensions in Appendix~\ref{app:arch}. Appendix~\ref{app:experiments} documents the experimental setups and training and evaluation protocols, while Appendices~\ref{app:case1} and~\ref{app:case2} provide full statements and proofs of the theoretical results. The accompanying \href{https://anonymous.4open.science/r/belief-routing-experiments-2EF4/}{anonymous code} includes implementations, configurations, plotting data, and commands to rerun the experiments and recreate the figures.

\bibliographystyle{iclr2027_conference}

\ifincludeappendix
\clearpage
\appendix

\section*{Appendix contents}
\pdfbookmark[0]{Appendix contents}{appendix.contents}

\begingroup
\setlength{\parindent}{0pt}
\setlength{\parskip}{0pt}

\newcommand{\appcontentsleaders}{\nobreak\hspace{0.7em}%
  \leaders\hbox to 0.6em{\hss\textcolor{black!30}{.}\hss}\hfill
  \nobreak\hspace{0.7em}}
\newcommand{\appcontentsguide}[2]{%
  \noindent\hyperref[#1]{#2\appcontentsleaders
    \makebox[1.6em][r]{\pageref*{#1}}}\par\vspace{5pt}}
\newcommand{\appcontentssection}[2]{%
  \par\addvspace{15pt}%
  \noindent\hyperref[#1]{\textbf{\makebox[1.8em][l]{\ref*{#1}}#2}%
    \appcontentsleaders\makebox[1.6em][r]{\textbf{\pageref*{#1}}}}\par}
\newcommand{\appcontentssubsection}[2]{%
  \par\addvspace{6pt}%
  \noindent\hspace*{1.8em}\hyperref[#1]{%
    \makebox[2.7em][l]{\ref*{#1}}#2\appcontentsleaders
    \makebox[1.6em][r]{\pageref*{#1}}}\par}

\vspace{4pt}
\appcontentsguide{tab:results_summary}{Summary of theoretical results}
\appcontentsguide{app:proof_guide}{Guide to the proofs}
\appcontentsguide{app:notation_guide}{Notation guide}
\vspace{5pt}
{\color{black!25}\hrule height 0.4pt}
\vspace{2pt}

\appcontentssection{app:related}{Extended related work}

\appcontentssection{app:conjugate_belief_geometry}{Beyond the linear-Gaussian testbed}

\appcontentssection{app:arch}{Architecture classes: practical connections and extensions}
\appcontentssubsection{app:practical_architecture_connections}{Connections to practical architectures}
\appcontentssubsection{app:count_normalized_extensions}{Count-normalized extensions}

\appcontentssection{app:experiments}{Additional experiments and experimental details}
\appcontentssubsection{app:experiments_practical}{Details of the main-text experiments}
\appcontentssubsection{app:experiments_bernoulli}{Validation on Beta--Bernoulli bandits}
\appcontentssubsection{app:experiments_analytical}{Validation of the theoretical results}

\appcontentssection{app:case1}{Case 1: full statements and proofs}
\appcontentssubsection{app:case1_vector_stationary}{Theorem~\ref*{thm:stationary_floor} and Proposition~\ref*{prop:stationary_pair}: Stationary kernels}
\appcontentssubsection{app:case1_adaptive}{Extension to time- and data-adaptive kernels}

\appcontentssection{app:case2}{Case 2: full statements and proofs}
\appcontentssubsection{app:vector_routing_geometry}{Proof of Theorem~\ref*{thm:routing_floor}}
\appcontentssubsection{app:case2a_vector}{Case 2.a: positional routing}
\appcontentssubsection{app:case2b_capacity}{Case 2.b: content routing}

\endgroup
 \clearpage

\begin{table}[!ht]
  \centering
  \caption{Summary of theoretical results. Class notations are defined in \S\ref{sec:architectures}.}
  \label{tab:results_summary}
  \smallskip
  \small
  \setlength{\tabcolsep}{1.5pt}
  \renewcommand{\arraystretch}{1.2}
  \begin{tabular}{@{}L{0.075\textwidth}L{0.205\textwidth}@{\hspace{\dimexpr2\tabcolsep+4pt\relax}}L{\dimexpr0.388\textwidth-4pt\relax}L{\dimexpr0.332\textwidth-6\tabcolsep\relax}@{}}
    \toprule
    & Bottleneck\newline(Belief geometry tool) & Lower-bound classes & Construction classes \\
    \midrule
    Case~1\newline(\S\ref{sec:case1_lb})
    & Belief maintenance\newline(Aggregation kernel)
    & Proposition~\ref{prop:stationary_pair}\par\smallskip
      $\mathcal F_{\mathrm{att}}(1,d,H,L)$, $Hd<p$ or fixed $L$\newline
      $\mathcal F_{\mathrm{SSM}}(1,d,H,1)$, $Hd<p$
    & Proposition~\ref{prop:stationary_pair}\par\smallskip
      $\mathcal F_{\mathrm{att}}(1,d,H,L^\star(T))$, $Hd\ge p$\newline
      $\mathcal F_{\mathrm{SSM}}(1,d,H,1)$, $Hd\ge p$ \\
    \midrule
    Case~2.a\newline(\S\ref{sec:case2a})
    & Evidence assembly\newline(Belief alignment)
    & Corollary~\ref{cor:case2a_floors}\par\smallskip
      $\mathcal F_{\mathrm{att}}(1,d,H,L)$\newline
      $\mathcal F_{\mathrm{SSM}}(1,d,H,d_{\mathrm{conv}})$, $d_{\mathrm{conv}}\le K$
    & Proposition~\ref{prop:case2a_pair}\par\smallskip
      $\mathcal F_{\mathrm{att}}(2,d,H,\widetilde L^\star)$, $Hd\ge2p$\newline
      $\mathcal F_{\mathrm{SSM}}(1,d,H,K+1)$, $Hd\ge p$ \\
    \midrule
    Case~2.b\newline(\S\ref{sec:case2b})
    & Addressing\newline(Addressing capacity)
    & Theorem~\ref{thm:case2b_floors}\par\smallskip
      $\mathcal F_{\mathrm{att}}(D,d,H,L)$\newline
      $\mathcal F_{\mathrm{SSM}}^{\mathrm{aff}}(D,d,H,d_{\mathrm{conv}})$\newline
      $\mathcal F_{\mathrm{SSM}}^{\mathrm{sel}}(D,d,H,d_{\mathrm{conv}})$, $DHd=o(\tfrac{R}{\log R})$
    & Proposition~\ref{prop:case2b_pair}\par\smallskip
      $\mathcal F_{\mathrm{att}}^{\mathrm{sm}}(2,d_{\mathrm{key}}\vee p,2,L^\star(T_\rho))$\newline
      $\mathcal F_{\mathrm{SSM}}^{\mathrm{sel}}(2,2R,p,2)$ \\
    \bottomrule
  \end{tabular}
\end{table}
\FloatBarrier
 
\paragraph{Guide to the proofs.}\label{app:proof_guide}
Boxes point to proof sections in the appendix, arrows show dependencies, and bold result numbers refer to the main text.
\par\smallskip
\noindent\makebox[\textwidth]{%
\begingroup
\definecolor{mapCaseOne}{HTML}{0072B2}
\definecolor{mapPositional}{HTML}{009E73}
\definecolor{mapContent}{HTML}{D55E00}
\definecolor{mapInk}{HTML}{333333}
\begin{tikzpicture}[
  x=1cm,y=1cm,
  every node/.style={font=\fontsize{9}{10.5}\selectfont,text=mapInk},
  box/.style={draw=black!35,line width=.45pt,rounded corners=1.5pt,
    align=center,inner xsep=4pt,inner ysep=4pt,fill=white},
  widebox/.style={box,text width=5.55cm,minimum height=.95cm},
  halfbox/.style={box,text width=2.65cm,minimum height=1.35cm},
  tallbox/.style={box,minimum height=1.8cm},
  proofedge/.style={-{Stealth[length=3.7pt,width=3pt]},draw=black!60,
    line width=.6pt,rounded corners=2pt},
  note/.style={font=\fontsize{8}{9}\selectfont,align=center,inner sep=2pt},
  casehead/.style={font=\fontsize{9}{10.5}\selectfont\bfseries,anchor=west},
]
  \fill[mapCaseOne!4,rounded corners=4pt] (.38,-.70) rectangle (6.7,-4.18);
  \fill[mapPositional!4,rounded corners=4pt] (.38,-5.55) rectangle (6.7,-12.1);
  \fill[mapContent!4,rounded corners=4pt] (7.25,-5.55) rectangle (13.57,-12.1);
  \node[casehead,text=mapCaseOne!65!black] at (.57,-.92)
    {Case 1: belief maintenance};
  \node[casehead,text=mapPositional!65!black] at (.57,-5.82)
    {Case 2.a: positional routing};
  \node[casehead,text=mapContent!70!black] at (7.44,-5.82)
    {Case 2.b: content routing};

  \node[widebox]
    (shared) at (6.95,-.125)
    {Whitened coordinates and regret\\
     \S\ref{app:case1_prediction_geometry}};
  \node[halfbox]
    (stationary) at (1.98,-2.05)
    {Stationary optimum\\
     \S\ref{app:case1_stationary_floor}\\
     \textbf{Theorem~\ref{thm:stationary_floor}}};
  \node[halfbox]
    (projection) at (5.10,-2.05)
    {Kernel-form projection\\
     \S\ref{app:case1_score_projection}};
  \node[widebox]
    (pair) at (3.54,-3.55)
    {Attention and SSM kernels\\
     \S\ref{app:case1_attention_kernels}--\ref{app:case1_stationary_pair}
     \quad \textbf{Proposition~\ref{prop:stationary_pair}}};
  \node[widebox]
    (adaptive) at (10.41,-3.55)
    {Extension to adaptive kernels\\
     \S\ref{app:case1_adaptive}};
  \node[box,text width=5cm,minimum height=.95cm]
    (routed) at (6.95,-4.85)
    {Belief geometry under routing\\
     \S\ref{app:vector_routing_geometry} \quad
     \textbf{Theorem~\ref{thm:routing_floor}}};

  \draw[proofedge] (shared.south -| 6.55,0) -- (6.55,-1.16)
    -- (1.98,-1.16) -- (stationary.north);
  \draw[proofedge] (5.10,-1.16) -- (projection.north);
  \draw[proofedge] (stationary.south)
    -- (stationary.south |- pair.north);
  \draw[proofedge] (projection.south)
    -- (projection.south |- pair.north);
  \draw[proofedge] (projection.east) -| (adaptive.north);
  \draw[proofedge] (shared.east) -- (13.57,0 |- shared.east)
    -- (13.57,0 |- routed.east) -- (routed.east);

  \node[widebox,anchor=north]
    (assembly) at (3.54,-6.4)
    {Assembly floors\\
     \S\ref{app:case2a_assembly_floors} \quad
     \textbf{Corollary~\ref{cor:case2a_floors}}};
  \node[halfbox,tallbox]
    (poslower) at (1.98,-8.65)
    {Lower bounds\\
     Attention (\S\ref{app:attention_context_floor})\\
     \textbf{Proposition~\ref{prop:attention_context_requirement}}\\[2pt]
     SSMs (\S\ref{app:case2a_ssm_full_class})};
  \node[halfbox,tallbox]
    (posconstruct) at (5.10,-8.65)
    {Constructions\\
     \S\ref{app:case2a_constructions}};
  \node[widebox]
    (posresults) at (3.54,-10.25)
    {Attention and SSM regret\\
     \S\ref{app:case2a_constructions}--\ref{app:case2a_ssm_full_class}
     \quad
     \textbf{Proposition~\ref{prop:case2a_pair}}};
  \node[widebox]
    (memory) at (3.54,-11.45)
    {Memory comparison\\
     \S\ref{app:case2a_memory_comparison}};

  \draw[proofedge] (posconstruct.south)
    -- (posconstruct.south |- posresults.north);
  \draw[proofedge] (poslower.south)
    -- (poslower.south |- posresults.north);
  \draw[proofedge] (posresults.south) -- (memory.north);
  \draw[proofedge] (poslower.west) -- (.38,-8.65)
    -- (.38,-11.45) -- (memory.west);

  \node[tallbox,text width=1.64cm,anchor=north]
    (capacity) at (8.28,-6.4)
    {Addressing\\capacity\\and regret\\
     \S\ref{app:case2b_capacity_regret}};
  \node[tallbox,text width=1.64cm,anchor=north]
    (capbounds) at (10.41,-6.4)
    {General\\capacity bounds\\
     \S\ref{app:case2b_capacity_geometry}};
  \node[tallbox,text width=1.64cm,anchor=north]
    (archgeometry) at (12.54,-6.4)
    {Architecture-\\specific\\geometry\\
     \S\ref{app:case2b_capacity_architectures}};
  \node[widebox]
    (addressing) at (10.41,-9.05)
    {Addressing floors\\
     \S\ref{app:case2b_capacity_theorem} \quad
     \textbf{Theorem~\ref{thm:case2b_floors}}};
  \node[widebox]
    (contentconstruct) at (10.41,-11.45)
    {Constructions\\
     \S\ref{app:case2b_vector_witnesses} \quad
     \textbf{Proposition~\ref{prop:case2b_pair}}};

  \draw[proofedge] (capacity.south)
    -- (capacity.south |- addressing.north);
  \draw[proofedge] (capbounds.south)
    -- (capbounds.south |- addressing.north);
  \draw[proofedge] (archgeometry.south)
    -- (archgeometry.south |- addressing.north);

  \draw[proofedge] (pair.east) -- (6.95,-3.55) -- (6.95,-4.24)
    -- (3.73,-4.24) -- (3.73,-5.47) -- (6.95,-5.47)
    -- (6.95,-7.55) -- (5.10,-7.55) -- (posconstruct.north);
  \draw[proofedge] (6.95,-7.55) -- (6.95,-10.5)
    -- (10.41,-10.5) -- (contentconstruct.north);
  \node[note,font=\fontsize{7.5}{8.5}\selectfont,fill=white,
    inner xsep=1.5pt,inner ysep=.75pt,rotate=90]
    at (6.95,-6.5) {reuse kernels};

  \draw[proofedge,preaction={draw=white,line width=2.5pt}]
    (routed.south -| 5.85,0) -- (5.85,0 |- assembly.north);
  \draw[proofedge] (routed.south -| 7.36,0) -- (7.36,-6.1)
    -- (8.28,-6.1) -- (capacity.north);

\end{tikzpicture}
\endgroup
}\par
\clearpage

\paragraph{Notation guide.}
\label{app:notation_guide}
Frequently used notation, grouped by role. References point to the definitions.

\begingroup
\small
\setlength{\tabcolsep}{4pt}
\renewcommand{\arraystretch}{1.12}
\noindent
\begin{tabular}{@{}L{0.25\textwidth}L{\dimexpr0.625\textwidth-4\tabcolsep\relax}L{0.125\textwidth}@{}}
\toprule
Symbol & Meaning & Defined in \\
\midrule
\multicolumn{3}{@{}l@{}}{\textbf{Data and routing}} \\
\addlinespace[3pt]
$\theta,\ p$
& Unknown regression parameter $\theta\in\mathbb R^p$ and its dimension.
& \S\ref{sec:setup} \\
\addlinespace[1pt]
$x_s,\ y_s,\ z_s$
& Regressor, response, and token $z_s=(x_{s+1},y_s,x_s)$; Case~2.b also includes keys.
& \S\ref{sec:setup}, \S\ref{sec:case2b} \\
\addlinespace[1pt]
$\Sigma_x,\ \Sigma_\theta,\ \sigma^2$
& Regressor covariance, prior covariance, and noise variance.
& \S\ref{sec:setup} \\
\addlinespace[1pt]
$T,\ T_\rho$
& Total time horizon; effective horizon (number of matched steps).
& \S\ref{sec:setup}, \S\ref{sec:case2_lb} \\
\addlinespace[1pt]
$\rho(s),\ K$
& Regressor index paired with $y_s$; positional delay $K$ gives $\rho(s)=s-K$ in Case~2.a.
& \S\ref{sec:setup}, \S\ref{sec:case2a} \\
\addlinespace[1pt]
$\mathcal M_t,\ N_t^\rho$
& Matched observation indices through time $t$ and their count.
& \S\ref{app:vector_routing_geometry} \\
\addlinespace[1pt]
$\psi_s,\ R,\ d_{\mathrm{key}}$
& Content key, number of distinct keys, and key-vector dimension (Case~2.b).
& \S\ref{sec:case2b} \\
\addlinespace[5pt]
\multicolumn{3}{@{}l@{}}{\textbf{Predictions and belief geometry}} \\
\addlinespace[3pt]
$g_t,\ g_t^{\mathrm{Bayes}}$
& Prediction of $y_{t+1}$ from $z_{1:t}$, and its Bayes-optimal counterpart.
& \S\ref{sec:setup} \\
\addlinespace[1pt]
$\hat\theta_t(g),\ \hat\theta_t^{\mathrm{Bayes}}$
& Implicit belief extracted from $g_t$ by projection, and posterior mean $\mathbb E[\theta\mid z_{1:t}]$.
& \S\ref{sec:setup} \\
\addlinespace[1pt]
$\mathrm{Reg}_T(g),\ \mathrm{Reg}_T(\mathcal F)$
& Cumulative excess squared loss over Bayes; the best achievable regret in class $\mathcal F$.
& \S\ref{sec:setup} \\
\addlinespace[1pt]
$\mathcal H_t^\rho,\ g_t^\perp$
& Belief-form prediction subspace; component of $g_t$ orthogonal to it.
& \S\ref{sec:setup} \\
\addlinespace[1pt]
$\mathcal K_t,\ \Pi_t$
& Kernel-form belief subspace and orthogonal projection onto it.
& \S\ref{sec:setup} \\
\addlinespace[1pt]
$\|\cdot\|_t,\ \|\cdot\|_{\Sigma_x},\ \mathrm{corr}_{\Sigma_x}^2$
& $L^2$ norm for predictions and whitened beliefs; $\Sigma_x$-weighted belief norm; belief alignment.
& \S\ref{sec:setup}, \S\ref{sec:case2_lb}, \S\ref{app:vector_routing_geometry} \\
\addlinespace[1pt]
$K_{t,s},\ \kappa_u$
& Matrix-valued aggregation-kernel weight on $x_{\rho(s)}y_s$; stationary kernels have $K_{t,s}=\kappa_{t-s}$.
& \S\ref{sec:setup}, \S\ref{sec:case1_lb} \\
\addlinespace[1pt]
$N_t,\ X_q,\ g_{t,q}$
& Number of candidate next keys, their latest regressors, and predictions under each candidate key (Case~2.b).
& \S\ref{sec:case2b} \\
\addlinespace[1pt]
$\mathcal C_t,\ \mathrm{Cap}_t(g)$
& Conditioning information making candidate Bayes predictions independent; capacity sums their squared conditional correlations with $g_{t,q}$.
& \S\ref{sec:case2b} \\
\addlinespace[5pt]
\multicolumn{3}{@{}l@{}}{\textbf{Architectural resources}} \\
\addlinespace[3pt]
$D,\ H,\ d,\ d_e$
& Depth, heads per layer, per-head width (SSM state width), and residual-stream width.
& \S\ref{sec:architectures} \\
\addlinespace[1pt]
$L,\ d_{\mathrm{conv}}$
& Attention window length and convolution tap count; a convolution reaches $d_{\mathrm{conv}}-1$ steps back.
& \S\ref{sec:architectures} \\
\addlinespace[1pt]
$L^\star(T),\ \widetilde L^\star$
& Optimal Case~1 attention window; $\widetilde L^\star=(K+1)\vee L^\star(T_\rho)$ also covers positional reach.
& \S\ref{sec:case1_lb}, \S\ref{sec:case2a} \\
\addlinespace[5pt]
\multicolumn{3}{@{}l@{}}{\textbf{Coordinates and statistics in the proofs}} \\
\addlinespace[3pt]
$\widetilde x_t,\ \widetilde\theta$
& Whitened variables: $\widetilde x_t=\Sigma_x^{-1/2}x_t$ and $\widetilde\theta=\Sigma_x^{1/2}\theta$.
& Eq.~\eqref{eq:vector_whitening} \\
\addlinespace[1pt]
$\widetilde\Sigma_\theta,\ U,\ \lambda_i$
& Whitened prior covariance $\Sigma_x^{1/2}\Sigma_\theta\Sigma_x^{1/2}$, its eigenbasis, and its eigenvalues.
& Eq.~\eqref{eq:vector_whitening} \\
\addlinespace[1pt]
$\widetilde\theta_t(g),\ \widetilde\theta_t^{\mathrm{Bayes}}$
& Whitened beliefs: $\Sigma_x^{1/2}\hat\theta_t(g)$ and $\Sigma_x^{1/2}\hat\theta_t^{\mathrm{Bayes}}$, respectively.
& \S\ref{app:case1_prediction_geometry} \\
\addlinespace[1pt]
$K_u,\ M_t$
& Whitened stationary kernel $K_u=\Sigma_x^{1/2}\kappa_u\Sigma_x^{1/2}$ and kernel mass $M_t=\sum_{u<t}K_u$ (Case~1).
& Eq.~\eqref{eq:vector_stationary_estimator} \\
\addlinespace[1pt]
$\Gamma,\ v_i$
& Evidence-innovation covariance and its eigenvalues $v_i=\lambda_i+\sum_j\lambda_j+\sigma^2$.
& Eq.~\eqref{eq:vector_covariance_scales} \\
\addlinespace[1pt]
$\zeta,\ \bar\zeta$
& Bias--variance scales: $\zeta=\sum_i\sqrt{\lambda_iv_i}$ and $\bar\zeta=\sqrt{(\sum_i\lambda_i)(\sum_i v_i)}$.
& \S\ref{sec:case1_lb} \\
\addlinespace[1pt]
$A_t,\ B_t$
& Accumulated evidence $\sum_{s\le t}x_sy_s$ and accumulated information $\sum_{s\le t}x_sx_s^\top$ (Case~1).
& \S\ref{app:case1_prediction_geometry} \\
\addlinespace[1pt]
$j,\ n,\ Y_j,\ r_j$
& Case~2.a: $j=s-K$, $n=t-K$; matched response $Y_j=y_{K+j}$ and evidence atom $r_j=\widetilde x_jY_j$.
& \S\ref{app:case2a_vector} \\
\bottomrule
\end{tabular}\par
\endgroup
 \clearpage

\section{Extended related work}\label{app:related}
This appendix expands the comparisons of \S\hyperref[par:intro_related_work]{\ref*{sec:intro}} result by result,
following the three capabilities of our belief geometry, and then places our
analysis relative to other expressivity axes, unifying frameworks, and
optimization results.

\paragraph{Belief maintenance and in-context regression (Case~1).}
Our testbed descends from the linear-regression formulation of in-context
learning of \citet{gargWhatCanTransformers2023}, and the theory that followed
asks which estimator a model can implement: transformers can be constructed to
run gradient descent and ridge regression \citep{akyurekWhatLearningAlgorithm2023},
one layer of linear self-attention performs one gradient step
\citep{oswaldTransformersLearnIncontext2023}, and deeper constructions implement
ridge, Lasso, and generalized linear estimators with in-context algorithm
selection \citep{baiTransformersStatisticiansProvable2023}. Statistical
treatments bound the risk at a fixed context length
\citep{wuHowManyPretraining2024,wakayamaInContextLearningProvably2025}. On the
recurrent side, the same gradient step has a one-layer recurrent implementation
\citep{tianLearningRecurrentState2026}. Closest to belief maintenance itself,
transformers can implement the Kalman filter and recurrent networks can
approximate Bayesian filters, both with time-uniform error bounds
\citep{goelCanTransformerRepresent2024,bishopRecurrentNeuralNetworks2023}, while
linear recurrences efficiently approximate a temporal relationship only when
its dependence on the distant past decays sufficiently
\citep{liApproximationOptimizationTheory2022}.

Prior work thus establishes implementability or bounds the loss at a fixed prediction step. Our
belief-maintenance analysis instead asks what maintaining the Bayes belief
costs over the stream, comparing both families under cumulative Bayes regret.
Our stationary-floor theorem (Theorem~\ref{thm:stationary_floor})
lower-bounds every stationary kernel by $\zeta\sqrt T+o(\sqrt T)$, and
Proposition~\ref{prop:stationary_pair}
shows that once $Hd\ge p$, the SSM's exponential kernel attains this floor while
attention's uniform kernel pays the factor $\tfrac{2}{\sqrt3}\,\bar\zeta/\zeta$.
Among the closest architecture lower bounds, that of
\citet{coleInContextLearningLinear2025} for a single-layer linear class is a
non-vanishing terminal loss under a parameter-norm constraint; our analysis
concerns cumulative regret and identifies the exact leading $\sqrt T$ constant.

\paragraph{Evidence assembly and positional routing (Case~2.a).}
Positional evidence assembly requires composing operations across tokens:
locating the relevant past input and combining it with the current response.
Related work examines how attention depth and convolution in SSMs enable such
interactions. For attention, one layer
cannot solve the induction-heads task unless its size grows with the sequence,
while two layers suffice \citep{sanfordOnelayerTransformersFail2024}; two layers
with one head each represent any conditional $k$-gram
\citep{ekboteWhatOneCannot2025}; a single nonlinear attention layer is
asymptotically Bayes optimal for regression on one latent location
\citep{marionAttentionLayersProvably2025}; and logarithmic depth separates from
a single layer on dynamical systems \citep{coleInContextLearningLinear2025}. On
the recurrent side, positional reach comes from short convolutions: H3
introduced its shift operator because state-space models struggled to recall
earlier tokens \citep{fuHungryHungryHippos2023}, and fixed convolution filters
support only a subset of token-interaction distances
\citep{aroraZoologyMeasuringImproving2023}. For Markov-chain prediction,
\citet{bondaschiMarkovLaplaceHow2026} show that a simplified one-layer Mamba
with width-$2$ convolution and an $\ell_1$-normalized readout exactly represents
the Bayes-optimal predictor for their first-order Markov model.

The cited attention results establish depth requirements and prediction
guarantees for specific tasks, while the recurrent constructions identify how
convolution enables useful predictors. Our analysis connects these mechanisms
through a common cumulative-regret criterion.
Corollary~\ref{cor:case2a_floors} quantifies the cost of insufficient evidence
assembly as linear regret in both architecture classes. Beyond establishing
that assembly is possible, Propositions~\ref{prop:case2a_pair}
and~\ref{prop:attention_context_requirement} compare the memory needed for
$\Theta(\sqrt{T_\rho})$ regret: the SSM separates local positional reach from
belief maintenance, whereas fixed-depth attention in our classes requires a
growing context to match this rate. Under standard per-position
caching, this yields an SSM memory advantage when $K=o(\sqrt{T_\rho})$.

\paragraph{Content addressing (Case~2.b).}
Content addressing requires selecting the relevant information from a model's
history. Induction heads illustrate this content-based selection: after observing $[A][B]$,
they predict $[B]$ when $[A]$ reappears. \citet{olssonIncontextLearningInduction2022}
identified these circuits as a mechanism behind in-context learning.
Mechanistic accounts describe associative retrieval through fast weights in
linear attention \citep{schlagLinearTransformersAre2021} and through modern
Hopfield networks in softmax attention
\citep{ramsauerHopfieldNetworksAll2021}.
\citet{aroraZoologyMeasuringImproving2023} trace most of the observed
language-modeling gap between attention and gated-convolution models to
failures to recall earlier context and introduce multi-query associative
recall to study this capability. Related experiments show that Mamba trails
transformers on retrieval while matching them on regression
\citep{parkCanMambaLearn2024}. Hybrid models such as Griffin and Jamba combine
recurrence with attention, improving retrieval and few-shot performance over
their recurrent-only baselines
\citep{deGriffinMixingGated2024,lieberJambaHybridTransformerMamba2024}.

These empirical differences have theoretical counterparts. For sparse
averaging, attention needs width
logarithmic in the sequence while any recurrent model needs $\Omega(N)$ bits of
state \citep{sanfordRepresentationalStrengthsLimitations2023}. For $k$-hop
induction, depth $\lfloor\log_2 k\rfloor+2$ suffices, whereas recurrent models
need depth $k$ or polynomial width, with similar bounds for kernel-based and
sparse attention \citep{sanfordTransformersParallelComputation2024}. Two-layer
transformers copy strings of exponential length while fixed-state models cannot
\citep{jelassiRepeatMeTransformers2024}; recurrent models with sublinear memory
cannot retrieve even with chain of thought \citep{wenRNNsAreNot2024}; and any
recurrent model needs $\Omega(N)$ bits for exact associative recall
\citep{aroraSimpleLinearAttention2025}.
\citet{mousavi-hosseiniWhenTransformersOutperform2025} also establish a
sample-complexity advantage for attention over recurrent networks on
sparse-token regression.

Our addressing-capacity analysis differs from these results in three respects.
First, at unbounded precision, the bit-based memory bounds above are vacuous
as state-dimension bounds and do not distinguish transition mechanisms;
Theorem~\ref{thm:case2b_floors} instead bounds real-valued state dimension
with no finite-precision assumption and separates the mechanisms at fixed
depth: linear attention incurs linear regret at any width and context length,
fixed and affinely gated transitions incur it at every state width,
sigmoid-gated transitions require $DHd=\Omega(R/\log R)$, and our sigmoid-gated
construction succeeds with $\Theta(R)$ recurrent-state dimension
(Proposition~\ref{prop:case2b_pair}). Second, the copying and associative-recall
tasks above are standalone and discrete, scored by exact success; our addressing
problem is embedded in noisy Bayesian
inference, alongside evidence assembly and belief maintenance, and scored by
regret. Third, sparse averaging and sparse-token regression hand the model the
relevant positions explicitly, whereas our key must be matched against $R$
near-orthogonal codewords of width $\Theta(\log R)$, the axis on which softmax
attention and sigmoid-selective SSMs separate.

\paragraph{Expressivity on other axes.}
Other expressivity results ask which sequential computations an architecture
can perform, such as tracking a changing state or computing parity---whether
a binary sequence contains an odd or even number of ones.
S4, Mamba, and transformers lie in $\mathrm{TC}^0$ and cannot solve general
state-tracking tasks \citep{merrillIllusionStateStateSpace2025}.
The transition spectrum also matters: linear recurrences whose transition
matrices have only positive eigenvalues cannot compute parity, whereas allowing
negative eigenvalues makes this possible
\citep{grazziUnlockingStateTrackingLinear2025}.
Gating provides another source of expressive power, enabling selective SSMs to
represent functions beyond the capabilities of linear transformers
\citep{cironeTheoreticalFoundationsDeep2024,cohen-karlikExpressivitySelectiveStateSpace2025}.
These results characterize computational capabilities such as state tracking
and parity. Our analysis instead quantifies the cumulative-regret and resource
costs of assembling evidence, maintaining beliefs, and addressing past
information.

\paragraph{Unifying frameworks.}
A complementary approach to comparing attention and SSMs is to identify the
mathematical structure they share. Linear attention admits a recurrent form
\citep{katharopoulosTransformersAreRNNs2020}, and selective SSMs and attention
are dual under structured state-space duality
\citep{daoTransformersAreSSMs2024}. The dynamical-systems framework of
\citet{sieberUnderstandingDifferencesFoundation2024} rewrites attention, SSMs,
and RNNs as one recurrence and identifies S6's input projections with
attention's keys and queries. Test-time learning provides another common
description: \citet{sunLearningLearnTest2025} interpret
sequence-modeling layers as test-time learners, recovering linear and softmax
attention as special cases, while \citet{wangTesttimeRegressionUnifying2025}
derive these layers as test-time regressors, with softmax attention as kernel
smoothing.

These frameworks unify architectural representations. Our belief
geometry instead unifies their analysis, translating architectural differences
into task-dependent regret and resource separations.
Appendix~\ref{app:arch} details how the corresponding components of the practical architectures enter our
formal classes.

\paragraph{Optimization.}
Our analysis is representational: it asks what each class can realize, not
what gradient descent finds. The optimization literature is complementary,
showing that related mechanisms are learnable in simplified settings. For
linear attention, the one-step estimator is the global minimizer of the
pretraining loss \citep{mahankaliOneStepGradient2023} and gradient flow reaches
it \citep{zhangTrainedTransformersLearn2023}; softmax attention converges in
phases and allocates heads to tasks
\citep{huangInContextConvergenceTransformers2023,chenTrainingDynamicsMultiHead2024};
induction heads emerge through distinct training phases
\citep{biettiBirthTransformerMemory2023,nichaniHowTransformersLearn2024,edelmanEvolutionStatisticalInduction2024,chenUnveilingInductionHeads2024};
and Mamba trained by gradient descent implements a variant of online gradient
descent on linear regression \citep{jiangTrainedMambaEmulates2025}. Optimization
guarantees also reveal a trade-off with memory: for diagonal recurrent networks,
long-term dependencies can make gradient-descent guarantees grow exponentially
with the sequence length \citep{cayciConvergenceGradientDescent2024}, whereas
for shallow attention the optimization error is independent of the sequence
length, at the price of memory that grows with it
\citep{ardaFiniteTimeAnalysisGradient2026}, echoing, from an optimization
perspective, the memory trade-off we establish in Case~2.a.
 
\section{Beyond the linear-Gaussian testbed}
\label{app:conjugate_belief_geometry}

Our three-operation viewpoint (evidence assembly, belief maintenance, addressing) in \S\ref{sec:setup} extends beyond linear-Gaussian regression,
and so does the belief geometry behind our regret analysis. Three
examples below show how the operations persist as posterior updates and
prediction rules change. We then identify when kernel form, the
Pythagorean decomposition, and the two-factor floor carry over directly,
and sketch how approximations can extend them further.

\paragraph{General formulation.}
Retain the online token layout and
partial-injective routing of \S\ref{sec:setup}, including supplied keys, and evaluate predictors by
the same squared-loss Bayes regret.
Allow an arbitrary prior $q_0$ on $\theta$ and an arbitrary iid regressor
distribution, with the parameter, regressors, and routing mutually independent.
Responses are conditionally independent given these quantities,
with density or mass function
$p_\theta(y_s\mid x_{\rho(s)})$ and finite second moments.

Write $q_t$ for the posterior distribution of $\theta$ given $z_{1:t}$ and
$m_\theta(x):=\EE[y_s\mid x_{\rho(s)}=x,\theta]$.
Bayes' rule and conditional expectation give, at matched steps,
\[
  q_t(d\theta)
  \propto
  p_\theta(y_t\mid x_{\rho(t)})q_{t-1}(d\theta),
  \qquad
  g_t^{\mathrm{Bayes}}
  =
  \int m_\theta(x_{\rho(t+1)})\,q_t(d\theta).
\]
Unmatched observations leave the posterior unchanged, and unmatched
targets have prediction zero.

Evidence assembly forms the likelihood contribution
$p_\theta(y_t\mid x_{\rho(t)})$ from the matched pair
$(x_{\rho(t)},y_t)$. Belief maintenance updates $q_{t-1}$ to $q_t$
and preserves the posterior or a sufficient summary for subsequent
updates and predictions. Addressing locates $x_{\rho(t+1)}$ before
its response is observed, so that the maintained belief is applied
to the correct regressor when forming $g_t^{\mathrm{Bayes}}$.

\paragraph{Conjugate families.}
Let $\{q_\eta\}$ be a finite-dimensional conjugate exponential family
with natural parameter $\eta$.
In natural parameters, multiplying the prior by a likelihood factor adds
the likelihood's contribution to the prior parameters. Denote this
complete sufficient-statistic increment by $S(x,y)$.
Starting from $q_0=q_{\eta_0}$, conjugacy gives
\[
  q_t=q_{\eta_t},
  \qquad
  \eta_t
  =\eta_{t-1}+S(x_{\rho(t)},y_t)
  =\eta_0+\sum_{s\leq t}S(x_{\rho(s)},y_s).
\]
Here and below, sums include only matched steps, and the recursive update
applies when $t$ is matched. This connects directly to our kernel form:
the natural parameters $\eta_t$ aggregate the sufficient-statistic increments
$S(x_{\rho(s)},y_s)$. Evidence assembly forms
$S(x_{\rho(t)},y_t)$, belief maintenance accumulates $\eta_t$, and
addressing supplies $x_{\rho(t+1)}$ for prediction under $q_{\eta_t}$.
For Gaussian regression,
$S(x,y)=(\sigma^{-2}xy,\sigma^{-2}xx^\top)$:
the $xy$ term is our evidence atom, while $xx^\top$ contributes
to posterior precision and enters the Bayes kernel in
Equation~\eqref{eq:bayes_kernel}.

The following three examples make these operations explicit and the final paragraph discusses when and how our regret analysis extends.

\paragraph{Example 1 (conjugate): Beta--Bernoulli regression.}
\label{app:conjugate_bernoulli}
Let the prior coordinates $\theta_j\sim\operatorname{Beta}(1,1)$ be
independent, and take iid regressors uniform on the standard basis vectors
$e_1,\ldots,e_p$. Keep the same routing and let
\[
  y_s\mid\theta,x_{\rho(s)}
  \sim\mathrm{Bern}(x_{\rho(s)}^\top\theta).
\]
With the accumulated statistics
\(
  N_t:=\sum_{s\leq t}x_{\rho(s)}
\) 
and 
\(
  A_t:=\sum_{s\leq t}x_{\rho(s)}y_s
\),
the centered Bayes belief has the kernel form
\[
  \hat\theta_t^{\mathrm{Bayes}}-\tfrac12\mathbf1
  =\sum_{s\leq t}K_{t,s}^{\mathrm{Bayes}}
    x_{\rho(s)}(y_s-\tfrac12),
  \qquad
  K_{t,s}^{\mathrm{Bayes}}
  :=\operatorname{diag}\!\left(\frac{1}{2+N_{t,j}}\right)_{j=1}^p,
\]
where $\mathbf1$ is the all-ones vector. And the Bayes predictor is
\[
  g_t^{\mathrm{Bayes}}
  =x_{\rho(t+1)}^\top\hat\theta_t^{\mathrm{Bayes}}.
\]
The three operations therefore have the same structure as in our linear-Gaussian setting.
Evidence assembly forms the centered evidence atoms
$x_{\rho(s)}(y_s-\tfrac12)$. Belief maintenance aggregates them
through $K_{t,s}^{\mathrm{Bayes}}$ to obtain the centered Bayes belief
$\hat\theta_t^{\mathrm{Bayes}}-\tfrac12\mathbf1$.
Addressing supplies $x_{\rho(t+1)}$ for prediction.
Thus, this non-Gaussian example retains all three operations, the kernel form,
and linear prediction form, although with different evidence atoms and a different aggregation kernel.
App.~\ref{app:experiments_bernoulli} tests our architectural lessons on this model
under the Case~2 routings.

\paragraph{Example 2 (conjugate): lognormal regression.}
Take $p=1$, a prior $\theta\sim\mathcal N(0,\nu)$ with $\nu>0$,
and iid regressors uniform on $[-1,1]$. Keep the same routing and let
\[
  \log y_s=x_{\rho(s)}\theta+\epsilon_s,
  \qquad
  \epsilon_s\stackrel{\mathrm{iid}}{\sim}\mathcal N(0,\sigma^2),
\]
with $\sigma^2>0$ and the noises independent of the parameter, regressors,
and routing. 
Completing the square gives
$q_t=\mathcal N(\hat\theta_t^{\mathrm{Bayes}},v_t)$, where
$v_t^{-1}=\nu^{-1}+\sigma^{-2}\sum_{s\leq t}x_{\rho(s)}^2$.
The Bayes belief has kernel form
\[
  \hat\theta_t^{\mathrm{Bayes}}
  =\sum_{s\leq t}K_{t,s}^{\mathrm{Bayes}}
    x_{\rho(s)}\log y_s,
  \qquad
  K_{t,s}^{\mathrm{Bayes}}:=\sigma^{-2}v_t.
\]
This is the same scalar Bayes kernel as in
Equation~\eqref{eq:bayes_kernel}, now applied to transformed evidence atoms.
Evidence assembly forms different evidence atoms
$x_{\rho(s)}\log y_s$. Belief maintenance aggregates them through
$K_{t,s}^{\mathrm{Bayes}}$ to obtain $\hat\theta_t^{\mathrm{Bayes}}$.
Averaging over the posterior and the next observation noise gives
\[
  g_t^{\mathrm{Bayes}}
  =\exp\!\left(
    x_{\rho(t+1)}\hat\theta_t^{\mathrm{Bayes}}
    +\frac{\sigma^2+x_{\rho(t+1)}^2v_t}{2}
  \right).
\]
Addressing still supplies $x_{\rho(t+1)}$ for prediction. Thus, the three operations and the kernel form are retained but the Bayes prediction is a nonlinear function of the belief parameters.

\paragraph{Example 3 (nonconjugate): logistic regression.}
\label{app:nonconjugate_logistic}
Keep the prior $\theta\sim\mathcal N(0,\nu)$, regressors, and routing of
Example~2, but let matched responses $y_s\in\{0,1\}$ be Bernoulli with
$p_\theta(1\mid x)=(1+e^{-x\theta})^{-1}$.
Multiplying the likelihood contributions gives
\[
  q_t(d\theta)
  \propto
  \exp\!\left\{
    \theta\sum_{s\leq t}x_{\rho(s)}y_s
    -\sum_{s\leq t}\log(1+e^{x_{\rho(s)}\theta})
  \right\}q_0(d\theta).
\]
Here, evidence assembly still combines each $y_s$ with its generating regressor
$x_{\rho(s)}$. Belief maintenance incorporates both the sum
$\sum_{s\leq t}x_{\rho(s)}y_s$ and the regressor-dependent logarithmic
terms. Unlike Examples~1 and~2, this update does not directly yield
a kernel-form Bayes belief.
Addressing supplies $x_{\rho(t+1)}$, and averaging
$p_\theta(1\mid x_{\rho(t+1)})$ over $q_t$ gives
$g_t^{\mathrm{Bayes}}$.
The three operations remain, even though the Gaussian posterior update
does not.
The main-text Case~2 experiments use a $p$-dimensional version of this model
(App.~\ref{app:experiments_practical}).

\paragraph{Extending our regret analysis.}
The three operations apply to every prior and likelihood in the general
formulation above. How much of our regret analysis extends with them
depends on the structure of the Bayes prediction and of the posterior
update.

When the Bayes prediction is linear in the Bayes belief through a fixed,
finite-dimensional feature map~$\varphi$ of the addressed regressor,
$g_t^{\mathrm{Bayes}}=\varphi(x_{\rho(t+1)})^\top\hat\theta_t^{\mathrm{Bayes}}$,
as in Example~1 with $\varphi(x)=x$, the projection onto the belief
subspace (Figure~\ref{fig:belief_geometry}) extends directly, with $\varphi(x)$ in place of $x$ and
$\EE[\varphi(x)\varphi(x)^\top]$ in place of $\Sigma_x$. So does the
two-factor floor of Theorem~\ref{thm:routing_floor}, which rests on this
projection alone.

The kernel-form refinement depends instead on the posterior update. When
the posterior admits a finite-dimensional sufficient statistic that
accumulates additively, as in conjugate families, belief maintenance
reduces to aggregating assembled increments, which mirrors the kernel
form of Equation~\eqref{eq:bayes_kernel} and is exactly what happens
in Examples~1 and~2. With both properties, as in Example~1, the full
Pythagorean decomposition in Equation~\eqref{eq:belief_pythagoras}
extends directly, and with it our regret-analysis framework.

Beyond these conditions, kernel form \eqref{eq:bayes_kernel} and the decomposition \eqref{eq:belief_pythagoras} still provide
a starting point for further analysis.
A nonlinear prediction without such features, as in Example~2, can be
treated as a separate readout of a kernel-form belief. Linearizing the
readout around the Bayes belief parameters makes the excess loss locally
quadratic in the belief mismatch, and our projection decomposition then holds
in the sensitivity-weighted geometry.
A general posterior update, as in Example~3, can be approximated by one
that retains kernel aggregation. Expanding the logarithmic terms to
second order around the prior mean $\theta=0$ gives a Gaussian
approximation whose mean has the kernel form of
Equation~\eqref{eq:bayes_kernel}, with $\sigma^2=4$ and pseudo-responses
$4(y_s-\tfrac12)$ in place of $y_s$.
Combined with the readout linearization, the decomposition then applies
to this approximate predictor, up to the errors of the posterior
approximation and the readout linearization.
 
\section{Architecture classes: practical connections and extensions}\label{app:arch}
Section~\ref{sec:architectures} defines our architecture classes. Here we
explain how their components connect to practical architectures and specify the
count-normalized architecture extensions used in the adaptive results of
Appendix~\ref{app:case1_adaptive}.

\subsection{Connections to practical architectures}
\label{app:practical_architecture_connections}

Table~\ref{tab:architecture_connections} collects the architectural
references discussed below.

\begin{table}[!ht]
  \centering
  \caption{Architectural references.}
  \label{tab:architecture_connections}
  \small
  \setlength{\tabcolsep}{12pt}
  \renewcommand{\arraystretch}{1.15}
  \begin{tabular}{@{}ll@{}}
    \toprule
    Architecture & Reference \\
    \midrule
    Transformer & \citet{vaswaniAttentionAllYou2017} \\
    LLaMA & \citet{touvronLLaMAOpenEfficient2023} \\
    \midrule
    S4 & \citet{guEfficientlyModelingLong2022} \\
    H3 & \citet{fuHungryHungryHippos2023} \\
    RetNet & \citet{sunRetentiveNetworkSuccessor2023} \\
    Mamba & \citet{guMambaLinearTimeSequence2024} \\
    Mamba-2 & \citet{daoTransformersAreSSMs2024} \\
    mLSTM & \citet{beckXLSTMExtendedLong2024} \\
    GLA & \citet{yangGatedLinearAttention2024} \\
    DeltaNet & \citet{yangParallelizingLinearTransformers2025} \\
    Gated DeltaNet & \citet{yangGatedDeltaNetworks2025} \\
    RWKV-7 & \citet{pengRWKV7GooseExpressive2025} \\
    \bottomrule
  \end{tabular}
\end{table}
\FloatBarrier

\paragraph{Attention.}
Our attention classes retain the Transformer's standard query--key scores,
value aggregation, multi-head structure, output projections, and residual
connections. The additive positional bias
$p_{t-s}$ is a tractable surrogate for the relative-lag dependence induced
by RoPE in LLaMA-type Transformers.

\paragraph{SSMs: convolution, writes, and reads.}
Our SSM convolution corresponds to the short causal convolution that Mamba applies
before its recurrence.
Multiplicative writes and reads are shared by H3 and Mamba:
writes form products of input features, while reads combine the stored state
with current input features. Our formulation retains these multiplicative write/read
operations, using affine feature projections and omitting additional nonlinear
processing after the state read. Our vector-state formulation with coordinatewise
multiplication also accommodates outer-product writes and matrix--vector reads.
Indeed, index the recurrent state coordinates by matrix entries,
$h_{t,(i,j)}=M_{t,ij}$. Repeating the affine projection rows gives
\[
u_{t,(i,j)}=v_{t,i}k_{t,j},
\qquad
o_{t,(i,j)}=M_{t,ij}q_{t,j}.
\]
The output projection sums over $j$, recovering $(M_tq_t)_i$, with one
recurrent coordinate per matrix entry.

\paragraph{SSMs: transitions.}
Fixed transitions support retention and decay, as in S4, H3, and RetNet.
Scalar gates control retention across a head, as in Mamba-2's
scalar-identity transition and mLSTM's scalar forget gate.
Diagonal gates allow coordinatewise retention, as in Mamba and GLA.
The delta rule makes a rank-one correction along the current key,
following DeltaNet.
Gated DeltaNet combines decay with such a correction.
Our generalized delta form follows RWKV-7's diagonal-plus-rank-one update,
allowing separate input and output directions.
Our classes retain these transition mechanisms, using affine projections to
form gate inputs and rank-one update vectors.

Our experiments use the practical architectures described in \S\ref{sec:architectures},
retaining RMSNorm in both families, SwiGLU MLPs in the Transformer, and
additional write/read nonlinearities in Mamba-2.

\subsection{Count-normalized extensions}
\label{app:count_normalized_extensions}

We use count normalization as a tractable surrogate for softmax normalization
in attention and gated normalization in recurrent architectures such as mLSTM.
It scales accumulated evidence by a deterministic token count, while learned
parameters remain fixed throughout each stream.
We use these extensions for the adaptive kernel constructions in
Appendix~\ref{app:case1_adaptive} and show that the assembly and addressing
floors persist under count normalization in Appendix~\ref{app:case2}.
We specify a deterministic layerwise schedule
$\mathbf c=(c_t^{(1)},\ldots,c_t^{(D)})$.
For $D=1$, we write $c_t$ instead of $\mathbf c$.

\paragraph{Attention.}
Let $\tau_t:=|J_t|=\min(t,L)$. Each layer uses
$c_t^{(r)}\in\{1,\tau_t\}$ and replaces its linear weights by
\[
  a_{t,s}^{(r,b)}:=\frac{s_{t,s}^{(r,b)}}{c_t^{(r)}}.
\]
These weights define
$\mathcal F_{\mathrm{att}}^{\mathrm{cal}}(D,d,H,L;\mathbf c)$.
The softmax-enabled class
$\mathcal F_{\mathrm{att}}^{\mathrm{sm,cal}}(D,d,H,L;\mathbf c)$
additionally permits any layer with $c_t^{(r)}\equiv1$ to use
softmax weights.

\paragraph{State-space models.}
Each layer uses $c_t^{(r)}\in\{1,t\}$, and only its ordinary read changes:
\[
  o_{t,\mathbf c}^{(r,b)}
  :=\frac{(W_Q^{(r,b)}\widetilde e_t^{(r,b)})\odot h_t^{(r,b)}}
          {c_t^{(r)}}.
\]
Using these reads defines
$\mathcal F_{\mathrm{SSM}}^{\mathrm{cal}}
(D,d,H,d_{\mathrm{conv}};\mathbf c)$ for fixed transitions and
$\mathcal F_{\mathrm{SSM}}^{\mathrm{sel,cal}}
(D,d,H,d_{\mathrm{conv}};\mathbf c)$ for selective-first, fixed-later
transitions. The same read-only normalization applies to
$\mathcal F_{\mathrm{SSM}}^{\mathrm{aff}}$.

For every class, the all-one schedule $\mathbf c=\mathbf1$ recovers its
base counterpart exactly. Count normalization adds no learned parameters
and leaves recurrent-state dimension, residual width, and access to past
tokens unchanged. It normalizes by a
deterministic count, unlike the data-dependent normalization in softmax,
LayerNorm, or RMSNorm.
 
\section{Additional experiments and experimental details}\label{app:experiments}
We first give full protocols and supplementary results for the practical
experiments reported in the main text (\S\ref{app:experiments_practical}).
We then confirm the Case~2 architectural lessons for the same practical
models on conjugate Beta--Bernoulli bandits
(\S\ref{app:experiments_bernoulli}), complementing the nonconjugate tasks of
the main text. Finally, we present additional experiments that test the exact theoretical predictions within our analytically tractable classes
(\S\ref{app:experiments_analytical}).

\paragraph{Code availability.}\label{app:code_availability}
The code, experiment configurations, plotting data, and reproduction
instructions are available in an anonymous repository:
\begin{center}
  \url{https://anonymous.4open.science/r/belief-routing-experiments-2EF4/}
\end{center}

\subsection{Details of the main-text experiments}
\label{app:experiments_practical}

This section details the practical experiments reported in the main text.
Section~\ref{app:experiments_case1_practical} examines kernel alignment in
Case~1 (\S\ref{sec:case1_lb}).
Section~\ref{app:experiments_common_evaluation} introduces the common setup
and evaluation protocol for the Case~2 heatmaps. We then cover positional
assembly in \S\ref{app:experiments_case2a_practical}
(Case~2.a, \S\ref{sec:case2a}) and selectivity and width scaling in
\S\ref{app:experiments_case2b_practical} (Case~2.b, \S\ref{sec:case2b}).

\subsubsection{Case 1: kernel alignment}
\label{app:experiments_case1_practical}

We detail the kernel-alignment experiments in \S\ref{sec:case1_lb}
(Figure~\ref{fig:case1_kernel_alignment}).

\paragraph{Tasks and tokens.}
We test whether practical models more readily learn the Bayes kernel
favored by their stationary counterparts. Each episode contains $T=128$
scalar observations, preceded by a single ordinary zero token. The model receives the scalar sequence $0,o_1,\ldots,o_T$.
The initial zero makes averaging the first $t+1$ tokens yield
$\sum_{s=1}^t o_s/(t+1)$, the posterior mean in the uniform regime.

In the uniform regime, the posterior mean is
$\hat\theta_t^{\mathrm{Bayes}}:=\EE[\theta\mid o_{1:t}]$:
\[
  \theta\sim\mathcal N(0,1),\qquad
  o_t=\theta+\epsilon_t,\quad
  \epsilon_t\stackrel{\mathrm{iid}}{\sim}\mathcal N(0,1),\qquad
  \hat\theta_t^{\mathrm{Bayes}}=\frac{\sum_{s=1}^t o_s}{t+1}.
\]
In the exponential regime,
\[
  \theta_1\sim\mathcal N(0,1),\qquad
  \theta_{t+1}=\theta_t+\xi_t,\quad
  \xi_t\stackrel{\mathrm{iid}}{\sim}\mathcal N(0,0.1),\qquad
  o_t=\theta_t+\epsilon_t,\quad
  \epsilon_t\stackrel{\mathrm{iid}}{\sim}\mathcal N(0,9).
\]
The posterior mean
$\hat\theta_t^{\mathrm{Bayes}}:=\EE[\theta_t\mid o_{1:t}]$ satisfies the exact
recursion $\hat\theta_0^{\mathrm{Bayes}}=0$ and
$\hat\theta_t^{\mathrm{Bayes}}=0.9\hat\theta_{t-1}^{\mathrm{Bayes}}+0.1o_t$.
The terminal Bayes aggregation kernels are therefore
$\kappa_u^{\mathrm{Bayes}}=1/(T+1)$ and
$\kappa_u^{\mathrm{Bayes}}=0.1(0.9)^u$, respectively, for $0\leq u<T$.
The model output $g_t$ estimates the posterior mean at all $128$
observation times in each episode.

\paragraph{Models and optimization.}
We use the practical counterparts described in \S\ref{sec:architectures}.
Both models have a single pre-norm layer with residual width $32$
(\texttt{d\_model}) and a
scalar linear readout. The Transformer uses four heads of width $8$, full
context $L=129$, RoPE with base $10{,}000$ and unscaled positions, and an
MLP of width $64$. Mamba-2 uses expanded width $64$ (\texttt{d\_inner}),
state width $16$, head dimension $16$, one state group, convolution width $4$,
and chunk size $64$.

Every run uses $450$ AdamW updates with fresh batches of $64$ episodes,
peak learning rate $10^{-3}$ with $5\%$ linear warmup and cosine decay to
zero, $\beta_2=0.95$, weight decay $0.1$, default remaining optimizer
arguments, and global gradient-norm clipping at $5$. The loss is mean squared error between
$g_t$ and $\hat\theta_t^{\mathrm{Bayes}}$, averaged over episodes and
observation times; the reported model is the final iterate. We use $20$
paired seeds. For each task and seed, all architecture and
initialization conditions receive the same procedural training batches and
the same independent evaluation set of $8192$ episodes.

\paragraph{Counter-aligned initialization.}
For each architecture and seed, the counter-aligned model starts from the
same random weights as the default model and then receives the change
below; all weights remain trainable.

For the Transformer, counter-alignment plants half the log-slope of the
exponential Bayes kernel, $\tfrac12\log(0.9)$ per lag, in every head's
attention scores. Since RoPE has no lag bias, it does so by changing the
query and key weights of one rotary frequency only along the mean
normalized observation embedding. For Mamba-2, counter-alignment averages
the base parameters with an analytic endpoint that approximates the
uniform posterior mean $\sum_{s=1}^{t}o_s/(t+1)$, using existing recurrent
coordinates to accumulate the observation sum and token count, including
the initial zero. The count coordinate is scaled so that Mamba-2's RMSNorm
approximately divides the accumulated sum by the count.
For both models, counter-alignment has strength $1/2$, uses no fitted
prediction target, and changes only the initialization, not the
architecture.

\paragraph{Fitted aggregation kernels.}
Following the kernel-fitting approach in
\S\ref{app:experiments_case1_analytical}, we describe each trained model's predictions through
\[
  g_T(o_{1:T})\approx\sum_{u=0}^{T-1}\kappa_u o_{T-u}.
\]

For each task and training seed, the final figure uses $N=2^{17}$ fresh
fitting episodes and $8192$ independent held-out episodes, shared across
its three model conditions. We estimate the kernel by ridge regression with penalty $10^{-2}$. Held-out $R^2$ measures agreement
with the trained model's predictions. Across the
three conditions, mean $R^2$ over twenty seeds ranges from
$0.9969$--$0.9993$ in the uniform regime and
$0.9641$--$0.9823$ in the exponential regime.

Figure~\ref{fig:case1_kernel_alignment} plots the $20$-seed mean fitted
kernel and pointwise $95\%$ percentile bands from $2000$ bootstrap
resamples of the training seeds. Its lower panels show the mean of the
seedwise absolute deviations from the exact Bayes kernel. The mean all-lag relative $\ell_2$ kernel errors are
$6.13\%$, $48.82\%$, and $20.58\%$ in the uniform regime, and $68.05\%$,
$13.27\%$, and $21.13\%$ in the order default Transformer, default Mamba-2, and counter-aligned model.
One counter-aligned Mamba-2 run (seed $7012$) failed to train, widening the
dotted uniform band at short lags; the median counter-aligned uniform error
is $11.41\%$.

\paragraph{Prediction error.}
We evaluate the trained models directly, independently of the kernel fit:
\[
  \mathcal E:=
  \frac{\EE[\sum_{t=1}^{T}(g_t-\hat\theta_t^{\mathrm{Bayes}})^2]}
       {\EE[\sum_{t=1}^{T}(\hat\theta_t^{\mathrm{Bayes}})^2]}.
\]
For the default Transformer and Mamba-2, respectively, the twenty-seed
mean $\mathcal E$ is $3.21\times10^{-3}$ and $6.09\times10^{-3}$
in the uniform regime, and $7.96\times10^{-2}$ and $1.86\times10^{-2}$
in the exponential regime. Ratios computed before rounding give the
main-text factors $1.9\times$ and $4.3\times$.
The Transformer has the lower error in $19$ of $20$ seeds in the uniform
regime, and Mamba-2 in all $20$ in the exponential regime.
Counter-alignment reduces mean $\mathcal E$ by $60.4\%$ for Mamba-2
in the uniform regime, including the failed run, and $70.5\%$ for the
Transformer in the exponential regime, relative to their default
initializations.

\paragraph{Effect of training budget.}
Longer training can narrow the exponential-regime gap: in a two-seed run with $1800$ updates, the Transformer acquires the exponential kernel, whereas its uniform-regime advantage persists. This does not change the inductive-bias picture of \S\ref{sec:case1_lb}: each model learns its favored kernel readily, within a small budget, and the other more slowly.

\subsubsection{Shared setup and evaluation for Case 2 heatmaps}
\label{app:experiments_common_evaluation}

To test whether our architectural lessons extend to nonconjugate setups
(App.~\hyperref[app:nonconjugate_logistic]{\ref*{app:conjugate_belief_geometry}}),
the heatmap experiments in Cases~2.a and~2.b use delayed-feedback
contextual logistic bandits \citep{fauryImprovedOptimisticAlgorithms2020},
specified in each subsection. Training and
evaluation use only matched steps, when delayed feedback arrives. The
prediction loss is binary cross-entropy to the current arm's sampled
reward, averaged over episodes and matched prediction times; the Bayes
prediction is used only for evaluation. 

Let $g_t^{\mathrm{Bayes}},g_t\in(0,1)$ denote the
Bayes and model predictions at step $t$ within an episode. We report the
Bayes-normalized performance score, summing over matched steps in all
evaluation episodes:
\[
  S
  :=
  1-
  \frac{\sum_{\text{episodes}}\sum_{t:\,\text{matched}}
    \operatorname{KL}\!\left(\mathrm{Bern}(g_t^{\mathrm{Bayes}})
    \,\|\,\mathrm{Bern}(g_t)\right)}
  {\sum_{\text{episodes}}\sum_{t:\,\text{matched}}
    \operatorname{KL}\!\left(\mathrm{Bern}(g_t^{\mathrm{Bayes}})
    \,\|\,\mathrm{Bern}(1/2)\right)}.
\]
This is cumulative excess log loss normalized between the prior-only
predictor ($S=0$) and Bayes ($S=1$); scores can be negative. 

\paragraph{Capability focus and training aids.}
\label{app:experiments_training_aids}
To isolate what each architecture can attain, each heatmap cell reports
the performance of the best validation-selected candidate from the fixed
search specified in each subsection, rather than average performance over
random initializations. To the same end, training may use aids that make these solutions easier to find: an initialization favoring the routing lag in Case~2.a
and an auxiliary objective built from past observable history in
Case~2.b, detailed in each subsection. The aids act only during
training; evaluation uses the unmodified architectures and native
inputs, so they do not change what a model can represent. A success
therefore shows what the architecture can represent, and a failure
despite the aids is stronger evidence than a failure without them.
Within each mechanism comparison, both variants of a family receive
identical aids, so the aids cannot account for the contrast between
them. All aids were fixed in development pilots before the reported grids were
trained, and all runs in a comparison train for the same number of updates.
Base learning-rate grids differ by family and width and were fixed in
development runs. Wherever the top rate wins in a reported cell scoring above
$0.1$, we extend the grid until the selected rate lies strictly inside it.

\subsubsection{Case 2.a: positional routing}
\label{app:experiments_case2a_practical}

We detail the positional-routing experiments in \S\ref{sec:case2a}
(Figure~\ref{fig:case2a_bandit_resources}), using the shared setup and
evaluation protocol of
\S\ref{app:experiments_common_evaluation}.

\paragraph{Task.}
\label{app:experiments_case2a_task}
We test the same positional-assembly demands in practical models on a
delayed contextual logistic bandit
with $p=30$ arms, horizon $T=600$, and
delay $K=16$. Each episode draws $\theta\sim\mathcal N(0,4I_p)$, one
coefficient per arm. At each step, the arm $j$ is uniform, its context $c$
has a uniform sign and magnitude drawn from $\operatorname{Unif}[0.5,1.5]$,
and the reward is Bernoulli with success probability
$(1+e^{-c\theta_j})^{-1}$, as in the logistic model of
App.~\hyperref[app:nonconjugate_logistic]{\ref*{app:conjugate_belief_geometry}}.
At time $t$, the model receives the current arm and context and the reward
generated at $t-K$, but not the source arm or its context; recovering them
is the positional-assembly problem. The input encodes the arm as a
$30$-dimensional one-hot vector, together with the context and the signed
reward. The Bayes predictor is the posterior probability of a reward
for the current arm and context, given all rewards received so far. This posterior has no closed form,
so we compute it numerically.

\paragraph{Models.}
We report the Transformer at $L\in\{16,17,32,64,128,256\}$, with
$Hd\in\{32,64,128,256,512\}$ at depth one and $Hd\in\{4,8,16,32,128\}$ at
depth two. Models below $Hd=32$ have one head, and the others use head
width $32$; each SwiGLU MLP has width $4Hd$. Queries and keys use RoPE with
base $10{,}000$ and unscaled positions. Before training, we set the query
and key weights of the first head in the first layer to target the task's
delay $K=16$.

For Mamba-2, we test depths one and two with per-layer
$Hd\in\{16,32,64,128,256,512\}$ and
$d_{\mathrm{conv}}\in\{16,17,24,32,40\}$. The residual width (\texttt{d\_model}) is $16$,
expanded width (\texttt{d\_inner}) is $32$, head dimension is $8$, and state width is chosen so
that the total recurrent-state dimension per layer equals the displayed
$Hd$; at $Hd=16$, the expanded width is $16$ with state width one. Convolutions start as causal identities, except that when
$d_{\mathrm{conv}}\geq17$, every first-layer channel starts with equal taps
$1/\sqrt2$ at the current token and at lag $16$. Both hints are training aids in the sense of
\S\hyperref[app:experiments_training_aids]{\ref*{app:experiments_common_evaluation}}:
they only set initial values, and all parameters remain trainable.

\paragraph{Training and evaluation.}
All runs use AdamW for $12{,}000$ updates with effective batch size $16$,
$\beta=(0.9,0.999)$, $\epsilon=10^{-8}$, weight decay $0.01$, gradient
clipping at $5$, and constant learning rates.

Each Transformer heatmap cell searches four random seeds
and five learning rates. At widths at most $128$, the rates are
$\{10^{-4},3\cdot10^{-4},10^{-3},3\cdot10^{-3},10^{-2}\}$; wider models use
smaller rates in the same progression, from $3\cdot10^{-5}$ to
$3\cdot10^{-3}$ at width $256$ and from $10^{-5}$ to $10^{-3}$ at width
$512$. Each seed/rate candidate first trains for $6000$ updates at $L=256$.
Every target context then independently restores that candidate's weights
and Adam state, then continues through update $12{,}000$ at its target $L$,
without restarting the optimizer or changing the rate.

Similarly, each Mamba-2 heatmap cell searches four random seeds and the five
rates used for Transformer widths at most $128$, adding $3\cdot10^{-2}$ at
$(D,Hd,d_{\mathrm{conv}})=(1,256,17)$, where the best of these lies at the
edge of the grid.

In every cell, we select the candidate whose final iterate scores best on a
fixed $1024$-episode validation set, and report that model's score on
$2048$ independent test episodes, which never influence selection.

\paragraph{Results.}
The one-layer Transformer remains near baseline, with test scores below
$.01$. Figure~\ref{fig:case2a_bandit_resources} shows that two-layer
attention benefits mainly from longer context, whereas Mamba-2 benefits
mainly from wider recurrent states once its convolution spans the delay.
For example, the two-layer Transformer at $Hd=16$ scores near zero at
$L=16$ and $.200$ at $L=17$, increasing to $.895$ at $L=256$; $Hd=4$
remains near baseline throughout.
For Mamba-2, $d_{\mathrm{conv}}=16$ remains near the prior-only baseline at
every tested per-layer state dimension. At $d_{\mathrm{conv}}=17$, the
scores for $Hd=\{16,32,64,128,256,512\}$ are
\[
  \begin{array}{c|rrrrrr}
    Hd & 16 & 32 & 64 & 128 & 256 & 512\\
    \midrule
    D=1 & .225 & .458 & .654 & .782 & .881 & .971\\
    D=2 & .265 & .511 & .567 & .824 & .924 & .975
  \end{array}
\]
Across all convolution widths, both depths first reach $S\geq0.8$ at
$Hd=128$, although depth one does so only at $d_{\mathrm{conv}}=32$; at the
$d_{\mathrm{conv}}=17$ boundary, it needs twice that width, $Hd=256$. Depth two scores higher
at every width except $Hd=64$.

\paragraph{Memory comparison.}
\label{app:experiments_case2a_memory}
We count learned parameters and persistent inference buffers for one
sequence at a common precision.
Among tested models with $S\geq0.8$, the lowest-storage Transformer has
$(D,L,Hd)=(2,256,16)$ and uses $9313+16{,}384=25{,}697$ scalar entries for
parameters and a standard key--value cache. Its Mamba-2 counterpart has
$(D,d_{\mathrm{conv}},Hd)=(1,17,256)$ and uses $3837+256+816=4909$
entries for parameters, recurrent state, and convolution buffers,
respectively. Their ratio gives the main-text factor $5.2\times$.

\paragraph{Optimization effects.}
Per-seed results additionally show how reliably training reaches each cell's best run.
For the two-layer Transformer at $L=256$, only two of four training seeds
reach $S\geq0.8$ at $Hd=32$, whereas all four score approximately $.93$ at
$Hd=128$, with learning rates selected separately within each seed by
validation.
Additional resources need not always help optimization, however.
For Mamba-2 at $D=2$ and $Hd=64$, increasing $d_{\mathrm{conv}}$ from $24$
to $40$ lowers the test score from $.583$ to $.242$.
A longer filter can reproduce a shorter one by zeroing its additional
coefficients, so this decline suggests optimization sensitivity under the
fixed training budget, rather than reduced representational capacity.

\subsubsection{Case 2.b: content routing}
\label{app:experiments_case2b_practical}

We detail the content routing experiments in \S\ref{sec:case2b}
(Figures~\ref{fig:case2b_mechanisms_scaling} and~\ref{fig:case2b_resource_scaling}),
using the shared setup and evaluation protocol of
\S\ref{app:experiments_common_evaluation}.

\paragraph{Task.}
We test the role and resource cost of selectivity in practical models on a
$10$-arm contextual logistic bandit with horizon $T=256$, using the
coefficient prior, contexts, rewards, and Bayes predictor of Case~2.a
(\S\hyperref[app:experiments_case2a_task]{\ref*{app:experiments_case2a_practical}}).
Keys occur in independently permuted cycles, and the feedback arriving with
a key is the reward generated at its previous occurrence. A shared
unit-norm codebook uses dimensions $d_{\mathrm{key}}=\{6,8,11,14,20\}$ for
$R=\{8,16,32,64,128\}$. We construct the codebook with maximum absolute
pairwise inner product below \(1/3\), shared across models. Each token
contains the current one-hot arm, codebook key, and context, the arriving
signed reward, and the preceding arm, key, and context; the arriving
reward's source arm and context are not given; recovering them is the
content-addressing problem.

\paragraph{Models.}
For the mechanism heatmaps (Figure~\ref{fig:case2b_mechanisms_scaling}),
we use $R=32$ and $d_{\mathrm{key}}=11$. The two-layer,
two-head Transformer uses a SwiGLU MLP with hidden width $4Hd$ and sweeps
linear versus softmax attention over
$Hd\in\{8,10,12,16,32\}$ and
$L\in\{16,24,32,40,63,64,256\}$.
Both variants use RoPE with base $100{,}000$ and positions divided by $64$,
without a positional initialization hint.

The Mamba-2 heatmaps use two layers, residual width $64$
(\texttt{d\_model}), expanded width $128$ (\texttt{d\_inner}), head dimension $8$,
$Hd\in\{128,256,384,512,1024,2048\}$, and
$d_{\mathrm{conv}}\in\{1,4,8,16\}$. They compare stock Mamba-2 (Full)
with two variants. Ablated removes sharp selection from both transitions
and writes: in both layers, its transitions are input-independent but
trainable, and the SiLU activations on the two write branches are replaced
by their first-order linearizations at zero, so its writes are bilinear, as
in our SSM classes. Fixed transitions removes only transition selectivity:
its decay is input-independent but trainable, while its writes match
Full's. All variants retain nonlinear reads, output gating, normalization,
and residual connections, initialize convolutions as identities, and use
state-channel prefixes of a shared $Hd=4096$ initialization.

For width scaling (Figure~\ref{fig:case2b_resource_scaling}), we vary $R\in\{8,16,32,64,128\}$. The Transformer is a
two-layer, two-head softmax model with $L=256$ and
$Hd\in\{2,4,6,8,10,12,16,24,32,48,64\}$, with the same RoPE settings
at every $R$. Mamba-2 uses two selective layers, residual width $64$, head dimension $8$,
$d_{\mathrm{conv}}=1$, and
$Hd\in\{128,256,384,512,1024,2048,4096,8192,16384\}$. Its expanded
width is $128$ and its state width varies. At $R=8,16$, we additionally test
$Hd\in\{8,16,32,64\}$ with state width one and expanded width $Hd$.

\paragraph{Training and evaluation.}
All runs use $12{,}000$ AdamW updates with batch size $32$,
$\beta=(0.9,0.999)$, $\epsilon=10^{-8}$, weight decay $0.01$, gradient
clipping at $5$, and constant learning rates. Training follows a fixed
two-phase curriculum. In the alignment phase, only the input projection and
first layer train, on an auxiliary target that pairs each arriving reward
with its source arm, key, and context from the observed history. In the
reward phase, the prediction loss takes over with a fresh optimizer. The
Transformer aligns for $3000$ updates and then trains all parameters on
reward for $9000$; Mamba-2 aligns for $9000$ updates, then freezes the
aligned layers and trains the rest on reward for $3000$. Development pilots
selected these schedules: the Transformer trained best with a short
alignment phase followed by joint training, and Mamba-2 with a longer
alignment phase whose layers then stay frozen.

Each Transformer cell searches three random seeds and learning rates
$\{10^{-4},3\cdot10^{-4},5\cdot10^{-4},10^{-3},3\cdot10^{-3},10^{-2},3\cdot10^{-2},10^{-1}\}$.
For the heatmaps, each seed/rate candidate first trains for $6000$ updates
at $L=256$. Each target context independently restores that candidate's
weights and Adam state, then continues for $6000$ updates at its target
$L$. Width-scaling runs (Figure~\ref{fig:case2b_resource_scaling}) use
$L=256$ throughout.

Mamba-2 trains each configuration from scratch, searching eight random
seeds and learning rates $\{3\cdot10^{-4},10^{-3},3\cdot10^{-3}\}$,
extended with $10^{-2}$ and $3\cdot10^{-2}$ in some cells.

For both families, we select the candidate whose final iterate scores best
on a fixed $256$-episode validation set, and report its score on $2048$
independent test episodes. The width-scaling frontier is the smallest
tested width whose selected validation score reaches $S\geq0.8$.

\begin{wrapfigure}[12]{r}{145.5bp}
  \vspace{-1.6\baselineskip}
  \centering
  \includegraphics{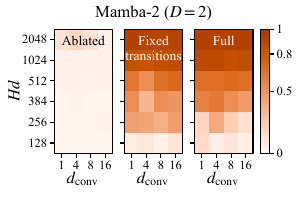}
  \begin{minipage}{\linewidth}
    \vspace{-1.2\baselineskip}
    \caption{Content addressing with the three Mamba-2 variants ($R=32$).
    Darker means closer to Bayes.}
    \label{fig:case2b_fixed_transitions}
  \end{minipage}
\end{wrapfigure}
\paragraph{Results.}
Figure~\ref{fig:case2b_mechanisms_scaling} shows that softmax attention and
full Mamba-2 escape the low scores of their restricted variants: softmax
attention reaches $S\geq0.8$ in $9/35$ cells and full Mamba-2 in $8/24$,
whereas linear attention and ablated Mamba-2 reach it in none.
Fixed transitions reaches $S\geq0.8$ in the same $8/24$ cells as Full
(Figure~\ref{fig:case2b_fixed_transitions}), so only Ablated, which lacks
sharp selection in both transitions and writes, stays near baseline,
consistent with Theorem~\ref{thm:case2b_floors}. Whereas the construction
in Prop.~\ref{prop:case2b_pair} escapes the floor through selective
transitions, Mamba-2's nonlinear writes provide a second route to the sharp
selection that our floors identify as necessary. Transformer
performance improves mainly with context once $Hd$ reaches $12$, whereas
Mamba-2 depends mainly on $Hd$, with no systematic benefit from increasing
$d_{\mathrm{conv}}$. At $Hd=12$, scores rise smoothly with context, from
$.045$ at $L=16$ to $.808$ at both $L=63$ and $L=64$, and $.951$ at
$L=256$.
The width-scaling frontier in Figure~\ref{fig:case2b_resource_scaling} is
\[
  \begin{array}{c|rrrrr}
    R & 8 & 16 & 32 & 64 & 128\\
    \midrule
    \text{Transformer }Hd & 16 & 10 & 12 & 12 & 24\\
    \text{test score} & .883 & .933 & .951 & .946 & .969\\
    \text{Mamba-2 }Hd & 64 & 256 & 1024 & 4096 & \text{--}\\
    \text{test score} & .874 & .831 & .889 & .917 & \text{--}
  \end{array}
\]
All nine frontier points exceed $0.8$ on independent testing. At $R=128$,
no tested Mamba-2 width reaches $S\geq0.8$; the best test scores are $.283$
at $Hd=8192$ and $.186$ at $Hd=16384$. Across the tested range, the
successful Transformer width stays between $10$ and $24$, whereas the
Mamba-2 frontier quadruples with each doubling of $R$, following $Hd=R^2$
from $R=8$ to $64$. This quadratic growth is consistent with
Theorem~\ref{thm:case2b_floors}(c), which only lower-bounds the required
width by $\widetilde\Omega(R)$. It suggests that practical Mamba-2 needs
even more state than this floor, although part of this growth may reflect
optimization under our fixed training budget. The dashed
$\log R$ and linear-$R$ curves are visual guides anchored at the first
point, not fitted exponents.

\paragraph{Optimization effects.}
As in Case~2.a, architectural resources affect not only representational
capability but also the outcomes of optimization.
For the Transformer at $R=128$, only one of three training seeds reaches
$S\geq0.8$ at the frontier width $Hd=24$, whereas all three score
$.942$--$.974$ at $Hd=64$, with learning rates selected separately within
each seed by validation.
In Figure~\ref{fig:case2b_mechanisms_scaling}, the $Hd=10$ row, just below
the width where performance rises sharply, is irregular: $.710$ at $L=63$
but $.069$ at $L=64$ and $.077$ at $L=256$. Since longer contexts only add
visible tokens, this irregularity suggests fragile optimization at the
threshold width rather than an effect of context length.
For Mamba-2 at $R=128$, the best test score falls from $.283$ at
$Hd=8192$ to $.186$ at $Hd=16384$; since a wider state can reproduce a
narrower one by leaving its extra channels unused, this decline again
suggests optimization sensitivity under the fixed training budget, rather
than a capacity limit.
\par\WFclear

\FloatBarrier
\subsection{Validation on Beta--Bernoulli bandits}
\label{app:experiments_bernoulli}

We repeat the Case~2 experiments of \S\ref{app:experiments_practical} on
conjugate Beta--Bernoulli bandits \citep{russoTutorialThompsonSampling2018}
with the same architectures and score,
using earlier training recipes documented in the released code. Each episode
draws independent arm means from $\operatorname{Unif}[0,1]$, the
$\operatorname{Beta}(1,1)$ prior, and rewards are Bernoulli with the pulled
arm's mean, as in the conjugate model of
App.~\hyperref[app:conjugate_bernoulli]{\ref*{app:conjugate_belief_geometry}}. Case~2.a keeps $p=30$ arms, $T=600$, and $K=16$, and
Case~2.b keeps $10$ arms, $T=256$, and the same keys. The Bayes predictor is
the exact Beta posterior predictive.

\begin{wrapfigure}[11]{r}{221.5bp}
  \vspace{-2.2\baselineskip}
  \centering
  \includegraphics{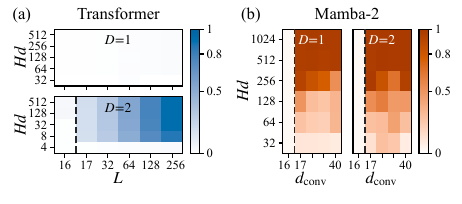}
  \vspace{-3em}
  \caption{Positional assembly on a Beta--Bernoulli bandit
  ($p=30$, $K=16$). Darker means closer to Bayes; dashed lines mark
  the $K+1=17$ threshold.}
  \label{fig:case2a_bandit_resources_bernoulli}
\end{wrapfigure}
\paragraph{Case~2.a.}
Figure~\ref{fig:case2a_bandit_resources_bernoulli} reproduces the
positional-assembly pattern of the main text. The one-layer Transformer
stays near baseline, and both two-layer models improve only once their
first layer spans the lag ($L,d_{\mathrm{conv}}\geq K+1=17$). Two-layer
attention then improves mainly with context and Mamba-2 mainly with
recurrent-state width. The widths at which the models first exceed $0.8$
differ from those on the nonconjugate task: $Hd=8$ rather than $16$ for the
Transformer, and $Hd=256$ rather than $128$ for Mamba-2. The lowest-memory successful Transformer still
stores $2.2\times$ as many entries as its Mamba-2 counterpart.
\par\WFclear

\paragraph{Case~2.b.}
Figures~\ref{fig:case2b_mechanisms_bernoulli}
and~\ref{fig:case2b_resource_scaling_bernoulli} reproduce the
content-addressing pattern: softmax attention and full Mamba-2 escape the
low scores of linear attention and ablated Mamba-2. At every $R$, both
families need the same or smaller widths than on the nonconjugate task, and
Mamba-2 now reaches $0.8$ at $R=128$, but the contrast is unchanged: the
smallest successful width stays between $8$ and $16$ for the Transformer,
whereas for Mamba-2 it again quadruples with each doubling of $R$,
following $Hd=R^2/2$ from $R=16$ to $128$.

\begin{figure}[htbp]
  \centering
  \begin{minipage}[t]{221.5bp}
    \vspace{0pt}
    \centering
    \includegraphics{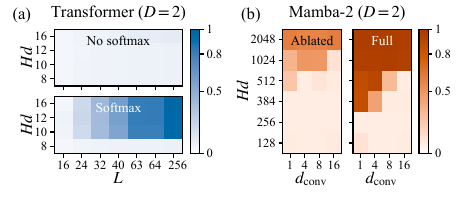}
    \vspace{-2.5em}
    \caption{Content addressing on a Beta--Bernoulli bandit
    ($p=10$, $R=32$). Darker means closer to Bayes.}
    \label{fig:case2b_mechanisms_bernoulli}
  \end{minipage}\hfill
  \begin{minipage}[t]{148bp}
    \vspace{0pt}
    \centering
    \includegraphics{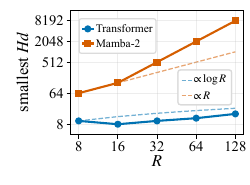}
    \vspace{-0.45em}
    \caption{Smallest tested widths
    reaching \mbox{$S\geq0.8$} on validation.}
    \label{fig:case2b_resource_scaling_bernoulli}
  \end{minipage}
\end{figure}
\FloatBarrier

\subsection{Validation of the theoretical results}
\label{app:experiments_analytical}

We use the architecture classes defined in \S\ref{sec:architectures} to test
the predicted context scale and aggregation kernels in
Case~1 (\S\ref{app:experiments_case1_analytical}), depth and positional reach
for evidence assembly in Case~2.a
(\S\ref{app:experiments_case2a_analytical}), and selectivity for content
addressing in Case~2.b (\S\ref{app:experiments_case2b_analytical}).

\subsubsection{Case 1: belief maintenance}
\label{app:experiments_case1_analytical}

\paragraph{Context-length scaling.}
Proposition~\ref{prop:stationary_pair} predicts that the leading-order optimal
window for one-layer linear attention is
\[
  L^\star(T)=\sqrt{3T\frac{\sum_i v_i}{\sum_i\lambda_i}}.
\]
We test this prediction in the matrix-valued Gaussian Case~1 with
\[
  p=2,\qquad \Sigma_x=I_2,\qquad
  \Sigma_\theta=
  \begin{bmatrix}
    2.5 & -1.5\\
    -1.5 & 2.5
  \end{bmatrix},
  \qquad \sigma^2=1.
\]
For this choice, $\sum_i\lambda_i=5$ and $\sum_i v_i=17$, giving
$L^\star(T)=\sqrt{51T/5}$. We use
$T\in\{512,1024,2048,4096\}$ and sweep
\[
  \frac{L}{L^\star(T)}
  \in
  \left\{
    \frac14,\frac12,\frac1{\sqrt2},1,\sqrt2,2,4
  \right\},
\]
rounding the resulting window to the nearest integer and plotting the realized
ratio after rounding.

We train the one-layer linear-attention model
$\mathcal F_{\mathrm{att}}(1,2,1,L)$ with residual width $d_e=6$, no positional
bias, and a scalar readout. At time $t$, it receives
$z_t=(x_{t+1},y_t,x_t)$ and predicts $y_{t+1}$. All parameters are learned from
a random initialization; no aggregation kernel is planted. We initialize all
affine biases at zero and divide the attention output projection by $\sqrt L$
so that its initial scale is comparable across windows. Training minimizes
squared error to the noisy response $y_{t+1}$, averaged over episodes and
prediction times.

For each horizon--window pair $(T,L)$, we use five paired random seeds and
$5000$ Adam updates per run, with initial learning rate $10^{-3}$, cosine
decay, and gradient clipping at $5$. Each update uses a fixed batch of
$2^{15}=32{,}768$ freshly sampled tokens, with the number of episodes adjusted
to the horizon. For a fixed horizon and seed, all
window lengths receive the same training stream. We evaluate the final iterate
on a common held-out set of $1024$ episodes at each horizon, using cumulative
squared-loss regret relative to the exact Bayes predictor.

\begin{wrapfigure}[13]{r}{0.52\textwidth}
  \vspace{-1.5\baselineskip}
  \centering
  \includegraphics[width=\linewidth]
    {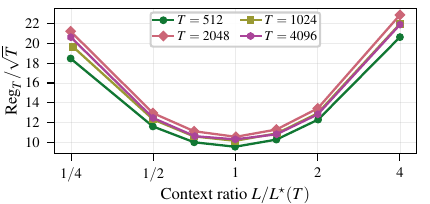}
  \caption{Context-length scaling in one-layer linear attention. Each point
  reports the minimum over five initializations.}
  \label{fig:case1_lstar_scaling}
\end{wrapfigure}

Figure~\ref{fig:case1_lstar_scaling} shows the predicted U-shaped
bias--variance trade-off. Normalizing by $\sqrt T$ approximately collapses the
curves across horizons, and in every case the empirical valley is centered on
the predicted scale. This location is not created by selecting the best run:
the five-seed mean is also minimized by the $L/L^\star(T)=1$ grid point at every
horizon. Among the $20$ individual seed--horizon profiles, $13$ are minimized
there and the remaining $7$ at the adjacent $1/\sqrt2$ grid point; none is
minimized elsewhere. Thus the experiment supports both the predicted
$\sqrt T$ context scaling and the location of its leading-order optimum under
finite optimization.

\paragraph{Learned aggregation kernels.}
We next test whether training from random initialization recovers the kernels
predicted by Proposition~\ref{prop:stationary_pair}. At $T=4096$, we reuse the
five $\mathcal F_{\mathrm{att}}(1,2,1,L^\star)$ runs at
$L^\star(T)=204$ and train five $\mathcal F_{\mathrm{SSM}}(1,2,1,1)$ runs.
Both models have residual width $6$ and use the $5000$-update protocol above; no target kernel is planted beforehand. Within each class, we select the run with lowest validation cumulative
regret before inspecting its kernel.

For each selected model, we fit kernel matrices $\kappa_u$ to its predictions
on fresh sequences of $501$ observations:
\[
  g_{501}\approx x_{502}^{\top}
  \sum_{u=0}^{500}\kappa_u x_{501-u}y_{501-u}.
\]
The fitted kernel aggregates evidence atoms into a belief, which is then
applied to the final regressor. Writing $\xi=U^\top x$, we regress on
$\xi_{502,i}\xi_{501-u,j}y_{501-u}$ for $i,j\in\{1,2\}$; the coefficients,
in original units, are $(U^\top\kappa_uU)_{ij}$.
We plot the diagonal entries $\kappa_{u,i}:=(U^\top\kappa_uU)_{ii}$:
off-diagonal entries account for only $0.91\%$ and $1.22\%$ of the fitted
attention and SSM kernels' squared Frobenius norms over $u=1,\ldots,500$.

\begin{wrapfigure}[16]{r}{0.52\textwidth}
  \vspace{-1.1\baselineskip}
  \centering
  \includegraphics[width=\linewidth]
    {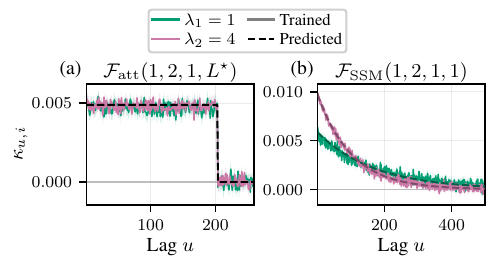}
  \caption{Optimal-kernel validation at $T=4096$.
  \emph{Trained}: kernels fitted to model predictions;
  \emph{Predicted}: kernels from Prop.~\ref{prop:stationary_pair}.
  Shading shows pointwise 95\% bands over 100 residual perturbations.}
  \label{fig:case1_analytical_kernels}
\end{wrapfigure}
We fit on $8192$ episodes, rescaling each product to unit empirical
root-mean-square without centering and minimizing mean squared error with ridge
penalty $10^{-2}$. To assess coefficient uncertainty, we use $100$
random-sign (Rademacher) perturbations of the regression residuals, holding
the trained models and probe inputs fixed. The shaded bands in
Figure~\ref{fig:case1_analytical_kernels} show the central $95\%$ of the
resulting estimates separately for each coefficient.

Figure~\ref{fig:case1_analytical_kernels} shows that kernels fitted to the
models' predictions closely follow the leading-order optima in
Proposition~\ref{prop:stationary_pair}: a uniform attention window and
mode-dependent exponential decay for the SSM. Over $u=1,\ldots,500$, their
relative Frobenius errors against these theoretical kernels are $13.5\%$
and $15.6\%$, respectively. Both kernel-only fits explain $99.92\%$ of the
variation in the models' predictions on $2048$ independent probe episodes
($R^2=0.9992$), showing that this kernel-form description accurately captures
the selected models' prediction behavior.
\par\WFclear

\FloatBarrier
\subsubsection{Case 2.a: positional routing}
\label{app:experiments_case2a_analytical}

\paragraph{Task and score.}
To test how depth and positional reach affect evidence assembly, we train
the analytical models on scalar Gaussian regression with $p=1$,
delay $K=4$, and horizon $T=512$. Each episode draws independent standard
normal $\theta$, regressors $x_t$, and noises $\epsilon_t$, with
$y_t=\theta x_{t-K}+\epsilon_t$ for $t>K$ and $y_1=\cdots=y_K=0$.
The model receives $z_t=(x_{t+1},y_t,x_t)$ and predicts $y_{t+1}$.
We report the Bayes-normalized performance score under squared loss:
\[
  S=1-
  \frac{\sum_{\text{episodes}}\sum_{t=1}^{T}
    (g_t-g_t^{\mathrm{Bayes}})^2}
  {\sum_{\text{episodes}}\sum_{t=1}^{T}(g_t^{\mathrm{Bayes}})^2}.
\]
All $T$ prediction times are included; zero prediction gives $S=0$ and Bayes
gives $S=1$.

\paragraph{Models.}
Both linear attention and the fixed-transition SSM use $D\in\{1,2\}$,
one head, residual width $18$, and $Hd\in\{1,2,4,8,16\}$. We sweep
$L\in\{1,2,3,4,5,8,16,32,64,96\}$ for attention and
$d_{\mathrm{conv}}\in\{1,2,3,4,5,8\}$ for the SSM.
All model parameters are trainable. Affine biases start at zero; SSM
convolutions start as causal identities and input-independent transitions
as $0.9I$. Attention's random query and key weights are scaled by $L^{-1/4}$
and output projections by $L^{-1/2}$; no routing lag is planted at initialization.

\paragraph{Training and evaluation.}
We train on noisy responses using squared loss; Bayes predictions are used
only for evaluation.
Each cell searches four paired initialization/data-stream seeds and learning
rates $\{10^{-4},3\!\times\!10^{-4},10^{-3},3\!\times\!10^{-3},10^{-2}\}$.
We use $12{,}000$ constant-rate AdamW updates, fresh batches of $16$ episodes,
$\beta=(0.9,0.999)$, $\epsilon=10^{-8}$, weight decay $0.01$, and gradient
clipping at $5$. We evaluate checkpoints at $3000$, $6000$, and $12{,}000$
updates and report the maximum over $60$ scheduled candidates per cell.
Selection and reporting use the same fixed set of $256$ validation episodes.

\begin{wrapfigure}[14]{r}{0.52\textwidth}
  \vspace{0\baselineskip}
  \centering
  \includegraphics[width=\linewidth]{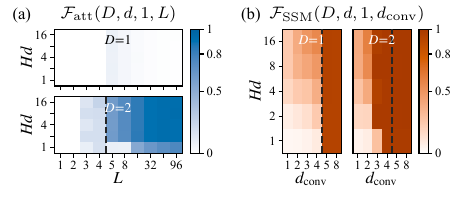}
  \vspace{-2em}
  \caption{Positional routing in the analytically tractable classes
  ($p=1$, $K=4$, $T=512$). Colors show the best validation score under the
  fixed search above. Dashed lines mark the single-layer reference $K+1=5$.}
  \label{fig:case2a_analytical}
\end{wrapfigure}
\paragraph{Results.}
Figure~\ref{fig:case2a_analytical} shows low scores for one-layer attention,
whereas two layers benefit strongly from longer context. At depth two,
$L=2$ remains near zero, $L=3,4$ give partial gains, and $Hd\geq2$ reaches
about $0.92$--$0.94$ at $L=64,96$.
The one-layer SSM reaches about $0.94$ at $d_{\mathrm{conv}}\geq5$;
two layers already reach $0.95$--$0.99$ at $d_{\mathrm{conv}}=3$ for
$Hd\geq2$. Thus additional depth enables successful learning with shorter
convolutions. This is a benefit not observed in the
practical Mamba-2 runs in Figure~\ref{fig:case2a_bandit_resources}
(\S\ref{app:experiments_case2a_practical}). The different task, training
dynamics, or architectural choices such as selective transitions may account for this contrast.

These best scores describe successful configurations, not typical training outcomes.
For example, at $d_{\mathrm{conv}}=3$ and $Hd\geq2$, only $1$--$4$ of the
$20$ seed--learning-rate runs per cell reach $S\geq0.8$. The short-convolution
depth advantage is therefore optimization-sensitive under this training protocol.
\par\WFclear

\FloatBarrier
\subsubsection{Case 2.b: content routing}
\label{app:experiments_case2b_analytical}

\paragraph{Task and score.}
To test how nonlinear selection mechanisms of architectures affect content addressing, we use the
Gaussian content-routing problem of \S\ref{sec:case2b} with
$p=1$, $R=8$, $d_{\mathrm{key}}=6$, $T=512$, and unit prior, regressor,
and noise variances. The fixed codebook has maximum absolute pairwise
inner product below $1/3$. Models receive the current and next keys along
with $(x_{t+1},y_t,x_t)$; the routed regressor is not supplied.
We use the Bayes-normalized squared-loss score $S$ defined in
\S\ref{app:experiments_case2a_analytical}, including all $T$ prediction times.

\paragraph{Models.}
All models have depth two and residual width $18$. Attention has two heads,
$Hd\in\{2,4,8,16\}$, and $L\in\{8,16,32,64\}$; we compare all-linear
attention with softmax in the first layer and linear attention in the second.
The one-head SSM compares fixed transitions in both layers with a diagonal
sigmoid transition in the first layer and a fixed transition in the second.
Both retain ordinary bilinear writes and reads, with
$d_{\mathrm{conv}}\in\{1,2,4,8\}$ and per-layer state widths $Hd\in\{1,2,4,8,16\}$. SSM convolutions start as causal identities,
fixed transitions as $0.9I$, and sigmoid gates at the constant $0.9$.
Linear attention scales its random query and key weights by $L^{-1/4}$ and
output projections by $L^{-1/2}$, as in Case~2.a; softmax attention keeps the
default random initialization.

\paragraph{Training and evaluation.}
We train on noisy responses using squared loss and use the same
$12{,}000$-update AdamW protocol as \S\ref{app:experiments_case2a_analytical}:
four seeds, five learning rates, and three checkpoints per run.
Each cell reports the maximum over $60$ scheduled candidates, with selection
and reporting on the same fixed set of $256$ validation episodes.

\Needspace{16\baselineskip}
\begin{wrapfigure}[13]{r}{0.52\textwidth}
  \vspace{-1\baselineskip}
  \centering
  \includegraphics[width=\linewidth]{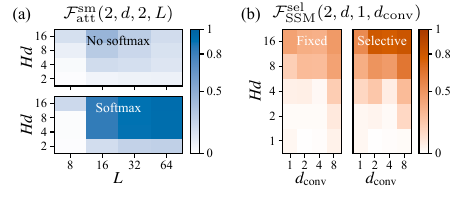}
  \vspace{-2em}
  \caption{Content routing in the analytically tractable classes
  ($p=1$, $R=8$, $T=512$, $D=2$). Colors show the best validation score
  under the fixed search.}
  \label{fig:case2b_analytical}
\end{wrapfigure}

\paragraph{Results.}
In Figure~\ref{fig:case2b_analytical}, softmax attention reaches
$0.83$--$0.93$ for $Hd\geq4$ and ${L\geq16}$, whereas all-linear attention
remains below $0.46$. Neither SSM performs well at very small state widths.
No displayed fixed-transition configuration reaches the success threshold
$S\geq0.8$; selective transitions enable success at
$Hd=16$, $d_{\mathrm{conv}}=8$, with score $0.820$.

As stated in the previous section, best-run performance does not imply reliable optimization: the successful
$Hd=16$, $d_{\mathrm{conv}}=8$ cell reaches $S\geq0.8$ in only one of its
$20$ seed--learning-rate runs.
\par\WFclear

\FloatBarrier
 
\section{Case 1: full statements and proofs}\label{app:case1}

For an overview of the notation and proof dependencies, see the
\hyperref[app:notation_guide]{notation guide} and
\hyperref[app:proof_guide]{proof guide} at the start of the appendix.

\subsection{Theorem~\ref{thm:stationary_floor} and Proposition~\ref{prop:stationary_pair}: Stationary kernels}
\label{app:case1_vector_stationary}

This section proves Theorem~\ref{thm:stationary_floor} and
Proposition~\ref{prop:stationary_pair}. We first introduce whitened coordinates
and derive the regret of a stationary kernel
(\S\ref{app:case1_prediction_geometry}), then optimize over all such
kernels (\S\ref{app:case1_stationary_floor}). Using the belief geometry in
Figure~\ref{fig:belief_geometry} (\S\ref{app:case1_score_projection}), we
connect the stationary-kernel analysis to
attention (\S\ref{app:case1_attention_kernels}) and SSMs
(\S\ref{app:case1_ssm_kernels}). Section~\ref{app:case1_stationary_pair} combines
these results to complete the proof of Proposition~\ref{prop:stationary_pair}.

\subsubsection{Whitened coordinates and stationary-kernel regret}
\label{app:case1_prediction_geometry}

We first express prediction error as belief-estimation error.
Lemma~\ref{lem:vector_stationary_risk} then gives the exact loss and regret
of a stationary kernel, providing the objective we optimize in
\S\ref{app:case1_stationary_floor}. Fix
$p\in\mathbb N$, $\Sigma_x\succ0$, $\Sigma_\theta\succ0$, and
$\sigma^2>0$, independently of the horizon, and consider
\[
  \theta\sim\mathcal N(0,\Sigma_\theta),
  \qquad
  x_t\stackrel{\mathrm{iid}}{\sim}\mathcal N(0,\Sigma_x),
  \qquad
  y_t=x_t^\top\theta+\epsilon_t,
  \qquad
  \epsilon_t\stackrel{\mathrm{iid}}{\sim}\mathcal N(0,\sigma^2),
\]
with $\theta$, $(x_t)_{t\geq1}$, and $(\epsilon_t)_{t\geq1}$ mutually
independent.  The token is
$z_t=(x_{t+1},y_t,x_t)$, and the predictor estimates the scalar
$y_{t+1}$.  The sufficient statistics
\[
  A_t:=\sum_{s\leq t}x_sy_s,
  \qquad
  B_t:=\sum_{s\leq t}x_sx_s^\top
\]
give the posterior mean:
\[
  \hat\theta_t^{\mathrm{Bayes}}
  =
  \left(\Sigma_\theta^{-1}+\sigma^{-2}B_t\right)^{-1}
  \sigma^{-2}A_t.
\]

Introduce the whitened coordinates
\begin{equation}
  \widetilde x_t:=\Sigma_x^{-1/2}x_t,
  \qquad
  \widetilde\theta:=\Sigma_x^{1/2}\theta,
  \qquad
  \widetilde\Sigma_{\theta}:=\Sigma_x^{1/2}\Sigma_\theta\Sigma_x^{1/2}
    =U\operatorname{diag}(\lambda_1,\ldots,\lambda_p)U^\top .
  \label{eq:vector_whitening}
\end{equation}
Then $\widetilde x_t\sim\mathcal N(0,I_p)$,
$\widetilde\theta\sim\mathcal N(0,\widetilde\Sigma_{\theta})$, and
$y_t=\widetilde x_t^\top\widetilde\theta+\epsilon_t$. The eigenvalues of $\widetilde\Sigma_{\theta}$ are the
generalized belief variances in the fresh-regressor prediction geometry.
The evidence innovation $\widetilde x_sy_s-\widetilde\theta$ has covariance
$\Gamma$, with eigenvalues $v_i$:
\begin{equation}
  \Gamma:=\widetilde\Sigma_{\theta}
    +\bigl[\operatorname{tr}(\widetilde\Sigma_{\theta})+\sigma^2\bigr]I_p,
  \qquad
  v_i:=\lambda_i+\operatorname{tr}(\widetilde\Sigma_{\theta})+\sigma^2.
  \label{eq:vector_covariance_scales}
\end{equation}

We first record the fresh-regressor projection used throughout the section.
Let $\mathcal G_t^-:=\sigma(\widetilde x_{1:t},y_{1:t})$.  For an arbitrary scalar
predictor $g$, define
\begin{equation}
  \widetilde\theta_t(g):=\EE[\widetilde x_{t+1}g_t\mid\mathcal G_t^-],
  \qquad
  g_t^\perp:=g_t-\widetilde x_{t+1}^\top\widetilde\theta_t(g).
  \label{eq:vector_implicit_belief}
\end{equation}
Since $\widetilde x_{t+1}$ is independent of $\mathcal G_t^-$ and has covariance $I_p$,
$g_t^\perp$ is orthogonal to every
$\widetilde x_{t+1}^\top a_t$ with
$a_t\in L^2(\mathcal G_t^-;\RR^p)$.  The whitened posterior mean is
\[
  \widetilde\theta_t^{\mathrm{Bayes}}
  =
  \left(
    \widetilde\Sigma_{\theta}^{-1}+\sigma^{-2}\sum_{s\leq t}\widetilde x_s\widetilde x_s^\top
  \right)^{-1}
  \sigma^{-2}\sum_{s\leq t}\widetilde x_sy_s.
\]
Thus fresh-regressor projection and posterior orthogonality give
\begin{align}
  \ell_t(g)-\ell_t(g^{\mathrm{Bayes}})
  &=
  \EE\|\widetilde\theta_t(g)-\widetilde\theta_t^{\mathrm{Bayes}}\|_2^2
  +\EE[(g_t^\perp)^2],
  \label{eq:vector_fresh_query_projection}\\
  &=
  \EE\|\widetilde\theta_t(g)-\widetilde\theta\|_2^2-e_t^{\mathrm{Bayes}}
  +\EE[(g_t^\perp)^2],
  \label{eq:vector_fresh_query_mse}
\end{align}
where the expected squared error of the Bayes belief and its cumulative sum are
\begin{equation}
  e_t^{\mathrm{Bayes}}:=
  \EE\operatorname{tr}\!\left[
    \left(
      \widetilde\Sigma_{\theta}^{-1}+\sigma^{-2}\sum_{s\leq t}\widetilde x_s\widetilde x_s^\top
    \right)^{-1}
  \right],
  \qquad
  \mathcal E_T^{\mathrm{Bayes}}:=\sum_{t\leq T}e_t^{\mathrm{Bayes}}.
  \label{eq:vector_bayes_baseline}
\end{equation}
We suppress their dependence on the fixed distributional parameters.
Equivalently, in the original coordinates,
\[
  \EE\!\left[
    \bigl(x_{t+1}^\top(\hat\theta_t-\theta)\bigr)^2
    \,\middle|\,\hat\theta_t,\theta
  \right]
  =
  (\hat\theta_t-\theta)^\top\Sigma_x(\hat\theta_t-\theta)
  =
  \|\widetilde\theta_t-\widetilde\theta\|_2^2.
\]

We first optimize over stationary-kernel predictors $g$ of the form
\[
  g_t=x_{t+1}^\top\hat\theta_t(g),
  \qquad
  \hat\theta_t(g)=\sum_{u<t}\kappa_ux_{t-u}y_{t-u},
\]
where $\kappa=(\kappa_u)_{u\geq0}$ is any deterministic matrix kernel.
Under \eqref{eq:vector_whitening},
\begin{equation}
  \widetilde\theta_t(g)=\Sigma_x^{1/2}\hat\theta_t(g)
  =\sum_{u<t}K_u\widetilde x_{t-u}y_{t-u},
  \quad
  K_u:=\Sigma_x^{1/2}\kappa_u\Sigma_x^{1/2},
  \quad
  M_t:=\sum_{u<t}K_u.
  \label{eq:vector_stationary_estimator}
\end{equation}
Here $M_t$ is the kernel mass. For this predictor, we also write
$\mathrm{Reg}_T(\kappa)=\mathrm{Reg}_T(K):=\mathrm{Reg}_T(g)$.
Whitening leaves predictions and regret
unchanged. Since the transformation is invertible, optimizing over the transformed stationary kernels is
equivalent to optimizing over the original ones.

The next lemma makes the stationary kernel's bias--variance trade-off
explicit, giving the regret formula we optimize to prove
Theorem~\ref{thm:stationary_floor}.

\begin{lemma}[Exact loss and regret of a stationary kernel]
\label{lem:vector_stationary_risk}
For every deterministic stationary kernel $K=(K_u)_{u\geq0}$ in whitened
coordinates, the corresponding predictor $g$ satisfies
\begin{gather}
  \EE\|\widetilde\theta_t(g)-\widetilde\theta\|_2^2
  =
  \operatorname{tr}\!\left[
    (M_t-I_p)\widetilde\Sigma_{\theta}(M_t-I_p)^\top
  \right]
  +
  \sum_{u<t}\operatorname{tr}\!\left[K_u\Gamma K_u^\top\right],
  \label{eq:vector_stationary_mse}\\
  \mathrm{Reg}_T(K)
  =
  \sum_{t\leq T}\!\Bigl\{
    \operatorname{tr}\!\left[
      (M_t-I_p)\widetilde\Sigma_{\theta}(M_t-I_p)^\top
    \right]
    +\sum_{u<t}\operatorname{tr}\!\left[K_u\Gamma K_u^\top\right]
  \Bigr\}
  -\mathcal E_T^{\mathrm{Bayes}}.
  \label{eq:vector_stationary_regret}
\end{gather}
The two terms in \eqref{eq:vector_stationary_mse} are squared bias and variance,
averaged over the prior: the bias comes from $M_t-I_p$, and the variance
from the aggregated evidence.
The loss is
$\ell_t(g)=\EE\|\widetilde\theta_t(g)-\widetilde\theta\|_2^2+\sigma^2$;
the noise term cancels in regret.
\end{lemma}

\begin{proof}
Write
\[
  \widetilde x_sy_s=\widetilde\theta+\xi_s,
  \qquad
  \xi_s:=(\widetilde x_s\widetilde x_s^\top-I_p)\widetilde\theta+\widetilde x_s\epsilon_s.
\]
Conditionally on $\widetilde\theta$, the innovations $\xi_s$ are centered and
independent across $s$.  Hence
\[
  \EE[\widetilde\theta\xi_s^\top]=0,
  \qquad
  \EE[\xi_s\xi_r^\top]=0
  \quad(s\neq r).
\]
For a standard Gaussian vector $\widetilde x$, Isserlis' identity gives
\[
  \EE[\widetilde x\widetilde x^\top \widetilde\Sigma_{\theta}\widetilde x\widetilde x^\top]=2\widetilde\Sigma_{\theta}+\operatorname{tr}(\widetilde\Sigma_{\theta})I_p.
\]
Therefore
\begin{align*}
  \EE[\xi_s\xi_s^\top]
  &=
  \EE[(\widetilde x\widetilde x^\top-I_p)\widetilde\Sigma_{\theta}(\widetilde x\widetilde x^\top-I_p)]+\sigma^2I_p\\
  &=
  \widetilde\Sigma_{\theta}+\bigl[\operatorname{tr}(\widetilde\Sigma_{\theta})+\sigma^2\bigr]I_p
  =\Gamma.
\end{align*}
Since
\[
  \widetilde\theta_t(g)-\widetilde\theta
  =(M_t-I_p)\widetilde\theta+\sum_{u<t}K_u\xi_{t-u},
\]
all bias--innovation and cross-time innovation terms vanish.  Taking the
  squared norm proves \eqref{eq:vector_stationary_mse}.  A stationary-kernel
predictor has $g_t^\perp=0$, so
\eqref{eq:vector_fresh_query_mse} gives
\eqref{eq:vector_stationary_regret} after summing over time.
\end{proof}

Two direct checks are useful.  The zero kernel has mean-squared belief-estimation error
$\operatorname{tr}(\widetilde\Sigma_{\theta})$ at every time.  With one observation and an arbitrary
matrix $K$,
\[
  \EE\|K\widetilde x_1y_1-\widetilde\theta\|_2^2
  =
  2\operatorname{tr}(K\widetilde\Sigma_{\theta}K^\top)
  +\bigl[\operatorname{tr}(\widetilde\Sigma_{\theta})+\sigma^2\bigr]\|K\|_F^2
  -2\operatorname{tr}(K\widetilde\Sigma_{\theta})
  +\operatorname{tr}(\widetilde\Sigma_{\theta}),
\]
which is the direct fourth-moment expansion of
\eqref{eq:vector_stationary_mse}.  When $p=1$, the innovation coefficient is
$v_1=2\lambda_1+\sigma^2$.

\subsubsection{Proof of Theorem~\ref{thm:stationary_floor}}
\label{app:case1_stationary_floor}

Starting from the exact regret formula in
Lemma~\ref{lem:vector_stationary_risk},
Lemma~\ref{lem:vector_stationary_spectral} reduces stationary-kernel
optimization to separate eigendirections, and
Lemma~\ref{lem:vector_stationary_prefix} gives the sharp bound in each
direction. Proposition~\ref{prop:vector_stationary_floor} combines these
results to prove Theorem~\ref{thm:stationary_floor} and compare exponential
and uniform kernels.

\begin{lemma}[Reduction to eigendirections]
\label{lem:vector_stationary_spectral}
Let $\widetilde K_u:=U^\top K_uU$ and
$\widetilde M_t:=\sum_{u<t}\widetilde K_u$. The sum of the bias and variance
terms in \eqref{eq:vector_stationary_regret} is exactly
\begin{equation}
  \sum_{r=1}^p\sum_{j=1}^p
  \left\{
    \lambda_j\sum_{t\leq T}
      \bigl(\widetilde M_{t,rj}-\mathbf1\{r=j\}\bigr)^2
    +v_j\sum_{t\leq T}\sum_{u<t}
      \widetilde K_{u,rj}^2
  \right\}.
  \label{eq:vector_spectral_scalarization}
\end{equation}
Every off-diagonal profile is therefore an independent nonnegative
addition.  An optimizer over all real stationary matrix kernels may be
taken in the form
\[
  K_u=U\operatorname{diag}
  (\kappa_{u,1},\ldots,\kappa_{u,p})U^\top,
\]
and direction $i$ has bias coefficient $\lambda_i$ and innovation
coefficient $v_i$.
\end{lemma}

\begin{proof}
Orthogonal invariance of the trace and
$U^\top \widetilde\Sigma_{\theta}U=\operatorname{diag}(\lambda_1,\ldots,\lambda_p)$ give
\[
  \operatorname{tr}[(M_t-I_p)\widetilde\Sigma_{\theta}(M_t-I_p)^\top]
  =
  \sum_{r,j}\lambda_j
  \bigl(\widetilde M_{t,rj}-\mathbf1\{r=j\}\bigr)^2.
\]
Likewise,
$U^\top\Gamma U=\operatorname{diag}(v_1,\ldots,v_p)$ gives the second
term of \eqref{eq:vector_spectral_scalarization}.  If $r\neq j$, the
corresponding profile has target zero in both positive quadratic terms, so
setting it identically to zero cannot increase the cost.  The remaining
diagonal profiles are independent scalar problems.
\end{proof}

Each eigendirection gives a scalar stationary-kernel optimization problem.
The following bound allows horizon-dependent bias and variance coefficients
and identifies the optimal averaging scale.

\begin{lemma}[Sharp scalar stationary bound]
\label{lem:vector_stationary_prefix}
Let $\lambda_T,v_T>0$ and define
\[
  \omega_T:=\sqrt{\frac{v_TT}{\lambda_T}}.
\]
Suppose $\omega_T\to\infty$ and $\omega_T=o(T)$.
These conditions hold automatically for fixed positive bias and variance
coefficients, as in Theorem~\ref{thm:stationary_floor}.
For an arbitrary real lag
profile $\kappa_0,\ldots,\kappa_{T-1}$, let
$m_t:=\sum_{u<t}\kappa_u$.  Then
\begin{equation}
  \inf_{\kappa_0,\ldots,\kappa_{T-1}\in\RR}
  \left\{
    \lambda_T\sum_{t\leq T}(m_t-1)^2
    +v_T\sum_{t\leq T}\sum_{u<t}\kappa_u^2
  \right\}
  =\sqrt{\lambda_Tv_TT}\,(1+o(1)).
  \label{eq:vector_stationary_prefix}
\end{equation}
The lower bound is uniform over the lag profile, including profiles chosen
as a function of $T$.
\end{lemma}

\begin{proof}
Set $a_0:=1$ and
$a_j:=1-\sum_{u<j}\kappa_u$ for $1\leq j\leq T$, so
$\kappa_{j-1}=a_{j-1}-a_j$.  The objective in
\eqref{eq:vector_stationary_prefix} becomes
\begin{equation}
  J_T(a)=
  \lambda_T\sum_{j=1}^Ta_j^2
  +v_T\sum_{j=1}^T(T-j+1)(a_{j-1}-a_j)^2.
  \label{eq:vector_prefix_functional}
\end{equation}
Choose an integer $m_T$ such that
$\omega_T\ll m_T\ll T$; for example,
$m_T=\lfloor\sqrt{T\omega_T}\rfloor$.  Dropping the terms with
$j>m_T$ and using $T-j+1\geq T-m_T+1$ gives a constant-coefficient
quadratic chain with edge coefficient
$c_T:=v_T(T-m_T+1)$.  If $q_j$ is the minimum of its length-$j$ tail as
a multiple of $a_0^2$, dynamic programming gives
\[
  q_j=\frac{c_T(\lambda_T+q_{j-1})}
            {c_T+\lambda_T+q_{j-1}},
  \qquad q_0=0.
\]
The iteration increases to the positive fixed point
\[
  q_+
  =\frac{\sqrt{\lambda_T^2+4\lambda_Tc_T}-\lambda_T}{2}
  =\sqrt{\lambda_Tv_TT}\,(1+o(1)).
\]
Let
\[
  q_-:=\frac{-\sqrt{\lambda_T^2+4\lambda_Tc_T}-\lambda_T}{2}<0
\]
be the other fixed point.  The exact cross-ratio identity for this
M\"obius iteration is
\[
  \frac{q_j-q_+}{q_j-q_-}
  =\rho_T^j\frac{q_0-q_+}{q_0-q_-},
  \qquad
  \rho_T
  :=\frac{c_T^2}{(c_T+\lambda_T+q_+)^2}.
\]
Because $c_T/\lambda_T=\omega_T^2(1+o(1))$, we have
$1-\rho_T=\Theta(1/\omega_T)$ and
$q_+/|q_-|=1+o(1)$.  Since $0\leq q_j\leq q_+$, the cross-ratio identity
therefore gives
\[
  0\leq1-\frac{q_j}{q_+}\leq C\rho_T^j.
\]
The condition $m_T/\omega_T\to\infty$ yields
$q_{m_T}=q_+(1-o(1))$, proving the lower bound.

For the matching upper bound, take the normalized exponential profile
$\kappa_u=(1-\alpha_T)\alpha_T^u$ with effective width
$W_T:=(1+\alpha_T)/(1-\alpha_T)=2\omega_T$.  Writing
$D_T(\alpha):=\sum_{t=1}^T\alpha^{2t}$, direct geometric summation gives
\[
  J_T(a)
  =\lambda_TD_T(\alpha_T)
   +v_T\left(\frac{T}{W_T}-\frac{D_T(\alpha_T)}{W_T}\right).
\]
Since $W_T\to\infty$ and $W_T=o(T)$,
\[
  D_T(\alpha_T)
  =\left(\frac{W_T}{4}-\frac12+\frac{1}{4W_T}\right)
    (1-\alpha_T^{2T})
  =\frac{W_T}{4}(1+o(1)).
\]
It follows that
\[
  J_T(a)
  =\lambda_T\frac{W_T}{4}(1+o(1))
   +v_T\frac{T}{W_T}(1+o(1))
  =\sqrt{\lambda_Tv_TT}\,(1+o(1)).
\]
\end{proof}

Define the total bias and variance coefficients $\Lambda$ and $V$ and use
the nearest-integer version of $L^\star(T)$ in the proofs:
\begin{align}
  \Lambda&:=\sum_{i=1}^p\lambda_i=\operatorname{tr}(\widetilde\Sigma_{\theta}),
  &
  V&:=\sum_{i=1}^pv_i
      =(p+1)\operatorname{tr}(\widetilde\Sigma_{\theta})+p\sigma^2,
  \label{eq:vector_trace_scales}\\
  \zeta&:=\sum_{i=1}^p\sqrt{\lambda_iv_i},
  &
  \bar\zeta&:=\sqrt{\Lambda V},
  \label{eq:vector_stationary_scale}
\end{align}
\begin{equation}
  L^\star(T):=\left\lfloor\sqrt{\frac{3VT}{\Lambda}}+\frac12\right\rfloor.
  \label{eq:vector_common_window}
\end{equation}

\begin{proposition}[Stationary floor]
\label{prop:vector_stationary_floor}
For fixed $p$, $\Sigma_x\succ0$, $\Sigma_\theta\succ0$, and
$\sigma^2>0$, the infimum over deterministic real stationary matrix kernels
$K$ satisfies
\begin{equation}
  \inf_K\mathrm{Reg}_T(K)
  =\zeta\sqrt T+o(\sqrt T).
  \label{eq:vector_stationary_floor}
\end{equation}
The modewise normalized exponential
\begin{equation}
  K_u^{\exp}
  =
  U\operatorname{diag}\!\left(
    (1-\alpha_i)\alpha_i^u
  \right)_{i=1}^pU^\top,
  \qquad
  \frac{1+\alpha_i}{1-\alpha_i}
  =2\sqrt{\frac{v_iT}{\lambda_i}},
  \label{eq:vector_modewise_exponential}
\end{equation}
attains \eqref{eq:vector_stationary_floor} to leading order.

For comparison, the identity-scaled uniform kernel
\[
  K_u^{\mathrm{unif}}
  :=\frac1{L^\star(T)}I_p\,
    \mathbf1\{0\leq u<L^\star(T)\}
\]
has
\begin{equation}
  \mathrm{Reg}_T(K^{\mathrm{unif}})
  =\frac{2}{\sqrt3}\bar\zeta\sqrt T+o(\sqrt T).
  \label{eq:vector_uniform_cost}
\end{equation}
The kernel parameters may depend on the horizon and the fixed distributional
parameters, but one parameter setting is held fixed throughout the stream.
\end{proposition}

\begin{proof}
Lemma~\ref{lem:vector_stationary_spectral} reduces the optimization to
$p$ scalar problems. Applying
Lemma~\ref{lem:vector_stationary_prefix} with
$(\lambda_T,v_T)=(\lambda_i,v_i)$ gives
\[
  \inf_K\bigl\{\mathrm{Reg}_T(K)
    +\mathcal E_T^{\mathrm{Bayes}}\bigr\}
  =
  \sqrt T\sum_{i=1}^p\sqrt{\lambda_iv_i}
  +o(\sqrt T),
\]
where fixed $p$ permits summing the directionwise remainders.  The
exponential profile used in the upper bound of that lemma in each
eigendirection is exactly \eqref{eq:vector_modewise_exponential}.

It remains to control the Bayes subtraction.  Let
$\widetilde B_t:=\Sigma_x^{-1/2}B_t\Sigma_x^{-1/2}
=\sum_{s\leq t}\widetilde x_s\widetilde x_s^\top$ and
$D_0:=\sigma^2\widetilde\Sigma_{\theta}^{-1}$.  Then
\[
  e_t^{\mathrm{Bayes}}
  =\sigma^2\EE\operatorname{tr}(\widetilde B_t+D_0)^{-1}.
\]
For all sufficiently large $t$, $\widetilde B_t$ is invertible and
\[
  0\preceq
  \widetilde B_t^{-1}-(\widetilde B_t+D_0)^{-1}
  \preceq\widetilde B_t^{-1}D_0\widetilde B_t^{-1}
  \preceq\|D_0\|_{\mathrm{op}}\widetilde B_t^{-2}.
\]
The fixed-dimensional inverse-Wishart identities give
\[
  \EE\operatorname{tr}(\widetilde B_t^{-1})
  =\frac{p}{t-p-1},
  \qquad
  \EE\operatorname{tr}(\widetilde B_t^{-2})=O(t^{-2}),
  \qquad p\ \text{fixed}.
\]
The finitely many earlier terms are bounded by
$\operatorname{tr}(\widetilde\Sigma_{\theta})$, so
\begin{equation}
  \mathcal E_T^{\mathrm{Bayes}}
  =p\sigma^2\log T+O(1)
  =o(\sqrt T).
  \label{eq:vector_bayes_baseline_fixed_p}
\end{equation}
Subtracting this term proves \eqref{eq:vector_stationary_floor}.

For the uniform kernel, summing its cold-start bias and variance terms
gives exactly
\[
  \Lambda\left(\frac L3-\frac12+\frac1{6L}\right)
  +V\left(\frac TL-\frac12+\frac1{2L}\right).
\]
The nearest-integer choice \eqref{eq:vector_common_window} minimizes its
leading terms, yielding $\tfrac{2}{\sqrt3}\bar\zeta\sqrt T+o(\sqrt T)$. Equation
\eqref{eq:vector_bayes_baseline_fixed_p} proves
\eqref{eq:vector_uniform_cost}.
\end{proof}

\subsubsection{Kernel-form projection}
\label{app:case1_score_projection}

We now connect the kernel analysis to the architecture classes through
belief geometry (Figure~\ref{fig:belief_geometry}).
Lemma~\ref{lem:vector_degree_two_projection} shows that, for degree-two
beliefs, projecting onto kernel form with deterministic kernel weights can
only decrease prediction error. This allows the attention and SSM proofs
below to focus on the stationary kernels their architectures realize.

Let $\mathcal Q_{2,t}$ be the space of $\RR^p$-valued polynomials of
total degree at most two in $(\widetilde x_{1:t},y_{1:t})$, and let
\[
  \mathcal S_t
  :=
  \left\{\sum_{s\leq t}K_s\widetilde x_sy_s:
  K_s\in\RR^{p\times p}\right\}.
\]
Here the matrices are deterministic.  Thus $\mathcal S_t$ is contained in
the whitened kernel-form space $\Sigma_x^{1/2}\mathcal K_t$, which also allows
regressor-dependent kernel weights.  This projection is a further
reduction, distinct from $\Pi_t$: its remainder need not lie outside kernel
form.

\begin{lemma}[Projection onto deterministic kernel form]
\label{lem:vector_degree_two_projection}
Every $f\in\mathcal Q_{2,t}$ has an orthogonal decomposition
$f=d+n$ with $d\in\mathcal S_t$ and $n\perp\mathcal S_t$.  Moreover,
$\EE[n^\top\widetilde\theta]=0$.  Consequently,
\begin{equation}
  \EE\|f-\widetilde\theta\|_2^2
  =\EE\|d-\widetilde\theta\|_2^2+\EE\|n\|_2^2.
  \label{eq:vector_score_pythagoras}
\end{equation}
\end{lemma}

\begin{proof}
By linearity, it suffices to check scalar monomials of degree at most two.
Constants, design-only monomials, and terms containing two
response factors have zero inner product with $\widetilde\theta$ by centering and the
joint sign symmetry
$(\widetilde\theta,\epsilon_{1:t})\mapsto-(\widetilde\theta,\epsilon_{1:t})$.
Linear response terms also have zero inner product with $\widetilde\theta$, because
$\EE[\widetilde\theta y_s]=\widetilde\Sigma_{\theta}\EE[\widetilde x_s]=0$.  Finally,
\begin{equation}
  \EE[\widetilde\theta \widetilde x_{r,i}y_s]
  =\widetilde\Sigma_{\theta}e_i\,\mathbf1\{r=s\}.
  \label{eq:vector_score_correlation}
\end{equation}
Since $\widetilde\Sigma_{\theta}\succ0$, the only degree-two monomials correlated with $\widetilde\theta$ are
the coordinates of the evidence atoms $\widetilde x_sy_s$.

The same sign symmetry and design parity show that every remaining monomial
is orthogonal to every evidence atom.  Terms with zero or two response factors
follow from the joint response-sign symmetry after multiplication by
$\widetilde x_ky_k$.  For a linear response term,
$\EE[y_r\widetilde x_ky_k]=0$: when $r=k$, reflect
$(\widetilde x_k,\epsilon_k)\mapsto(-\widetilde x_k,-\epsilon_k)$; when $r\neq k$, conditioning
on the designs leaves $\EE[y_ry_k\mid \widetilde x]=\widetilde x_r^\top \widetilde\Sigma_{\theta}\widetilde x_k$, whose remaining
expectation vanishes by design parity.  For a cross-atom $\widetilde x_{r,i}y_s$ with
$r\neq s$, conditioning on the designs again reduces its product with
$\widetilde x_ky_k$ to design monomials.  If $k\in\{r,s\}$, the other independent
design vector has odd degree; if $k\notin\{r,s\}$, $\widetilde x_r$ remains an
independent centered factor.  Orthogonal projection onto $\mathcal S_t$
therefore leaves a remainder orthogonal both to $\mathcal S_t$ and to
$\widetilde\theta$, proving \eqref{eq:vector_score_pythagoras}.
\end{proof}

\subsubsection{One-layer attention}
\label{app:case1_attention_kernels}

We next characterize one-layer attention through the aggregation kernels it realizes.
Lemma~\ref{lem:vector_attention_projection} shows that the aggregation kernel
assigns the same matrix weight, of rank at most $Hd$, to every past evidence
atom within the context window.
Lemma~\ref{lem:vector_uniform_body} then shows that, at sufficient width,
uniform kernels achieve the leading-order optimum within the one-layer
attention class at every context length.

\begin{lemma}[One-layer attention projection and kernel rank]
\label{lem:vector_attention_projection}
Fix $d_e,d,H,L\geq1$.  For every
$g\in\mathcal F_{\mathrm{att}}(1,d,H,L)$, there are deterministic
matrices $K,K_0\in\RR^{p\times p}$ such that
\begin{equation}
  \widetilde\theta_t(g)
  =
  K\!\sum_{s=\max(1,t-L+1)}^{t-1}\widetilde x_sy_s
  +K_0\widetilde x_ty_t+\widetilde\theta_t^\perp,
  \qquad
  \widetilde\theta_t^\perp\perp\mathcal S_t,
  \quad
  \EE[(\widetilde\theta_t^\perp)^\top\widetilde\theta]=0.
  \label{eq:vector_attention_projection}
\end{equation}
Here $\widetilde\theta_t^\perp$ is the orthogonal remainder after projecting
$\widetilde\theta_t(g)$ onto $\mathcal S_t$.
Moreover, the common past-lag kernel weight admits a headwise factorization
\begin{equation}
  K=\sum_{b=1}^H A_bC_b,
  \qquad
  A_b\in\RR^{p\times d},
  \quad
  C_b\in\RR^{d\times p},
  \label{eq:vector_attention_head_factorization}
\end{equation}
and hence $\operatorname{rank}(K)\leq Hd$.  When $d_e\geq2p+1$ and
$Hd\geq p$, a distributed-head construction realizes
$K=K_0=U\operatorname{diag}(a_1,\ldots,a_p)U^\top$ for any real
$a_1,\ldots,a_p$, with $\widetilde\theta_t^\perp=g_t^\perp=0$.
\end{lemma}

\begin{proof}
Whitening is an invertible linear change to the two regressor blocks of each
token, so it can be absorbed into the initial embedding and the affine
query, key, and value maps without changing $(d,H,L)$.  For the lower-bound
projection, homogenize the whitened token expressions as linear forms in
$\bar z_s:=(1,\widetilde x_{s+1},y_s,\widetilde x_s)$.  This represents the legal
affine maps through a constant coordinate. Let
$J_+,J_0\in\RR^{(2p+2)\times p}$ and $e_y\in\RR^{2p+2}$ select
$\widetilde x_{s+1}$, $\widetilde x_s$, and $y_s$, respectively, from $\bar z_s$.  After
composing the embedding with the affine query and key maps, and combining
the value, output, and final prediction maps into a scalar map, head $b$ contributes
\[
  a_b^\top\bar z_t
  +\sum_{s=\max(1,t-L+1)}^t
    \left[
      p_{t-s}^{(b)}
      +(\mathsf Q_b\bar z_t)^\top(\mathsf K_b\bar z_s)
    \right]
    (v_b^\top\bar z_s),
\]
where $\mathsf Q_b,\mathsf K_b\in\RR^{d\times(2p+2)}$ and
$v_b\in\RR^{2p+2}$ are deterministic.  For
$s<t$, multiplication by $\widetilde x_{t+1}$ and conditional expectation over the
fresh regressor give
\[
  A_b(\mathsf K_b\bar z_s)(v_b^\top\bar z_s),
  \qquad
  A_b:=(\mathsf Q_bJ_+)^\top\in\RR^{p\times d}.
\]
The positional-bias term has zero fresh-regressor projection.  Define
\[
  C_b
  :=(e_y^\top v_b)\mathsf K_bJ_0
    +(\mathsf K_be_y)v_b^\top J_0
  \in\RR^{d\times p}.
\]
The two summands place $\widetilde x_s$ in the key and $y_s$ in the value, or vice
versa.  Thus the kernel weight on the same-time evidence atom $\widetilde x_sy_s$ from
head $b$ is $A_bC_b$.  All remaining source-token monomials are placed in
$\widetilde\theta_t^\perp$ by Lemma~\ref{lem:vector_degree_two_projection}.  Because the maps
are shared over positions, this weight is independent of $s$ and of the
lag.  Summing heads gives
\eqref{eq:vector_attention_head_factorization} and
$\operatorname{rank}(K)\leq Hd$.

At $s=t$, the positional term contributes only a degree-zero nuisance.  The
bilinear score--value term is cubic in $\bar z_t$; after multiplication by
the fresh Gaussian regressor $\widetilde x_{t+1}$, only terms with one or three internal
fresh-regressor factors survive conditional expectation. The result has degree at most two in
$(\widetilde x_t,y_t)$ and may supply a separate current-atom kernel weight $K_0$.
Lemma~\ref{lem:vector_degree_two_projection} places every other resulting
monomial in $\widetilde\theta_t^\perp$, which proves
\eqref{eq:vector_attention_projection}.

For the witness, use the first $2p+1$ residual coordinates for the embedding
\[
  e_s^{(0)}=(\widetilde x_{s+1},\widetilde x_s,y_s,0_{d_e-(2p+1)})\in\RR^{d_e}.
\]
Since $Hd\geq p$, partition $[p]$ into sets
$I_1,\ldots,I_H$ with $|I_b|\leq d$.  Head $b$ uses the unnormalized score
\[
  \sum_{i\in I_b}a_i(u_i^\top \widetilde x_{t+1})(u_i^\top \widetilde x_s)
\]
where $u_i$ is column $i$ of $U$.  Its value is $y_s$, and all
head-output maps write into the shared label
coordinate.  In one designated head choose the relative positional bias
$p_0=-1$ and $p_u=0$ for $1\leq u<L$; set every positional bias in the
other heads to zero.  The final map reads the label coordinate.  Summing the
heads recovers $\widetilde x_{t+1}^\top K \widetilde x_s$, while the designated lag-zero bias
cancels the residual contribution $y_t$.  Hence the output is exactly
\[
  g_t
  =y_t+\sum_{s=\max(1,t-L+1)}^t
    \left[\widetilde x_{t+1}^\top K \widetilde x_s-\mathbf1\{s=t\}\right]y_s
  =\widetilde x_{t+1}^\top
    K\sum_{s=\max(1,t-L+1)}^t\widetilde x_sy_s.
\]
Each head uses at most $d$ query/key coordinates, with no nuisance or
fresh-regressor residual.  Taking $a_i=1/L$ recovers the normalized uniform
kernel.
\end{proof}

Lemma~\ref{lem:vector_attention_projection} allows the current evidence atom
to receive a different kernel weight from past atoms. The next lemma accounts
for the additional error this freedom introduces and shows that uniform
kernels remain optimal to leading order within this class at every context length.

\begin{lemma}[One-layer attention leading-order optimal kernel]
\label{lem:vector_uniform_body}
Fix the covariance parameters above and take $1\leq L\leq T$; larger
contexts give the same windows as $L=T$.  Write
$\tau_t:=\min(t,L)$ and $N_j:=\sum_{t\leq T}\tau_t^j$ for $j=1,2$.
The best uniform kernel has
\begin{equation}
  \bar K:=U\operatorname{diag}(a_1,\ldots,a_p)U^\top,
  \qquad
  a_i:=\frac{\lambda_iN_1}{\lambda_iN_2+v_iN_1},
  \qquad
  K_u^{\mathrm{unif},\star}:=\bar K\mathbf1\{0\leq u<L\},
  \label{eq:vector_equal_tap_kernel}
\end{equation}
and cumulative mean-squared belief-estimation error
\begin{equation}
  \mathcal E_{L,T}
  :=\sum_{i=1}^p\left(
    \lambda_iT-\frac{\lambda_i^2N_1^2}{\lambda_iN_2+v_iN_1}
  \right).
  \label{eq:vector_equal_tap_cost}
\end{equation}
For every $d_e,d,H\geq1$,
\begin{equation}
  \mathrm{Reg}_T\bigl(\mathcal F_{\mathrm{att}}(1,d,H,L)\bigr)
  +\mathcal E_T^{\mathrm{Bayes}}
  \geq\mathcal E_{L,T}-O(L/T).
  \label{eq:vector_attention_all_context}
\end{equation}
If $Hd\geq p$ and $d_e\geq2p+1$, this uniform kernel can be realized by
an attention model in the class, giving the reverse bound with $\mathcal E_{L,T}$ in place of
$\mathcal E_{L,T}-O(L/T)$.  In the original coordinates its kernel weights within the context window
are
\begin{equation}
  \kappa_u^{\mathrm{unif},\star}=\Sigma_x^{-1/2}\bar K\Sigma_x^{-1/2}
  =\frac1L\Sigma_x^{-1}+O(L^{-2}+T^{-1}),
  \qquad 0\leq u<L,
  \label{eq:vector_uniform_kernel_correction}
\end{equation}
with the same matrix correction at every lag and zero weights otherwise.
The bounds above are uniform in $L,d_e,d,H$ at fixed covariance
parameters; matrix remainders are measured in Frobenius norm.
In particular, when $L=\Theta(\sqrt T)$,
\begin{equation}
  \mathcal E_{L,T}=\frac{\Lambda L}{3}+\frac{VT}{L}+O(1),
  \label{eq:vector_uniform_body_cost}
\end{equation}
and the sufficient-width class matches this kernel's regret up to $O(1)$.
\end{lemma}

\begin{proof}
Set $\Delta:=K_0-K$.  The two projection identities and
Lemma~\ref{lem:vector_stationary_risk} give the exact cost
\begin{align}
  \mathrm{Reg}_T(g)+\mathcal E_T^{\mathrm{Bayes}}
  &=\sum_{t\leq T}\Bigl[
    \| (\tau_tK+\Delta-I_p)\widetilde\Sigma_{\theta}^{1/2}\|_F^2
    +(\tau_t-1)\|K\Gamma^{1/2}\|_F^2\notag\\
  &\qquad+\|(K+\Delta)\Gamma^{1/2}\|_F^2
    +\EE\|\widetilde\theta_t^\perp\|_2^2+\EE[(g_t^\perp)^2]\Bigr].
  \label{eq:vector_attention_exact_cost}
\end{align}
For equal kernel weights, Lemma~\ref{lem:vector_stationary_spectral} gives the directionwise quadratic
$\lambda_iT-2\lambda_iN_1a_i+(\lambda_iN_2+v_iN_1)a_i^2$.
Its minimizer and value are
\eqref{eq:vector_equal_tap_kernel}--\eqref{eq:vector_equal_tap_cost},
and Lemma~\ref{lem:vector_attention_projection} supplies the legal witness.
When $L=1$, only $K_0$ remains, so this is also the exact class optimum at
sufficient width.  Henceforth take $L\geq2$.

\emph{Cost of changing the current-atom kernel weight.}
In the cubic expansion used in the attention projection proof, the common
evidence contribution is $\widetilde x_{t+1}^\top K \widetilde x_sy_s$. The only other placements of its
three factors that contribute to the current-atom kernel weight at $s=t$ are
\[
  y_t \widetilde x_{s+1}^\top D_1\widetilde x_s,
  \qquad y_s \widetilde x_{s+1}^\top D_2\widetilde x_t,
  \qquad D_1+D_2=\Delta,
\]
where $D_1,D_2$ collect the coefficients across heads.
Lower-degree terms cannot contribute to this kernel weight, and terms with
three fresh factors contribute only a constant to the implicit belief,
orthogonal to the evidence atoms.

For $s\in J_t$ with $s\leq t-2$, let $R_{t,s}$ be the sum of these two
terms.  Under the sign symmetries
$(\widetilde x_j,\epsilon_j)\mapsto(-\widetilde x_j,-\epsilon_j)$ it is odd exactly at the
three sites $\{s,s+1,t\}$, and it is also odd under the joint response-sign
symmetry used in Lemma~\ref{lem:vector_degree_two_projection}.
Only source $s$ contains both earlier sites, so these are the only legal
degree-three terms with both properties.  Different $s$ give orthogonal
components of $g_t^\perp$: they contain no fresh regressor and their sign
patterns differ.  Affine terms, positional biases, and residuals have too
few factors to cancel them.  We omit $s=t-1$, where the sites overlap.
The Gaussian moments used above give
\[
  \EE R_{t,s}^2
  =\tfrac12\|\Delta\Gamma^{1/2}\|_F^2
   +\tfrac12\|(D_1-D_2)((\Lambda+\sigma^2)I_p-\widetilde\Sigma_{\theta})^{1/2}\|_F^2.
\]
Since $(\Lambda+\sigma^2)I_p-\widetilde\Sigma_{\theta}\succ0$, orthogonality yields
\begin{equation}
  \sum_{t\leq T}\EE[(g_t^\perp)^2]
  \geq\frac m2\|\Delta\Gamma^{1/2}\|_F^2,
  \qquad m:=\sum_{t\leq T}(\tau_t-2)_+.
  \label{eq:vector_attention_current_cost}
\end{equation}

\emph{Comparison with the uniform kernel.}
Write $\widetilde K=U^\top KU$ and
$\widetilde\Delta=U^\top\Delta U$.  To complete the same quadratic, put
\begin{align*}
  P_i&:=\lambda_iN_2+v_iN_1,
  &J_i&:=\lambda_iN_1+v_iT,\\
  f_i&:=\frac{\lambda_i^2(TN_2-N_1^2)}{P_i},
  &h_i&:=T(\lambda_i+v_i)-\frac{J_i^2}{P_i}+\frac{mv_i}{2}.
\end{align*}
Combining \eqref{eq:vector_attention_exact_cost} and
\eqref{eq:vector_attention_current_cost}, then completing squares column
by column, gives
\begin{equation}
  \mathrm{Reg}_T(g)+\mathcal E_T^{\mathrm{Bayes}}
  \geq\mathcal E_{L,T}+\sum_{i=1}^p\Bigl[
    P_i\bigl\|(\widetilde K-a_iI_p
                 +(J_i/P_i)\widetilde\Delta)e_i\bigr\|_2^2
    +h_i\|\widetilde\Delta e_i\|_2^2
    -2f_i\widetilde\Delta_{ii}\Bigr].
  \label{eq:vector_attention_square_completion}
\end{equation}
Here $e_i$ is the $i$th coordinate vector.  Uniformly for $2\leq L\leq T$,
$P_i\asymp TL^2$, $J_i/P_i=O(L^{-1})$, and $0\leq f_i\leq\lambda_iL$.
Indeed, $TL/2\leq N_1\leq TL$, $N_2\geq TL^2/3$, and
$TN_2-N_1^2\leq T\sum_t(L-\tau_t)^2\leq TL^3/3$.
Minimizing only the innovation part in
\eqref{eq:vector_attention_exact_cost} also gives
\[
  h_i-\frac{mv_i}{2}\geq v_iT(1-T/N_1)\geq0.
\]
Together with $m=(L-2)(T-(L+1)/2)$, this implies
$h_i\geq v_iT(L-1)/12$.  For $L=2$, where $m=0$, the innovation term
alone gives $h_i\geq v_iT/3$.
Completing the remaining square in $\widetilde\Delta e_i$ therefore
lowers $\mathcal E_{L,T}$ by at most
$\sum_i f_i^2/h_i=O(L/T)$, proving
\eqref{eq:vector_attention_all_context}.  The upper bound comes from the
uniform witness, not from assuming that this relaxed minimizer is legal.

At sufficient width, the same squares show that any predictor within
$o(L/T)$ of the class infimum satisfies
\begin{equation}
  \|K-\bar K\|_F=O((T\sqrt L)^{-1}),
  \qquad \|\Delta\|_F=O(T^{-1}).
  \label{eq:vector_attention_optimal_coefficients}
\end{equation}
To obtain the simpler kernel expression, use
\[
  a_i=\left(\frac{N_2}{N_1}+\frac{v_i}{\lambda_i}\right)^{-1},
  \qquad
  0\leq L-\frac{N_2}{N_1}
  =\frac{L(L^2-1)}{6N_1}\leq\frac{L^2-1}{3T},
  \qquad \frac{N_2}{N_1}\geq\frac L3.
\]
These imply $a_i=1/L+O(L^{-2}+T^{-1})$, and transforming back by
$\Sigma_x^{-1/2}$ proves \eqref{eq:vector_uniform_kernel_correction}.
When $L=\Theta(\sqrt T)$, $\bar K-I_p/L=O(T^{-1})$ and the
uniform-kernel quadratic has coefficient $P_i=\Theta(TL^2)$ on the squared amplitude.
Optimizing its amplitudes therefore improves the normalized uniform cost
computed in Proposition~\ref{prop:vector_stationary_floor} by only $O(1)$.
This proves \eqref{eq:vector_uniform_body_cost}.
\end{proof}

\paragraph{Full-context attention.}
The distinction between an optimal amplitude and the normalization $1/L$
matters at $L=T$.  Here $N_1=T(T+1)/2$, $N_2=T(T+1)(2T+1)/6$, so
\[
  a_i=\frac{3\lambda_i}{\lambda_i(2T+1)+3v_i}
      =\frac{3}{2T}+O(T^{-2}),
  \qquad \mathcal E_{T,T}=\frac{\Lambda T}{4}+O(1).
\]
Thus, at $Hd\geq p$ and $d_e\geq2p+1$,
\begin{equation}
  \mathrm{Reg}_T\bigl(\mathcal F_{\mathrm{att}}(1,d,H,T)\bigr)
  =\frac{\Lambda T}{4}-p\sigma^2\log T+O(1).
  \label{eq:vector_full_context_rate}
\end{equation}
Moreover, every predictor within $o(1)$ of the class infimum has
$K=3I_p/(2T)+O(T^{-3/2})$ by
\eqref{eq:vector_attention_optimal_coefficients}.  The current-atom kernel weight is
also of order $1/T$ in every direction.  To see its lower bound, retain
the $m$-term in \eqref{eq:vector_attention_square_completion} and use
$m/T^2\to1/2$, $f_i/T\to\lambda_i/4$, and
$h_i-mv_i/2\geq0$.  Comparison with the uniform witness gives
\[
  \|(T\Delta-\widetilde\Sigma_{\theta}\Gamma^{-1})\Gamma^{1/2}\|_F^2
  \leq\operatorname{tr}(\widetilde\Sigma_{\theta}^2\Gamma^{-1})+o(1),
\]
so $\|T\Delta-\widetilde\Sigma_{\theta}\Gamma^{-1}\|_{\mathrm{op}}\leq\Lambda/(\Lambda+\sigma^2)+o(1)$.
Since $\Lambda/(\Lambda+\sigma^2)<1$ and $0\preceq \widetilde\Sigma_{\theta}\Gamma^{-1}\prec I_p/2$, eventually
\begin{equation}
  \frac1{2T}I_p\preceq\frac{K_0+K_0^\top}{2},
  \qquad \|K_0\|_{\mathrm{op}}\leq\frac3T.
  \label{eq:vector_full_context_current_tap}
\end{equation}
Transforming back by $\Sigma_x^{-1/2}$ gives the $\Theta(1/T)$ evidence weights
described in the main text, including the current atom.
The kernel mass is
$tK+\Delta=3tI_p/(2T)+O(T^{-1/2})$, uniformly for $t\leq T$.
It remains bounded away from $I_p$ throughout an initial constant fraction
of the horizon, explaining the linear cold-start bias.

\subsubsection{Fixed-transition SSMs}
\label{app:case1_ssm_kernels}

We next apply the same projection to fixed-transition SSMs.
Lemma~\ref{lem:vector_ssm_projection} shows that the aggregation kernels
obtained after projection are stationary and constructs an exponential
kernel that attains the floor in Theorem~\ref{thm:stationary_floor}
at sufficient width.

\begin{lemma}[Fixed-transition SSM projection]
\label{lem:vector_ssm_projection}
Fix $d_e,d,H\geq1$.  Concatenate the $H$ heads
of the one-layer, identity-convolution, fixed-transition SSM into a state with
block-diagonal transition $\Phi\in\RR^{Hd\times Hd}$.  Every
$g\in\mathcal F_{\mathrm{SSM}}(1,d,H,1)$ has
\begin{equation}
  \widetilde\theta_t(g)
  =\sum_{u<t}K_u\widetilde x_{t-u}y_{t-u}+\widetilde\theta_t^\perp,
  \qquad
  \widetilde\theta_t^\perp\perp\mathcal S_t,
  \quad
  \EE[(\widetilde\theta_t^\perp)^\top\widetilde\theta]=0,
  \label{eq:vector_ssm_projection}
\end{equation}
where
\begin{equation}
  K_u=W_{\mathrm{read}}\Phi^uW_{\mathrm{write}}
  \quad(u\geq1),
  \qquad
  K_0\in\RR^{p\times p}\ \text{is separate},
  \label{eq:vector_ssm_impulse_response}
\end{equation}
for deterministic effective read and evidence-write maps
$W_{\mathrm{read}}\in\RR^{p\times Hd}$ and
$W_{\mathrm{write}}\in\RR^{Hd\times p}$. Hence every class member obeys the
global stationary floor.  If $Hd\geq p$ and $d_e\geq2p+2$, the
class contains the modewise exponential witness
\eqref{eq:vector_modewise_exponential}.
\end{lemma}

\begin{proof}
Whitening the two regressor blocks is absorbed into the initial embedding and
the SSM's affine write and read maps.  After concatenating heads,
\[
  h_t=\Phi h_{t-1}+u_t
  =\sum_{s\leq t}\Phi^{t-s}u_s.
\]
For $s<t$, computing the implicit belief from the ordinary token--state read gives
a fixed map $W_{\mathrm{read}}$ from the state. More
explicitly, after the residual update and final prediction the output can be
written
\[
  g_t=r_0^\top e_t^{(0)}
      +\widetilde x_{t+1}^\top W_{\mathrm{read}}h_t+\bar q_t^\top h_t,
\]
where $r_0$ is deterministic and $\bar q_t$ is
$\mathcal G_t^-$-measurable. The direct residual term contributes to the
implicit belief only a component orthogonal to the evidence atoms. On the old-state
part $\Phi h_{t-1}$, the last term,
$\bar q_t^\top\Phi h_{t-1}$, has zero fresh-regressor projection.
For a historical write at time $s<t$, the token field $\widetilde x_{s+1}$ belongs to
$\mathcal G_t^-$ at the eventual read time.  We may therefore apply
Lemma~\ref{lem:vector_degree_two_projection} in $\mathcal Q_{2,t}$;
the cross-atom parity calculation in its proof shows that terms involving
$\widetilde x_{s+1}$ remain orthogonal to the evidence atoms. Applied componentwise, the lemma thus
decomposes the write's bilinear part as
$W_{\mathrm{write}}\widetilde x_sy_s$ plus channels orthogonal to the evidence atoms. Deterministic
linear maps preserve this orthogonality, so the lag-$u$ kernel weight is
$W_{\mathrm{read}}\Phi^uW_{\mathrm{write}}$ for $u\geq1$.

At $s=t$, the fresh regressor also enters the current bilinear write, and
computing the implicit belief can produce a separate $K_0$. The current token
contributes at most degree three to $g_t$; multiplying by the fresh vector
and taking conditional expectation leaves a polynomial of degree at most two in
$(\widetilde x_t,y_t)$, so Lemma~\ref{lem:vector_degree_two_projection} applies.
All remaining projected channels form $\widetilde\theta_t^\perp$, and the component outside
the fresh-regressor subspace forms $g_t^\perp$.  Equations
\eqref{eq:vector_score_pythagoras} and
\eqref{eq:vector_fresh_query_mse} make both nonnegative additions to the
stationary-kernel cost.

For the witness, use the first $2p+1$ residual coordinates for the embedding
\[
  e_s^{(0)}=(\widetilde x_{s+1},\widetilde x_s,y_s,0_{d_e-(2p+1)})
  \in\RR^{d_e}.
\]
Distribute the $p$ eigendirections among the $H$ heads, with at most $d$
directions per head.  This is possible when $Hd\geq p$.
For $a_{i,t}:=u_i^\top \widetilde x_t$, assign one state coordinate satisfying
\[
  h_{i,t}=\alpha_i h_{i,t-1}+a_{i,t}y_t.
\]
The bilinear write is legal in the identity-convolution class.  Read that
coordinate with $(1-\alpha_i)a_{i,t+1}$, sum the head-local outputs into one
clean residual coordinate, and let the final prediction read that coordinate:
\[
  g_t
  =
  \sum_{i=1}^p
  (1-\alpha_i)a_{i,t+1}
  \sum_{u<t}\alpha_i^ua_{i,t-u}y_{t-u}.
\]
Each head transition may contain distinct diagonal entries, and the modes
do not communicate.  The resulting kernel is exactly
\eqref{eq:vector_modewise_exponential}, with
$\widetilde\theta_t^\perp=g_t^\perp=0$.
\end{proof}

\subsubsection{Proof of Proposition~\ref{prop:stationary_pair}}
\label{app:case1_stationary_pair}

We first use the two architecture projections to establish the width
floors in Lemma~\ref{lem:vector_width_thresholds}.
We then combine these floors with the kernel bounds and constructions to prove
Proposition~\ref{prop:stationary_pair}.

\begin{lemma}[Width floors]
\label{lem:vector_width_thresholds}
Define
\begin{equation}
  d_i:=\frac{\lambda_iv_i}{\lambda_i+v_i},
  \label{eq:vector_width_deficiency_matrix}
\end{equation}
and denote their increasing rearrangement by
$d_{(1)}\leq\cdots\leq d_{(p)}$.  For every
$g\in\mathcal F_{\mathrm{att}}(1,d,H,L)$, let $K$ be the common
past-lag kernel weight in Lemma~\ref{lem:vector_attention_projection}.  Then
\begin{equation}
  \mathrm{Reg}_T(g)
  \geq
  T\sum_{j=1}^{p-\operatorname{rank}(K)}d_{(j)}
  -\mathcal E_T^{\mathrm{Bayes}}
  \geq
  T\sum_{j=1}^{(p-Hd)_+}d_{(j)}
  -\mathcal E_T^{\mathrm{Bayes}}.
  \label{eq:vector_attention_dimension_floor}
\end{equation}
Similarly, every $g\in\mathcal F_{\mathrm{SSM}}(1,d,H,1)$ satisfies
\begin{equation}
  \mathrm{Reg}_T(g)
  \geq
  T\sum_{j=1}^{(p-Hd)_+}d_{(j)}
  -\mathcal E_T^{\mathrm{Bayes}}.
  \label{eq:vector_ssm_dimension_floor}
\end{equation}
Consequently, at fixed $p,\widetilde\Sigma_{\theta},\sigma^2$, both classes
incur $\Omega(T)$ regret whenever $Hd<p$. This holds uniformly over $d_e$
and, for attention, $L$, including horizon-dependent choices.
\end{lemma}

\begin{proof}
For attention, let $P$ be the orthogonal projector onto
$\operatorname{col}(K)^\perp$ and put
\[
  a_t
  :=K\!\sum_{s=\max(1,t-L+1)}^{t-1}\widetilde x_sy_s+K_0\widetilde x_ty_t.
\]
Since $a_t\in\mathcal S_t$, the projection lemma gives
\[
  \EE\|\widetilde\theta_t(g)-\widetilde\theta\|_2^2
  =\EE\|a_t-\widetilde\theta\|_2^2+\EE\|\widetilde\theta_t^\perp\|_2^2
  \geq\EE\|P(a_t-\widetilde\theta)\|_2^2.
\]
Equation~\eqref{eq:vector_fresh_query_mse} likewise discards $g_t^\perp$ as
a nonnegative contribution.  Since $PK=0$, keeping only the projected
current-atom error gives, at every time,
\[
  \EE\|\widetilde\theta_t(g)-\widetilde\theta\|_2^2
  \geq
  \operatorname{tr}\!\left[
    P(K_0-I_p)\widetilde\Sigma_{\theta}(K_0-I_p)^\top P
  \right]
  +
  \operatorname{tr}\!\left[
    PK_0\Gamma K_0^\top P
  \right].
\]

For the SSM, instead let $P$ project onto
$\operatorname{col}(W_{\mathrm{read}})^\perp$.  Then
$\operatorname{rank}(P)\geq(p-Hd)_+$ and $PK_u=0$ for every $u\geq1$.
Lemma~\ref{lem:vector_stationary_risk} and the SSM projection give the same
displayed lower bound.

This SSM argument also allows independent write/read feature widths.
Linear combinations of bilinear write features remain quadratic, while
reads pairing affine token features with linear state projections reduce
to the same token--state read. Thus the projection proof still gives
\eqref{eq:vector_ssm_impulse_response} through an $Hd$-dimensional state.

Set $X:=PK_0$, so $PX=X$.  Minimizing the displayed quadratic over this
relaxed variable gives
\[
  X^\star=P\widetilde\Sigma_{\theta}(\widetilde\Sigma_{\theta}+\Gamma)^{-1}
\]
and minimum
\[
  \operatorname{tr}\!\left[
    P\{\widetilde\Sigma_{\theta}-\widetilde\Sigma_{\theta}(\widetilde\Sigma_{\theta}+\Gamma)^{-1}\widetilde\Sigma_{\theta}\}
  \right]
  =\operatorname{tr}(PD_{\mathrm{floor}}),
\]
where
\[
  D_{\mathrm{floor}}:=\widetilde\Sigma_{\theta}\Gamma
  (\widetilde\Sigma_{\theta}+\Gamma)^{-1}
  =U\operatorname{diag}(d_1,\ldots,d_p)U^\top,
\]
with the $d_i$ defined in the lemma statement.
This follows by expanding the quadratic and completing the square
in the positive-definite metric $\widetilde\Sigma_{\theta}+\Gamma$.

Let $u_i$ be the $i$th column of $U$, and set
$q_i:=u_i^\top Pu_i$.  Then $0\leq q_i\leq1$ and
$\sum_iq_i=\operatorname{rank}(P)$.  Moreover,
\[
  \operatorname{tr}(PD_{\mathrm{floor}})=\sum_{i=1}^pd_iq_i.
\]
For attention,
$\operatorname{rank}(P)=p-\operatorname{rank}(K)$, so Ky Fan's minimum
principle gives the first inequality in
\eqref{eq:vector_attention_dimension_floor};
$\operatorname{rank}(K)\leq Hd$ gives the second.  For the SSM,
$\operatorname{rank}(P)\geq(p-Hd)_+$ gives
\eqref{eq:vector_ssm_dimension_floor}.  Summing over time and subtracting the
Bayes baseline proves both finite-horizon bounds.  At fixed $p,\widetilde\Sigma_{\theta},\sigma^2$,
all $d_i$ are positive and $\mathcal E_T^{\mathrm{Bayes}}=o(T)$, so $Hd<p$ gives
the claimed $\Omega(T)$ lower bounds.
\end{proof}

We now complete the proof of Proposition~\ref{prop:stationary_pair} by
combining the kernel bounds with the architecture constructions.
Fix $p,\Sigma_x,\Sigma_\theta,\sigma^2$ independently of $T$.
For attention, Lemma~\ref{lem:vector_uniform_body} reduces the class lower
bound to $\mathcal E_{L,T}$ up to $O(L/T)=O(1)$, with $1\leq L\leq T$.
Direct summation gives $TN_2-N_1^2\geq TL(L^2-1)/12$.
Using also $N_1\leq TL$ and $N_2\leq LN_1$ in
\eqref{eq:vector_equal_tap_cost} yields
\[
  \mathcal E_{L,T}\geq
  \sum_{i=1}^p\left[
    \frac{\lambda_i(L^2-1)}{12(L+v_i/\lambda_i)}
    +\frac{v_iT}{L+v_i/\lambda_i}
  \right]
  \gtrsim L+\frac TL.
\]
Since the Bayes subtraction is $O(\log T)$, any context achieving
$O(\sqrt T)$ regret must therefore satisfy $L=\Theta(\sqrt T)$.
On this range, \eqref{eq:vector_uniform_body_cost} expands
$\mathcal E_{L,T}$, so \eqref{eq:vector_attention_all_context} gives the
leading lower bound $\Lambda L/3+VT/L\geq\tfrac{2}{\sqrt3}\bar\zeta\sqrt T$,
minimized at the window \eqref{eq:vector_common_window}. Thus, for every
$d_e,d,H\geq1$,
\begin{equation}
  \inf_{L\geq1}\mathrm{Reg}_T\!\left(
    \mathcal F_{\mathrm{att}}(1,d,H,L)
  \right)
  \geq\frac{2}{\sqrt3}\bar\zeta\sqrt T+o(\sqrt T)
  \label{eq:vector_attention_class_lower}
\end{equation}
uniformly over the class parameters. Whenever $d_e\geq2p+1$ and $Hd\geq p$,
the distributed-head uniform-kernel construction gives equality:
\begin{equation}
  \mathrm{Reg}_T\!\left(
    \mathcal F_{\mathrm{att}}(1,d,H,L^\star(T))
  \right)
  =\frac{2}{\sqrt3}\bar\zeta\sqrt T+o(\sqrt T).
  \label{eq:vector_attention_class_rate}
\end{equation}
Outside this sufficient-width regime, only
\eqref{eq:vector_attention_class_lower} is asserted. The bound on
$\mathcal E_{L,T}$ also gives linear regret at fixed $L$, regardless of width.

For the SSM, Lemma~\ref{lem:vector_ssm_projection} and
Proposition~\ref{prop:vector_stationary_floor} give, for every $d_e,d,H\geq1$,
\begin{equation}
  \mathrm{Reg}_T\!\left(
    \mathcal F_{\mathrm{SSM}}(1,d,H,1)
  \right)
  \geq\zeta\sqrt T+o(\sqrt T).
  \label{eq:vector_ssm_class_rate}
\end{equation}
When $Hd\geq p$ and $d_e\geq2p+2$, the modewise exponential construction
attains this bound.

The aggregate-width claims for both classes follow from
Lemma~\ref{lem:vector_width_thresholds}. When $Hd<p$, all $d_i$ are strictly
positive, and \eqref{eq:vector_bayes_baseline_fixed_p} yields the corresponding $\Omega(T)$
lower bounds. This completes the proof of Proposition~\ref{prop:stationary_pair}.

\paragraph{Effect of task statistics.}
\label{app:case1_statistical_effects}
The leading regret coefficients distinguish how the two architectures
respond to the task's statistical structure. Whenever both architecture
witnesses are legal, their ratio is
\begin{equation}
  \frac{2}{\sqrt3}\frac{\bar\zeta}{\zeta}
  =
  \frac{2}{\sqrt3}
  \frac{\sqrt{(\sum_i\lambda_i)(\sum_i v_i)}}
       {\sum_i\sqrt{\lambda_iv_i}}
  \geq\frac{2}{\sqrt3}.
  \label{eq:vector_architecture_ratio}
\end{equation}
Cauchy--Schwarz gives the inequality, with equality exactly when all
$\lambda_i$ are equal, equivalently $\Sigma_\theta=q\Sigma_x^{-1}$
for some $q>0$.

Fix the total signal
variance $\Lambda=\sum_i\lambda_i$, and the noise variance $\sigma^2$.
Then $\bar\zeta=\sqrt{\Lambda((p+1)\Lambda+p\sigma^2)}$ is unchanged by
redistributing signal variance across eigendirections. In contrast,
$\zeta=\sum_i\sqrt{\lambda_i(\lambda_i+\Lambda+\sigma^2)}$ decreases
when variance is transferred from a smaller eigenvalue to a larger one:
the function $x\mapsto\sqrt{x(x+\Lambda+\sigma^2)}$ is strictly concave
for $x>0$. Thus, spectral concentration increases the SSM's advantage.

\paragraph{Optimal context and transition weights.}
For the full uniform and exponential lag profiles, the effective averaging width
$(\sum_{u\geq0}\kappa_u)^2/\sum_{u\geq0}\kappa_u^2$ is the number of
equally weighted evidence atoms giving the same innovation variance
after normalizing the kernel mass to one. It equals $L$ and
$(1+\alpha)/(1-\alpha)$, respectively, as used in the proof of
Lemma~\ref{lem:vector_stationary_prefix}.
Equations~\eqref{eq:vector_common_window} and
\eqref{eq:vector_modewise_exponential} give the leading-order optimal widths:
\[
  L^\star(T)=\sqrt{\frac{3VT}{\Lambda}}+O(1),
  \qquad
  \frac{1+\alpha_i^\star}{1-\alpha_i^\star}
  =2\sqrt{\frac{v_iT}{\lambda_i}}.
\]
More evidence-innovation variance relative to signal variance calls for
longer averaging. Attention uses one shared window, which is unchanged
by spectral concentration at fixed total signal variance and noise;
the SSM instead adjusts each decay to its eigendirection. In the
$p$-dimensional diagonal realization of
Lemma~\ref{lem:vector_ssm_projection},
$\|\Phi^\star\|_{\mathrm{op}}=\max_i\alpha_i^\star
=1-T^{-1/2}\min_i\sqrt{\lambda_i/v_i}+O(T^{-1})$.
Thus, as the horizon increases, the SSM realizes longer memory through its spectral radius approaching one at rate \(T^{-1/2}\), with a leading constant determined by the signal-to-innovation variance ratios.

\paragraph{Loss landscape at the optimal weights.}
The optimal weights also reveal how the local optimization landscape
changes with the horizon. We compare average-loss curvatures along
selected weight directions because larger curvature makes gradient
updates more sensitive to the step size.
In the diagonal SSM construction of
Lemma~\ref{lem:vector_ssm_projection}, fix an eigendirection $i$
and vary its transition weight $\alpha$ and read gain $\beta$, keeping
the other maps fixed. Its kernel is $\kappa_{u,i}=\beta\alpha^u$, and
Lemma~\ref{lem:vector_stationary_risk} gives its cumulative contribution
to belief-estimation error:
\[
  J^{\mathrm{SSM}}_{T,i}(\alpha,\beta)
  =\sum_{t=1}^T\left\{
    \lambda_i\left[\frac{\beta(1-\alpha^t)}{1-\alpha}-1\right]^2
    +\frac{v_i\beta^2(1-\alpha^{2t})}{1-\alpha^2}
  \right\}.
\]
At the leading-order optimal construction,
$\alpha=\alpha_i^\star$ and $\beta=1-\alpha_i^\star$.
Differentiating with the read gain held fixed gives
\[
  \left.\frac1T\frac{\partial^2J^{\mathrm{SSM}}_{T,i}}{\partial\alpha^2}
  \right|_{(\alpha,\beta)=(\alpha_i^\star,\,1-\alpha_i^\star)}
  =\frac{2\lambda_i}{(1-\alpha_i^\star)^2}+O(\sqrt T)
  =2v_iT+O(\sqrt T).
\]
This is also the transition-coordinate curvature of the average
prediction loss and regret, since the omitted terms are independent of
these weights. Thus, although the optimal transition weights remain bounded, the transition-weight curvature grows linearly with the horizon, with a leading coefficient proportional to the evidence-innovation variance \(v_i\).

For comparison, in the uniform-attention construction of
Lemma~\ref{lem:vector_attention_projection}, vary the query-projection gain $a$
in eigendirection $i$, keeping the other maps fixed.
Recalling $\tau_t=\min(t,L)$ and $V=\sum_j v_j$, the quadratic
in the proof of Lemma~\ref{lem:vector_uniform_body} gives
\[
  J^{\mathrm{att}}_{T,i}(a;L)
  =\sum_{t=1}^T\left[\lambda_i(a\tau_t-1)^2+v_i a^2\tau_t\right].
\]
At $L=L^\star(T)$, its average-loss curvature is
\[
  \left.\frac1T\frac{\partial^2J^{\mathrm{att}}_{T,i}}{\partial a^2}
  \right|_{L=L^\star(T)}
  =\frac2T\sum_{t=1}^T(\lambda_i\tau_t^2+v_i\tau_t)
  =\frac{6\lambda_iV}{\Lambda}T+O(\sqrt T).
\]
Both curvatures grow linearly with the horizon. In the specified
query-gain and transition-weight coordinates, the leading coefficients
sum to $6V$ for attention and $2V$ for the SSM. Thus, local step-size
sensitivity increases with the horizon. Attention's summed leading
curvature is three times the SSM's in these coordinates, but this does
not establish a convergence-rate ordering between the families.

\subsection{Extension to time- and data-adaptive kernels}
\label{app:case1_adaptive}

This section shows how count normalization and additional depth let both
families adapt their kernels to elapsed time and realized information,
escaping the stationary regret floor of
Theorem~\ref{thm:stationary_floor}. We first explain the two
mechanisms in the scalar case (\S\ref{app:case1_adaptive_scalar}), then
extend the predictors and their regret rates to general $p$ and state
the architecture-class optima (\S\ref{app:case1_adaptive_vector}).
Sections~\ref{app:case1_timeadaptive_proof} and~\ref{app:case1_dataadaptive_proof}
prove the matching lower bounds and give the time- and data-adaptive
constructions, respectively.

\subsubsection{Scalar adaptive kernels and regret rates}
\label{app:case1_adaptive_scalar}

To make the two adaptive mechanisms clear, we first consider the scalar
case $p=1$, with $\Sigma_x=1$ and
$\Sigma_\theta=\nu>0$.  Then
$A_t=\sum_{s\leq t}x_sy_s$, $B_t=\sum_{s\leq t}x_s^2$, and
$\hat\theta_t^{\mathrm{Bayes}}=\nu A_t/(\sigma^2+\nu B_t)$.

As in Lemma~\ref{lem:vector_stationary_risk}, kernel mismatch has a
bias--variance decomposition. Here we condition on the observed
regressors to allow data-dependent kernel weights.
For a predictor $g$, write its projected belief as
$\Pi_t\hat\theta_t(g)=\sum_{s\leq t}k_{t,s}(g)x_sy_s$.  As in the
definition of $\mathcal K_t$, these kernel weights depend only on the
regressors $x_{1:t}$.  Define
$\mu_t(g):=\sum_{s\leq t}k_{t,s}(g)x_s^2$ and
$V_t(g):=\sum_{s\leq t}k_{t,s}(g)^2x_s^2$.
Here $\mu_t(g)$ is the regressor-weighted kernel mass, whose deviation
from one determines the bias, and $\sigma^2V_t(g)$ is the conditional
observation-noise variance.

\begin{lemma}[Scalar regret decomposition]\label{lem:regret_decomp}
For any class $\mathcal F$,
$\mathrm{Reg}_T(\mathcal F)=\inf_{g\in\mathcal F}
\{\mathsf K_T(g)+\mathsf R_T(g)\}$, where
\[
  \mathsf R_T(g)
  :=
  \sum_{t\leq T}
  \bigl\|g_t-x_{t+1}\Pi_t\hat\theta_t(g)\bigr\|_t^2
\]
charges departures from kernel-form belief prediction, and
\[
  \begin{aligned}
  \mathsf K_T(g)
  &:={}
  \sum_{t\leq T}
  \EE\!\left[
    \bigl(\Pi_t\hat\theta_t(g)-\hat\theta_t^{\mathrm{Bayes}}\bigr)^2
  \right] \\
  &={}
  \nu\sum_{t\leq T}\EE[(\mu_t(g)-1)^2]
  +\sigma^2\sum_{t\leq T}\EE[V_t(g)]
  -\sigma^2\log T+O(1).
  \end{aligned}
\]
\end{lemma}

\begin{proof}
The Pythagorean identity~\eqref{eq:belief_pythagoras} separates kernel
mismatch from the two orthogonal residual terms, giving the first identity.
To evaluate the kernel mismatch, condition on $x_{1:t}$ and substitute
$y_s=x_s\theta+\epsilon_s$.  The projected belief is
$\mu_t(g)\theta+\sum_{s\leq t}k_{t,s}(g)x_s\epsilon_s$, so independence
of $\theta$ and the observation noise gives
\[
  \EE\!\left[
    (\Pi_t\hat\theta_t(g)-\theta)^2\mid x_{1:t}
  \right]
  =\nu(\mu_t(g)-1)^2+\sigma^2V_t(g).
\]
The identity extends from bounded kernel weights to the closure
defining $\mathcal K_t$ by $L^2$ approximation.  Since the Bayes belief is
the conditional mean of $\theta$, posterior orthogonality then yields
\[
  \EE[(\Pi_t\hat\theta_t(g)-\hat\theta_t^{\mathrm{Bayes}})^2]
  =\nu\EE[(\mu_t(g)-1)^2]+\sigma^2\EE[V_t(g)]
   -\EE\!\left[\frac{\nu\sigma^2}{\sigma^2+\nu B_t}\right].
\]
Summing and using the scalar specialization of
\eqref{eq:vector_bayes_baseline_fixed_p} proves the claim.
\end{proof}

The final term subtracts the cumulative uncertainty already incurred by the
Bayes predictor. To motivate adaptive kernels, write $U_t:=B_t/t$.
The Bayes kernel satisfies
\[
  k_{t,s}^{\mathrm{Bayes}}
  =\frac{1}{t\bigl(U_t+\sigma^2/(\nu t)\bigr)}
  \sim\frac{1}{tU_t}\qquad(t\to\infty).
\]
The factor $1/t$ adapts to elapsed time, while $1/U_t$ adjusts for realized
information. Since $U_t$ concentrates at $1$, this suggests kernels
\[
  k_{t,s}=\frac{\hat\mu(U_t)}{t},
\]
where $\hat\mu(v)$ approximates $1/v$ near $v=1$.
The next theorem quantifies how the accuracy of this local approximation
determines cumulative kernel mismatch.

\begin{theorem}[Adaptive rate]\label{thm:matching_order_cost}
Suppose $g$ realizes $k_{t,s}=\hat\mu(U_t)/t$ for a fixed measurable function
$\hat\mu$ of polynomial growth on $[0,\infty)$, meaning
$|\hat\mu(v)|\leq C(1+v^m)$ for some $C>0$ and $m\geq0$.  Assume
$1-v\hat\mu(v)=c(v-1)^r+O(|v-1|^{r+\eta})$ as $v\to1$, for some
integer $r\geq1$, $c\neq0$, and $\eta>0$.  Then
\[
  \mathsf K_T(g)=
  \begin{cases}
    2\nu c^2\log T+O(1), & r=1,\\[2pt]
    \Theta(1), & r\geq2.
  \end{cases}
\]
\end{theorem}

\begin{proof}
The variance has the same leading term as Bayes, so the remaining bias
determines the rate.  For the displayed kernel,
\[
  \mu_t(g)=U_t\hat\mu(U_t),
  \qquad
  V_t(g)=\frac{U_t\hat\mu(U_t)^2}{t}.
\]
Since $B_t\sim\chi_t^2$,
$\EE[(U_t-1)^2]=2/t$ and
$\EE|U_t-1|^q=O(t^{-q/2})$ for every fixed $q>0$.
For even moments, the latter bound follows by expanding powers of
$\sum_{s\leq t}(x_s^2-1)$: every nonzero term repeats each index at least
twice.  H\"older's inequality gives the remaining moments.
The exponential tails of $U_t$, together with polynomial growth and
Cauchy--Schwarz, make contributions outside any fixed neighborhood of
$1$ exponentially small.  We may therefore average the local matching
expansion to obtain
\[
  \EE[(\mu_t(g)-1)^2]
  =\begin{cases}
    2c^2/t+O(t^{-1-\eta/2}), & r=1,\\[2pt]
    O(t^{-r}), & r\geq2.
  \end{cases}
\]
Likewise, $v\hat\mu(v)^2=1+O(|v-1|)$ near $1$, hence
$\EE[U_t\hat\mu(U_t)^2]=1+O(t^{-1/2})$.  Thus the variance term
cancels the Bayes baseline up to a bounded remainder:
\[
  \sigma^2\sum_{t\leq T}\EE[V_t(g)]-\sigma^2\log T
  =\sigma^2\sum_{t\leq T}\frac{1+O(t^{-1/2})}{t}
   -\sigma^2\log T
  =O(1).
\]
Lemma~\ref{lem:regret_decomp} now gives the logarithmic formula for $r=1$
and the bounded upper bound for $r\geq2$.

For the positive lower bound, it suffices to consider $t=1$.
The first-step Bayes kernel is $1/(v+\sigma^2/\nu)$, whereas the matching
assumption implies $\hat\mu(v)\to1$ as $v\to1$.
These kernels differ on an interval around $1$, which $U_1=x_1^2$ visits
with positive probability.  The first-step kernel mismatch is therefore
strictly positive; all later terms are nonnegative.  This proves
$\Theta(1)$ for $r\geq2$.
\end{proof}

The constant approximation $\hat\mu(v)=1$ and the first-order Taylor
approximation $\hat\mu(v)=2-v$ give the predictors
\[
  g_t^{\mathrm{time}}:=x_{t+1}\frac{A_t}{t},
  \qquad
  g_t^{\mathrm{data}}:=x_{t+1}(2-U_t)\frac{A_t}{t}.
\]
Both apply kernel-form beliefs to $x_{t+1}$, so their
residual costs vanish and these kernel costs are their full regrets.
For $g^{\mathrm{time}}$, $1-v\hat\mu(v)=-(v-1)$ gives
$\mathrm{Reg}_T(g^{\mathrm{time}})=2\nu\log T+O(1)$.
The data-adaptive correction cancels the first-order fluctuation, since
\[
  (2-U_t)U_t=1-(U_t-1)^2.
\]
Thus $g^{\mathrm{data}}$ has matching order two and
$\mathrm{Reg}_T(g^{\mathrm{data}})=\Theta(1)$.

\subsubsection{Adaptive kernels and class optima for general \texorpdfstring{$p$}{p}}
\label{app:case1_adaptive_vector}

The same correction extends to any fixed $p$. In the whitened
coordinates of Appendix~\ref{app:case1_vector_stationary}, define the normalized
evidence vector $a_t$ and information matrix $U_t$ by
\[
  a_t:=\frac{\Sigma_x^{-1/2}A_t}{t}=\frac1t\sum_{s\leq t}\widetilde x_sy_s,
  \qquad
  U_t:=\frac{\widetilde B_t}{t}
      =\frac1t\sum_{s\leq t}\widetilde x_s\widetilde x_s^\top.
\]
Count normalization again uses $a_t$ as the belief. The scalar cancellation
extends directly to
\[
  (2I_p-U_t)U_t=I_p-(U_t-I_p)^2,
\]
so the same first-order cancellation applies to the coefficient on
$\widetilde\theta$. The corresponding predictors are
\[
  g_t^{\mathrm{time}}:=\widetilde x_{t+1}^\top a_t,
  \qquad
  g_t^{\mathrm{data}}:=\widetilde x_{t+1}^\top(2I_p-U_t)a_t.
\]

Lemma~\ref{lem:case1_vector_adaptive_risks} gives the corresponding regret
bounds for general $p$.

\begin{lemma}[Regret of adaptive predictors]
\label{lem:case1_vector_adaptive_risks}
For every $t\geq1$,
\[
  \EE\|a_t-\widetilde\theta\|_2^2=\frac{\operatorname{tr}\Gamma}{t},
  \qquad
  \EE\|(2I_p-U_t)a_t-\widetilde\theta\|_2^2
    =\frac{p\sigma^2}{t}+O(t^{-2}).
\]
The prediction losses are these quantities plus $\sigma^2$.
Both predictors have zero residual cost in
\eqref{eq:belief_pythagoras}, and
\[
  \mathrm{Reg}_T(g^{\mathrm{time}})
    =(p+1)\operatorname{tr}(\widetilde\Sigma_\theta)\log T+O(1),
  \qquad
  \mathrm{Reg}_T(g^{\mathrm{data}})=O(1).
\]
\end{lemma}

\begin{proof}
Count normalization leaves the empirical information fluctuation intact;
the correction cancels it to first order. For the first error, the innovation
covariance calculation in Lemma~\ref{lem:vector_stationary_risk} gives
\[
  \EE[(\widetilde x_ry_r-\widetilde\theta)(\widetilde x_sy_s-\widetilde\theta)^\top]
    =\mathbf1\{r=s\}\Gamma.
\]
Averaging these innovations gives $\operatorname{tr}(\Gamma)/t$.

For the correction, put $\Delta_t:=U_t-I_p$. The exact identity
$(2I_p-U_t)U_t=I_p-\Delta_t^2$ gives
\[
  (2I_p-U_t)a_t-\widetilde\theta
  =-\Delta_t^2\widetilde\theta
   +(I_p-\Delta_t)\frac1t\sum_{s\leq t}\widetilde x_s\epsilon_s.
\]
Conditioned on $\widetilde x_{1:t}$, the two terms are independent and centered, and
$t^{-1}\sum_{s\leq t}\widetilde x_s\epsilon_s$ has covariance $\sigma^2U_t/t$.
Taking the squared norm and averaging therefore gives
\[
  \EE\|(2I_p-U_t)a_t-\widetilde\theta\|_2^2
  =\EE\operatorname{tr}(\widetilde\Sigma_\theta\Delta_t^4)
   +\frac{\sigma^2}{t}\EE\operatorname{tr}
      (I_p-\Delta_t-\Delta_t^2+\Delta_t^3).
\]
Each entry of $\Delta_t$ is an average of independent centered variables
$\widetilde x_{s,i}\widetilde x_{s,j}-\mathbf1\{i=j\}$. In its fourth-moment expansion, each
nonzero term repeats every sample index at least twice. There are
$O(t^2)$ such terms with bounded moments, so
$\EE|\Delta_{t,ij}|^4=O(t^{-2})$. Since $p$ is fixed, summing entries
and applying H\"older's inequality gives
\[
  \EE\|\Delta_t\|_F^k=O(t^{-k/2}),\qquad k=2,3,4.
\]
Together with $\EE\Delta_t=0$, these bounds make the bias term
$O(t^{-2})$ and the noise term $p\sigma^2/t+O(t^{-2})$.
They hold from $t=1$, including $t<p$, without an empirical inverse.

Both predictions are linear in the fresh regressor. Their
original-coordinate kernels are $\Sigma_x^{-1}/t$ and
$\Sigma_x^{-1/2}(2I_p-U_t)\Sigma_x^{-1/2}/t$, respectively. Their weights depend only
on the regressors; truncating them and using finite Gaussian moments
shows that both beliefs belong to the closure defining $\mathcal K_t$.
Both residual terms in \eqref{eq:belief_pythagoras} therefore vanish.
Finally, \eqref{eq:vector_fresh_query_mse} subtracts the Bayes baseline
\eqref{eq:vector_bayes_baseline_fixed_p}. Since
$\operatorname{tr}\Gamma-p\sigma^2=(p+1)\operatorname{tr}\widetilde\Sigma_\theta$ and the
$O(t^{-2})$ remainders are summable, this proves both regret bounds.
\end{proof}

The following propositions establish the optima of the count-normalized
classes (Appendix~\ref{app:count_normalized_extensions}).
Sections~\ref{app:case1_timeadaptive_proof} and~\ref{app:case1_dataadaptive_proof}
prove the matching class lower bounds and construct both architectures
to realize $g^{\mathrm{time}}$ and $g^{\mathrm{data}}$, respectively.

\begin{proposition}[Time-adaptive pair]\label{prop:timeadaptive_pair}
For $d,H\geq1$ with $Hd\geq p$ and $d_e\geq2p+2$, each of
\[
  \mathcal F_{\mathrm{att}}^{\mathrm{cal}}(1,d,H,T;t),
  \qquad
  \mathcal F_{\mathrm{SSM}}^{\mathrm{cal}}(1,d,H,1;t)
\]
contains $g^{\mathrm{time}}$. For either class $\mathcal F$,
\[
  \mathrm{Reg}_T(\mathcal F)
    =(p+1)\operatorname{tr}(\widetilde\Sigma_\theta)\log T+O(1).
\]
\end{proposition}

One additional layer makes the kernel data-adaptive through the correction
$2I_p-U_t$. Attention implements this correction on the fresh regressor,
whereas the SSM implements it on the maintained belief.

\begin{proposition}[Data-adaptive pair]\label{prop:dataadaptive_pair}
The predictor $g^{\mathrm{data}}$ is realized under the following
sufficient conditions:
\begin{enumerate}[label=(\alph*),leftmargin=*]
  \item Attention
  $\mathcal F_{\mathrm{att}}^{\mathrm{cal}}(2,d,H,T;(t,t))$
  with $d\geq p$, $H\geq1$, and $d_e\geq3p+2$.
  \item The fixed-transition, ordinary-read class
  $\mathcal F_{\mathrm{SSM}}^{\mathrm{cal}}(2,d,H,1;(t,1))$
  with $d,H\geq1$, $Hd\geq p^2+p$, and $d_e\geq p^2+3p+2$.
\end{enumerate}
Under these conditions, either class satisfies
$\mathrm{Reg}_T(\mathcal F)=\Theta(1)$.
\end{proposition}

Thus time adaptation through count normalization gives both families
logarithmic regret, while data adaptation to realized information
reduces it to constant regret.

\subsubsection{Proof of Proposition~\ref{prop:timeadaptive_pair}: Time-adaptive pair}
\label{app:case1_timeadaptive_proof}

Count normalization changes the coefficients of a one-layer predictor, but
not its polynomial degree. Lemma~\ref{lem:case1_degree_two_beliefs}
establishes this restriction, and Lemma~\ref{lem:case1_degree_two_floor}
turns it into a logarithmic regret floor. We then construct both
architectures to attain it. We use the observed polynomial
space $\mathcal Q_{2,t}$ and deterministic kernel-form space $\mathcal S_t$ from
\S\ref{app:case1_score_projection}.

\begin{lemma}[Degree-two implicit beliefs]
\label{lem:case1_degree_two_beliefs}
For every $d_e,d,H\geq1$, the whitened belief
$\widetilde\theta_t(g)$ belongs to $\mathcal Q_{2,t}$ whenever
\[
  g\in\mathcal F_{\mathrm{att}}^{\mathrm{cal}}
    (1,d,H,T;t)
  \quad\text{or}\quad
  g\in\mathcal F_{\mathrm{SSM}}^{\mathrm{cal}}
    (1,d,H,1;t).
\]
The conclusion is uniform in the architecture parameters, even when they
depend on $T$.
\end{lemma}

\begin{proof}
Each attention summand is a bilinear score times an affine value, hence
has degree at most three in the tokens. For a fixed-transition SSM, every
state coordinate is a linear combination of bilinear writes, and the
ordinary read likewise gives degree at most three. Affine biases,
positional biases, residual maps, and deterministic count normalization do not
increase this bound.

In $\widetilde\theta_t(g)=\EE[\widetilde x_{t+1}g_t\mid\mathcal G_t^-]$, a surviving
monomial must contain a fresh-regressor factor before multiplication by
$\widetilde x_{t+1}$. Gaussian integration removes these factors and leaves observed
degree at most two. Overlap between neighboring token blocks does not
increase the degree, and widths and parameter magnitudes affect only
the coefficients.
\end{proof}

\begin{lemma}[Degree-two regret floor]
\label{lem:case1_degree_two_floor}
If $\widetilde\theta_t(g)\in\mathcal Q_{2,t}$ at every step, then
\[
  \mathrm{Reg}_T(g)
  \geq\sum_{t\leq T}\sum_{i=1}^p
      \frac{\lambda_i v_i}{t\lambda_i+v_i}
      -\mathcal E_T^{\mathrm{Bayes}}
  =(p+1)\operatorname{tr}(\widetilde\Sigma_\theta)\log T+O(1).
\]
The bound is uniform in $g$, including horizon-dependent parameters.
\end{lemma}

\begin{proof}
Lemma~\ref{lem:vector_degree_two_projection} reduces the best degree-two
approximation of $\widetilde\theta$ to its projection onto $\mathcal S_t$. The evidence
moments from Lemma~\ref{lem:vector_stationary_risk} are
\[
  \EE[\widetilde\theta(\widetilde x_sy_s)^\top]=\widetilde\Sigma_\theta,
  \qquad
  \EE[(\widetilde x_ry_r)(\widetilde x_sy_s)^\top]
    =\widetilde\Sigma_\theta+\mathbf1\{r=s\}\Gamma.
\]
Thus the projection is
\[
  \widetilde\theta_t^{(2)}=\widetilde\Sigma_\theta(t\widetilde\Sigma_\theta+\Gamma)^{-1}\sum_{s\leq t}\widetilde x_sy_s,
\]
since its error is orthogonal to every evidence atom. Using the eigenvalues
$\lambda_i,v_i$ from \eqref{eq:vector_covariance_scales}, its mean-squared error is
\[
  \EE\|\widetilde\theta-\widetilde\theta_t^{(2)}\|_2^2
  =\sum_{i=1}^p\frac{\lambda_i v_i}{t\lambda_i+v_i}
  =\frac{\operatorname{tr}\Gamma}{t}+O(t^{-2}).
\]
Equation~\eqref{eq:vector_score_pythagoras}
charges other polynomial components nonnegatively, as does
\eqref{eq:vector_fresh_query_mse} for the fresh-regressor residual.
Summing and subtracting the Bayes baseline
\eqref{eq:vector_bayes_baseline_fixed_p} proves the claim, since
$\operatorname{tr}\Gamma-p\sigma^2=(p+1)\operatorname{tr}\widetilde\Sigma_\theta$.
\end{proof}

Combining Lemmas~\ref{lem:case1_degree_two_beliefs}
and~\ref{lem:case1_degree_two_floor} gives the lower bounds in
Proposition~\ref{prop:timeadaptive_pair}. To attain them, preserve
$(\widetilde x_{t+1},\widetilde x_t,y_t)$ in $2p+1$ residual coordinates and reserve one clean
prediction coordinate. For attention, partition $[p]$ among at most $H$
sets $I_b$ of size at most $d$. Head $b$ uses query $\widetilde x_{t+1,I_b}$, key
$\widetilde x_{s,I_b}$, and value $y_s$ in one value coordinate, with unused slots
and positional biases zeroed. Output maps sum the count-normalized heads into
the prediction coordinate, giving
\[
  g_t=\frac1t\sum_{s\leq t}\sum_b
          \widetilde x_{t+1,I_b}^\top \widetilde x_{s,I_b}y_s
      =\widetilde x_{t+1}^\top a_t=g_t^{\mathrm{time}}.
\]
For the SSM, distribute $p$ state coordinates across the heads. Coordinate
$i$ has transition $1$, write $\widetilde x_{t,i}y_t$, and read selector
$\widetilde x_{t+1,i}$. Its state is the unnormalized sum
$\sum_{s\leq t}\widetilde x_{s,i}y_s$; normalization divides only its read by $t$.
Summing these reads into the clean prediction coordinate gives the same
output. Both constructions fit whenever $Hd\geq p$ and $d_e\geq2p+2$,
and Lemma~\ref{lem:case1_vector_adaptive_risks} gives their regret with
zero residual cost. The SSM uses $p$ active states out of $Hd$.

\subsubsection{Proof of Proposition~\ref{prop:dataadaptive_pair}: Data-adaptive pair}
\label{app:case1_dataadaptive_proof}

A positive first-step excess loss supplies the lower bound for constant
regret. Lemma~\ref{lem:case1_first_step_floor} establishes this floor;
we then construct both two-layer predictors to attain constant regret.
At the first prediction time, the token is $z_1=(x_2,y_1,x_1)$ and
\[
  g_1^{\mathrm{Bayes}}(z_1)
  =\frac{y_1x_2^\top\Sigma_\theta x_1}
          {\sigma^2+x_1^\top\Sigma_\theta x_1}.
\]
\begin{lemma}[Positive first-step regret]
\label{lem:case1_first_step_floor}
Every predictor $g$ in either class
\[
  \mathcal F_{\mathrm{att}}^{\mathrm{cal}}(2,d,H,T;(t,t)),
  \qquad
  \mathcal F_{\mathrm{SSM}}^{\mathrm{cal}}(2,d,H,d_{\mathrm{conv}};(t,1))
\]
satisfies
\[
  \mathrm{Reg}_T(g)\geq c>0,
\]
where $c$ depends only on the fixed distributional parameters. The bound is
uniform over $T,d_e,d,H,d_{\mathrm{conv}}\geq1$.
\end{lemma}

\begin{proof}
At $t=1$, both normalization schedules equal $(1,1)$. An attention layer's
bilinear score times its affine value increases polynomial degree by at most
a factor of three. For the SSM, zero padding and zero initial state leave only
a bilinear current-token write followed by a token--state read, giving the
same degree bound regardless of convolution width. Affine maps and
residuals do not increase this bound. Thus every two-layer first prediction
belongs to the space $\mathcal P_9$ of polynomials of total degree at most
nine in $z_1$, independently of widths or parameter values.

The law of $z_1$ has an everywhere-positive density and finite polynomial
moments. If a polynomial $P$ equaled the Bayes predictor almost surely,
multiplying by its denominator would give a polynomial identity. Fix
$v\ne0$, set $x_1=rv$, $x_2=v$, $y_1=1$, and write
$\alpha=v^\top\Sigma_\theta v>0$. The identity would imply
\[
  (\sigma^2+\alpha r^2)P(v,1,rv)=\alpha r,
\]
which is impossible for a univariate polynomial.
Since $\mathcal P_9$ is finite-dimensional and therefore closed in
$L^2(P_{z_1})$, its squared distance from the Bayes predictor is positive:
\[
  c:=\inf_{P\in\mathcal P_9}
  \EE[(P(z_1)-g_1^{\mathrm{Bayes}}(z_1))^2]>0.
\]
Bayes orthogonality and nonnegative later excess losses give
$\mathrm{Reg}_T(g)\geq\EE[(g_1-g_1^{\mathrm{Bayes}})^2]\geq c$,
including for horizon-dependent widths and parameters.
\end{proof}

It remains to realize $g^{\mathrm{data}}$ in each architecture class.
For attention, preserve the $2p+1$ raw whitened token
coordinates and reserve a zero-initialized $p$-vector block and scalar
output coordinate. One active head per layer suffices when $d\geq p$.
The first layer uses query $\widetilde x_{t+1}$, key $\widetilde x_s$, and value $\widetilde x_s$, writing
only into the vector block:
\[
  \frac1t\sum_{s\leq t}(\widetilde x_{t+1}^\top \widetilde x_s)\widetilde x_s=U_t\widetilde x_{t+1}.
\]
The second layer uses the affine query $2\widetilde x_{t+1}-U_t\widetilde x_{t+1}$, selected
from the raw and new blocks, and retains key $\widetilde x_s$ and value $y_s$ from
the untouched raw source coordinates. Its output in the clean scalar
coordinate is
\[
  \frac1t\sum_{s\leq t}(2\widetilde x_{t+1}-U_t\widetilde x_{t+1})^\top \widetilde x_sy_s
  =\widetilde x_{t+1}^\top(2I_p-U_t)a_t=g_t^{\mathrm{data}}.
\]
Here $U_t$ is symmetric. Both layers use
full context, normalization by $t$, and zero positional biases. Padding and
zeroing unused heads extend the construction to every $d\geq p$ and
$d_e\geq3p+2$.

The SSM construction uses two fixed-transition layers with ordinary reads
and identity convolutions. Alongside the $2p+1$ raw token coordinates,
reserve residual blocks for $a_t$, all entries of $U_t$, and the prediction.
Its first layer uses $p$ recurrent coordinates to accumulate evidence and
$p^2$ to accumulate information, with identity transitions and writes
$\widetilde x_{t,j}y_t$ and $\widetilde x_{t,i}\widetilde x_{t,j}$, respectively.
They hold the unnormalized statistics
\[
  (\Sigma_x^{-1/2}A_t)_j,\qquad
  (\Sigma_x^{-1/2}B_t\Sigma_x^{-1/2})_{ij},\qquad i,j\in[p].
\]
Constant read selectors $1$ with normalization by $t$ expose $a_{t,j}$ and
$(U_t)_{ij}$ in their reserved residual blocks; normalization does not alter
the stored recurrences.

The second layer has one state coordinate per ordered pair $(i,j)$, each
with transition zero and normalization factor $1$. Its affine write factors are
$a_{t,j}$ and $2\mathbf1\{i=j\}-(U_t)_{ij}$; the ordinary read selector
is $\widetilde x_{t+1,i}$. Summing these coordinatewise reads into the clean output
gives
\[
  \sum_{i,j}\widetilde x_{t+1,i}a_{t,j}
     \bigl(2\mathbf1\{i=j\}-(U_t)_{ij}\bigr)
  =\widetilde x_{t+1}^\top(2I_p-U_t)a_t=g_t^{\mathrm{data}}.
\]
All constants are legal affine biases, and the upper layer accesses the
statistics only through residual coordinates. Every state has its own
affine write and read rows, so packing at most $d$ coordinates per head
is valid whenever $Hd\geq p^2+p$. The residual width required is
$d_e\geq(2p+1)+p+p^2+1=p^2+3p+2$. The construction uses $p^2+p$ active
states in the first layer and $p^2$ in the second, totaling $2p^2+p$
out of $2Hd$ across both layers.

Both constructions realize $g^{\mathrm{data}}$ at every prediction time,
including startup. Lemma~\ref{lem:case1_vector_adaptive_risks} gives
$O(1)$ regret with zero residual cost, and
Lemma~\ref{lem:case1_first_step_floor} supplies the positive class floor.
This proves Proposition~\ref{prop:dataadaptive_pair}.
 
\section{Case 2: full statements and proofs}\label{app:case2}

For an overview of the notation and proof dependencies, see the
\hyperref[app:notation_guide]{notation guide} and
\hyperref[app:proof_guide]{proof guide} at the start of the appendix.

We first prove the two-factor floor
(\S\ref{app:vector_routing_geometry}). Section~\ref{app:case2a_vector}
develops Case~2.a: assembly floors, constructions, class lower bounds, and
the memory comparison. Section~\ref{app:case2b_capacity} covers Case~2.b:
addressing-capacity bounds and both constructions.

\subsection{Proof of Theorem~\ref{thm:routing_floor}}
\label{app:vector_routing_geometry}

We use the whitened coordinates of \S\ref{app:case1_prediction_geometry}.
We extend the same belief projection to routing and use it to prove the
two-factor floor in Theorem~\ref{thm:routing_floor}. Recall the whitened coordinates
\[
  \widetilde x_t:=\Sigma_x^{-1/2}x_t,
  \qquad \widetilde\theta:=\Sigma_x^{1/2}\theta,
  \qquad \widetilde\Sigma_\theta:=\Sigma_x^{1/2}\Sigma_\theta\Sigma_x^{1/2},
\]
so that $\widetilde x_t\sim\mathcal N(0,I_p)$ and
$y_s=\widetilde x_{\rho(s)}^\top\widetilde\theta+\epsilon_s$ at a matched step.  Throughout this
appendix, the route is either deterministic, as in positional routing, or is
measurable with respect to an auxiliary key process independent of
$(\widetilde\theta,(\widetilde x_t)_t,(\epsilon_t)_t)$, as in content routing.  Together with partial
injectivity, this makes the routed regressors at distinct matched steps
distinct iid standard Gaussian vectors.

Recall the Case~1 notation
\[
  \widetilde\Sigma_\theta=U\operatorname{diag}(\lambda_1,\ldots,\lambda_p)U^\top,
  \qquad v_i:=\lambda_i+\operatorname{tr}\widetilde\Sigma_\theta+\sigma^2,
\]
and
\[
  \Lambda:=\operatorname{tr}\widetilde\Sigma_\theta,
  \qquad V:=\sum_{i=1}^p v_i,
  \qquad
  d_i:=\frac{\lambda_i v_i}{\lambda_i+v_i}.
\]
For $n$ distinct matched observations, let
\begin{equation}
  P_n:=\left(
    \widetilde\Sigma_\theta^{-1}+\sigma^{-2}\sum_{j=1}^n \widetilde x_j\widetilde x_j^\top
  \right)^{-1},
  \qquad
  \mathcal E_N^{\mathrm{Bayes}}
  =\sum_{n=1}^N\EE\operatorname{tr}P_n .
  \label{eq:routed_vector_posterior_baseline}
\end{equation}
At fixed $p$ and fixed $\widetilde\Sigma_\theta\succ0$,
\begin{equation}
  \mathcal E_N^{\mathrm{Bayes}}
  =p\sigma^2\log N+O(1).
  \label{eq:routed_vector_bayes_baseline}
\end{equation}
Here $\EE\operatorname{tr}P_n=e_n^{\mathrm{Bayes}}$ from
\eqref{eq:vector_bayes_baseline}, so~\eqref{eq:routed_vector_bayes_baseline}
is the Case~1 Bayes error estimate after $N$ matched observations.

For each time $t$, write
\[
  \mathcal M_t:=\{s\leq t:s\text{ is matched}\},
  \qquad N_t^\rho:=|\mathcal M_t|.
\]

Fix a prediction time $t$ for which $t+1$ is matched. Let $\mathcal G_t^-$ contain the route, all
observed labels, and every regressor in the history except all token
occurrences of $\widetilde x_{\rho(t+1)}$. Although this regressor can appear in
earlier tokens, exogenous and injective routing ensures that no observed
label uses it. Thus
\[
  \widetilde x_{\rho(t+1)}\sim\mathcal N(0,I_p)
  \quad\text{and}\quad
  \widetilde x_{\rho(t+1)}\ \text{is independent of }\mathcal G_t^-.
\]
For an arbitrary square-integrable prediction $g_t$, the Case~1 projection
\eqref{eq:vector_implicit_belief} therefore becomes
\begin{equation}
  \widetilde\theta_t(g):=\EE[\widetilde x_{\rho(t+1)}g_t\mid\mathcal G_t^-],
  \qquad
  g_t^\perp:=g_t-\widetilde x_{\rho(t+1)}^\top\widetilde\theta_t(g).
  \label{eq:routed_vector_implicit_belief}
\end{equation}

The Case~1 projection argument in \S\ref{app:case1_prediction_geometry}
shows that the map $a\mapsto\widetilde x_{\rho(t+1)}^\top a$ is an isometry from
$L^2(\mathcal G_t^-;\RR^p)$ into scalar $L^2$.  Equation
\eqref{eq:routed_vector_implicit_belief} is the orthogonal projection of
$g_t$ onto its image. The Bayes belief
$\widetilde\theta_t^\mathrm{Bayes}:=\EE[\widetilde\theta\mid z_{1:t}]$
is $\mathcal G_t^-$-measurable because no observed label uses the held-out
regressor. Thus
\begin{align}
  g_t^\mathrm{Bayes}&=\widetilde x_{\rho(t+1)}^\top\widetilde\theta_t^\mathrm{Bayes},
  \label{eq:routed_vector_bayes_prediction}\\
  \|g_t-g_t^\mathrm{Bayes}\|_t^2
  &=\|\widetilde\theta_t(g)-\widetilde\theta_t^\mathrm{Bayes}\|_t^2
    +\|g_t^\perp\|_t^2,
  \label{eq:routed_vector_prediction_pythagoras}
\end{align}
where, for vector-valued beliefs,
$\langle a,b\rangle_t:=\EE[a^\top b]$ and
$\|a\|_t^2:=\EE\|a\|_2^2$.

For vectors in either Hilbert space, write
\[
  \operatorname{corr}^2(A,B)
  :=\frac{\langle A,B\rangle^2}{\|A\|^2\|B\|^2}.
\]
Throughout the Case~2 proofs, we use the convention $0/0=0$.

\begin{proof}[Proof of Theorem~\ref{thm:routing_floor}]
At a matched target, distance to the particular prediction $g_t$ dominates
distance to the line it spans:
\[
  \|g_t-g_t^\mathrm{Bayes}\|_t^2
  \geq
  \inf_{c\in\RR}\|cg_t-g_t^\mathrm{Bayes}\|_t^2
  =\|g_t^\mathrm{Bayes}\|_t^2
   \bigl(1-\operatorname{corr}^2(g_t,g_t^\mathrm{Bayes})\bigr).
\]
By orthogonality in~\eqref{eq:routed_vector_implicit_belief} and the
Bayes factorization~\eqref{eq:routed_vector_bayes_prediction},
\begin{align*}
  \langle g_t,g_t^\mathrm{Bayes}\rangle_t
  &=\langle\widetilde\theta_t(g),\widetilde\theta_t^\mathrm{Bayes}\rangle_t,\qquad
  \|g_t^\mathrm{Bayes}\|_t^2=\|\widetilde\theta_t^\mathrm{Bayes}\|_t^2,\\
  \|g_t\|_t^2
  &=\|\widetilde\theta_t(g)\|_t^2+\|g_t^\perp\|_t^2.
\end{align*}
Consequently,
\[
  \operatorname{corr}^2(g_t,g_t^\mathrm{Bayes})
  =\operatorname{corr}^2(\widetilde\theta_t(g),\widetilde\theta_t^\mathrm{Bayes})
   \frac{\|\widetilde\theta_t(g)\|_t^2}
        {\|\widetilde\theta_t(g)\|_t^2+\|g_t^\perp\|_t^2}.
\]
Unmatched targets contribute a
nonnegative excess loss, so summing over matched targets proves the theorem.
\end{proof}

\subsection{Case~2.a: positional routing}
\label{app:case2a_vector}

We first prove the assembly floors of Corollary~\ref{cor:case2a_floors}
(\S\ref{app:case2a_assembly_floors}), then give the constructions and upper
bounds for Proposition~\ref{prop:case2a_pair}
(\S\ref{app:case2a_constructions}). The attention and SSM lower bounds for
Proposition~\ref{prop:case2a_pair} follow in
\S\ref{app:attention_context_floor} and
\S\ref{app:case2a_ssm_full_class}, respectively. The attention bound also proves
Proposition~\ref{prop:attention_context_requirement}, and we conclude with
the memory comparison.

We now specialize the geometry of Appendix~\ref{app:vector_routing_geometry} to
the deterministic route $\rho(s)=s-K$, with $T_\rho=T-K$. We retain the
whitened coordinates and innovation covariance $\Gamma$ from
\S\ref{app:case1_prediction_geometry}. To reuse the Case~1 calculations,
we index matched observations by $j=s-K$ and write $n=t-K$ for the number
available at prediction time $t$:
\[
  Y_j:=y_{K+j}=\widetilde x_j^\top\widetilde\theta+\epsilon_{K+j},
  \qquad r_j:=\widetilde x_jY_j,
  \qquad 1\leq j\leq T_\rho.
\]
After $n$ matched observations, the Bayes belief is
\begin{equation}
  \widetilde\theta_{K+n}^\mathrm{Bayes}
  =\EE[\widetilde\theta\mid\widetilde x_{1:n},Y_{1:n}].
  \label{eq:case2a_rotated_bayes_belief}
\end{equation}
At prediction time $t=K+n$, the routed regressor for the next target is
$\widetilde x_{n+1}$ and the evidence atoms $r_1,\ldots,r_n$ are available.  The single
matched prediction before the first observation has zero Bayes mean and
contributes zero to each construction below, so the $T_\rho$ rows $n=1,\ldots,T_\rho$
cover the nontrivial matched predictions.

The next lemma extends the Case~1 bias--variance decomposition to routed
evidence, using the kernel mass $M_n$ and innovation covariance $\Gamma$.

\begin{lemma}[Exact estimation error and regret of a routed kernel]
\label{lem:case2a_vector_kernel_risk}
For deterministic matrices $K_{n,j}\in\RR^{p\times p}$, the corresponding
kernel-form predictor $g$ has
\[
  \widetilde\theta_{K+n}(g)=\sum_{j=1}^nK_{n,j}r_j,
  \qquad M_n:=\sum_{j=1}^nK_{n,j}.
\]
Then
\begin{align}
  \EE\|\widetilde\theta_{K+n}(g)-\widetilde\theta\|_2^2
  &={}
  \bigl\|(M_n-I_p)\widetilde\Sigma_\theta^{1/2}\bigr\|_F^2
  +\sum_{j=1}^n\operatorname{tr}
      (K_{n,j}\Gamma K_{n,j}^\top),
  \label{eq:case2a_vector_kernel_risk}\\
  \sum_{n=1}^{T_\rho}
  \|\widetilde x_{n+1}^\top\widetilde\theta_{K+n}(g)
      -g_{K+n}^{\mathrm{Bayes}}\|_2^2
  &={}
  \sum_{n=1}^{T_\rho}\EE\|\widetilde\theta_{K+n}(g)-\widetilde\theta\|_2^2
  -\mathcal E_{T_\rho}^{\mathrm{Bayes}}.
  \label{eq:case2a_vector_kernel_regret}
\end{align}
Scalar prediction norms denote $L^2$ norms. Both identities hold at finite
horizons; the first gives estimation error, and prediction loss adds $\sigma^2$.
\end{lemma}

\begin{proof}
Write $r_j=\widetilde\theta+e_j$, where
\[
  e_j=(\widetilde x_j\widetilde x_j^\top-I_p)\widetilde\theta
      +\widetilde x_j\epsilon_{K+j}.
\]
Conditionally on $\widetilde\theta$, the $e_j$ are centered and independent.  Gaussian
fourth moments give
\[
  \EE[\widetilde\theta e_j^\top]=0,
  \qquad
  \EE[e_je_k^\top]=0\quad(j\neq k),
  \qquad
  \EE[e_je_j^\top]
  =\widetilde\Sigma_\theta+(\operatorname{tr}\widetilde\Sigma_\theta+\sigma^2)I_p=\Gamma.
\]
Expanding
$\widetilde\theta_{K+n}(g)-\widetilde\theta
=(M_n-I_p)\widetilde\theta+\sum_{j\leq n}K_{n,j}e_j$
proves~\eqref{eq:case2a_vector_kernel_risk}.  The next regressor $\widetilde x_{n+1}$
is fresh and isotropic, so coefficient error and prediction error are
isometric.  Posterior Pythagoras then subtracts
$\EE\operatorname{tr}P_n$ on row $n$; summing proves
\eqref{eq:case2a_vector_kernel_regret}.
\end{proof}

\subsubsection{Proof of Corollary~\ref{cor:case2a_floors}: assembly floors}
\label{app:case2a_assembly_floors}

We first show that one attention layer can assemble only the current
evidence atom (Lemma~\ref{lem:case2a_vector_one_layer_projection});
Corollary~\ref{cor:case2a_vector_one_layer_floor} then gives its regret floor.
For one-layer fixed-transition SSMs,
Lemma~\ref{lem:case2a_vector_short_conv_bridge} shows that convolution
width $d_{\mathrm{conv}}\le K$ limits evidence assembly to at most
$d_{\mathrm{conv}}$ atoms. Lemma~\ref{lem:vector_bounded_bridge} turns
this limitation into a belief-alignment gap, which yields linear regret
at fixed $K$ through Theorem~\ref{thm:routing_floor}.

\begin{lemma}[One attention layer assembles at most the current evidence atom]
\label{lem:case2a_vector_one_layer_projection}
Fix $K\geq1$, ambient width $d_e\geq1$, $d,H,L\geq1$, and
$c_t\in\{1,\tau_t\}$.  For every
$g\in\mathcal F_{\mathrm{att}}^{\mathrm{cal}}
(1,d,H,L;c_t)$ and every matched step $n$, its
implicit belief has the orthogonal decomposition
\begin{equation}
  \widetilde\theta_{K+n}(g)=C_n r_n+\eta_n,
  \qquad
  \eta_n\perp\widetilde\theta,
  \qquad
  \eta_n\perp r_n,
  \label{eq:case2a_one_layer_projection}
\end{equation}
where $r_n=\widetilde x_nY_n$ is the current evidence atom and $C_n$ is a
deterministic matrix. If $L<K+1$, then $C_n=0$. In
particular, no evidence atom $r_j$ with $j<n$ appears in the Bayes-relevant
projection.
\end{lemma}

\begin{proof}
Use the same homogenization as in the proof of
Lemma~\ref{lem:vector_attention_projection}: put
$\bar z_s:=(1,z_s^\top)^\top$.  After combining the affine embedding,
query, key, value, output, residual, and prediction maps, there are
deterministic parameters $a$, $(M_h,v_h)_{h=1}^H$, and relative
positional-bias sequences $(p_u^{(h)})$ such that, exactly,
\begin{equation}
  g_t
  =a^\top\bar z_t
  +\frac1{c_t}\sum_{h=1}^H\sum_{j\in J_t}
    \left[p_{t-j}^{(h)}+\bar z_t^\top M_h\bar z_j\right]
    (v_h^\top\bar z_j).
  \label{eq:case2a_one_layer_reduced}
\end{equation}
Hold out the two adjacent-token occurrences of the fresh regressor
$\widetilde x^\star:=\widetilde x_{t+1-K}$.  In one summand of
\eqref{eq:case2a_one_layer_reduced}, let $m_\star$ be the number of its
internal affine factors from which $\widetilde x^\star$ is selected.  Multiplication
by the external $\widetilde x^\star$ contributes one further factor, so Gaussian
parity eliminates even $m_\star$.  If $K>1$, the current token contains no
$\widetilde x^\star$ and one source token contributes at most its key and value
occurrences; hence $m_\star\leq2$ and only $m_\star=1$ survives.  If $K=1$,
the current token is also a fresh-regressor token, so $m_\star\leq3$; the
additional surviving case $m_\star=3$ is a deterministic Gaussian
fourth-moment term.  In either case the contribution to the implicit belief
has degree at most two in the remaining design and response coordinates.

Project this contribution onto the degree-two space spanned by the evidence atoms
$\{\widetilde x_{s-K}y_s:K<s\leq t\}$.  A historical label $y_s$, $s<t$, occurs only
in its own source token, while its routed regressor $\widetilde x_{s-K}$ occurs in two
strictly earlier tokens.  No summand of~\eqref{eq:case2a_one_layer_reduced}
contains all three of that label, its routed regressor, and the held-out
regressor.  Hence its kernel weight in the projection is zero by the independent
Gaussian sign flips of the missing regressor.  For $s=t$, however, $y_t$ is
in the current token and the source $z_{t-K}$ contains both $\widetilde x_{t-K}$ and one
occurrence of $\widetilde x^\star$; this is the unique routed degree-two contribution.  It is
available exactly when $t-K$ lies in the window, namely when $L\geq K+1$.
Collecting its deterministic kernel weights over all heads gives
$C_nr_n$.  By the degree-two projection argument of
Lemma~\ref{lem:vector_degree_two_projection}, the remainder $\eta_n$ is
orthogonal to both $\widetilde\theta$ and $r_n$.
\end{proof}

This leaves attention with only one noisy evidence atom for estimating the
belief, yielding the following floor.

\begin{corollary}[One-layer attention floor]
\label{cor:case2a_vector_one_layer_floor}
For $c_t\in\{1,\tau_t\}$, every
$g\in\mathcal F_{\mathrm{att}}^{\mathrm{cal}}(1,d,H,L;c_t)$
satisfies
\begin{equation}
  \mathrm{Reg}_T(g)
  \geq
  T_\rho\sum_{i=1}^p d_i-\mathcal E_{T_\rho}^{\mathrm{Bayes}}
  =\Omega(T_\rho).
  \label{eq:case2a_vector_one_layer_floor}
\end{equation}
The bound is uniform in $d_e,d,H,L$ and all learned parameters.  If
$L<K+1$, the coefficient $\sum_i d_i$ may be replaced by $\Lambda$.
\end{corollary}

\begin{proof}
Since $\widetilde\theta_{K+n}(g)$ is $\mathcal G_{K+n}^{-}$-measurable and the additional
unpaired regressors leave the Bayes posterior unchanged, posterior
Pythagoras gives
\[
  \EE\|\widetilde\theta_{K+n}(g)-\widetilde\theta_{K+n}^\mathrm{Bayes}\|^2
  =\EE\|\widetilde\theta_{K+n}(g)-\widetilde\theta\|^2
   -\EE\operatorname{tr}P_n.
\]
Lemma~\ref{lem:case2a_vector_one_layer_projection} and orthogonality let us
discard $\eta_n$.  The best linear estimate of $\widetilde\theta$ from the single evidence atom
$r_n$ has error covariance
\[
  \widetilde\Sigma_\theta
  -\widetilde\Sigma_\theta(\widetilde\Sigma_\theta+\Gamma)^{-1}\widetilde\Sigma_\theta,
\]
whose trace is $\sum_i d_i$. Summing and using the prediction
identity~\eqref{eq:routed_vector_prediction_pythagoras} proves the claim. If $C_n=0$, the
corresponding prior error is $\operatorname{tr}\widetilde\Sigma_\theta=\Lambda$.
\end{proof}

We now turn to SSMs, where a short convolution limits assembly to a bounded
number of recent labels. We use two consequences of the general routing
setup in \S\ref{app:vector_routing_geometry}: first, a persistent
belief-alignment gap forces linear regret; then, limited evidence assembly
creates such a gap.

After $N_t^\rho$ matched observations, posterior Pythagoras gives
\begin{equation}
  \|\widetilde\theta_t^\mathrm{Bayes}\|_t^2
  =\Lambda-\EE\operatorname{tr}P_{N_t^\rho}.
  \label{eq:routed_vector_bayes_energy}
\end{equation}
For $n>p+1$,
\[
  \EE\operatorname{tr}P_n
  \leq\sigma^2\EE\operatorname{tr}
    \left(\sum_{j=1}^n\widetilde x_j\widetilde x_j^\top\right)^{-1}
  =\frac{p\sigma^2}{n-p-1}.
\]
Thus the Bayes belief energy approaches $\Lambda$ and is at least
$\Lambda/2$ after a finite $n_0=n_0(p,\widetilde\Sigma_\theta,\sigma)$
matched observations. A uniform defect $\delta>0$ in the product of the
two factors in Theorem~\ref{thm:routing_floor} on a positive fraction of
matched targets therefore implies $\mathrm{Reg}_T(g)=\Omega(T_\rho)$;
only the first $n_0$ matched targets are lost.

It remains to show why limited evidence assembly leaves a belief-alignment
gap. Unlike the deterministic evidence-atom projection
in~\eqref{eq:case2a_one_layer_projection}, we now condition on the regressors
and project $\widetilde\theta_t(g)$ onto the span of the matched labels:
\begin{equation}
  \widetilde\theta_t(g)=\sum_{s\in\mathcal M_t}w_{t,s}y_s+\eta_t,
  \label{eq:routed_vector_label_projection}
\end{equation}
where $w_{t,s}\in\RR^p$ is design-measurable and $\eta_t$ is conditionally
orthogonal to every matched label.  The conditional label covariance is
strictly positive definite because $\sigma^2>0$, so these coefficients are
unique. The next lemma shows that full belief alignment is impossible if
these coefficients correlate with their routed regressors for only a
bounded number of labels.

\begin{lemma}[Belief alignment under limited evidence assembly]
\label{lem:vector_bounded_bridge}
Fix $p$, $C<\infty$, ${\widetilde\Sigma_\theta\succ0}$, and ${\sigma^2>0}$. There are constants
$n_C<\infty$ and $\delta_{\widetilde\Sigma_\theta,C}>0$ such that the following holds.
If the coefficients in \eqref{eq:routed_vector_label_projection} satisfy
\begin{equation}
  \EE[w_{t,s}\widetilde x_{\rho(s)}^\top]=0
  \qquad(s\in\mathcal M_t\setminus\mathcal J_t)
  \label{eq:vector_route_blind}
\end{equation}
for a deterministic set $\mathcal J_t\subseteq\mathcal M_t$ with
$|\mathcal J_t|\le C$, then
\[
  \operatorname{corr}^2(\widetilde\theta_t(g),\widetilde\theta_t^\mathrm{Bayes})
  \leq1-\delta_{\widetilde\Sigma_\theta,C}
\]
whenever $N_t^\rho\ge n_C$.  The conclusion is uniform over all other model
parameters for which the same bound on $|\mathcal J_t|$ holds.
\end{lemma}

\begin{proof}
The tower property and conditional posterior linearity give
\[
  \langle\widetilde\theta_t(g),\widetilde\theta_t^\mathrm{Bayes}\rangle_t
  =\EE[\widetilde\theta_t(g)^\top\widetilde\theta]
  =\sum_{s\in\mathcal M_t}\EE[w_{t,s}^\top \widetilde\Sigma_\theta\widetilde x_{\rho(s)}].
\]
Put $\mathsf C_t
:=\sum_{s\in\mathcal M_t}w_{t,s}\widetilde x_{\rho(s)}^\top$ and denote the last
display by $m$.  Conditional orthogonality in
\eqref{eq:routed_vector_label_projection} gives
\begin{equation}
  \|\widetilde\theta_t(g)\|_t^2
  \geq
  \EE\operatorname{tr}(\mathsf C_t\widetilde\Sigma_\theta
    \mathsf C_t^\top)
  +\sigma^2\EE\sum_{s\in\mathcal M_t}\|w_{t,s}\|_2^2.
  \label{eq:vector_bridge_norm_lower}
\end{equation}
By matrix Cauchy--Schwarz,
\[
  \EE\operatorname{tr}(\mathsf C_t\widetilde\Sigma_\theta
    \mathsf C_t^\top)\geq\frac{m^2}{\Lambda}.
\]
Condition~\eqref{eq:vector_route_blind} reduces $m$ to the terms in
$\mathcal J_t$, and another Cauchy--Schwarz inequality gives
\[
  m^2
  \leq C\operatorname{tr}(\widetilde\Sigma_\theta^2)
      \EE\sum_{s\in\mathcal J_t}\|w_{t,s}\|_2^2.
\]
For $C>0$, set
$a_{\widetilde\Sigma_\theta,C}:=\sigma^2\Lambda/[C\operatorname{tr}(\widetilde\Sigma_\theta^2)]$. Then
\[
  \|\widetilde\theta_t(g)\|_t^2
  \geq\frac{m^2}{\Lambda}(1+a_{\widetilde\Sigma_\theta,C}).
\]
By~\eqref{eq:routed_vector_bayes_energy} and the posterior-variance bound
following it, choose $n_C$ so that
\[
  \|\widetilde\theta_t^\mathrm{Bayes}\|_t^2
  \geq\Lambda\frac{2+a_{\widetilde\Sigma_\theta,C}}{2(1+a_{\widetilde\Sigma_\theta,C})}.
\]
Then
\[
  \operatorname{corr}^2(\widetilde\theta_t(g),\widetilde\theta_t^\mathrm{Bayes})
  \leq\frac{2}{2+a_{\widetilde\Sigma_\theta,C}}
  =1-\frac{a_{\widetilde\Sigma_\theta,C}}{2+a_{\widetilde\Sigma_\theta,C}}.
\]
Take $\delta_{\widetilde\Sigma_\theta,C}:=a_{\widetilde\Sigma_\theta,C}/(2+a_{\widetilde\Sigma_\theta,C})$.  If $C=0$, then
$\mathcal J_t=\varnothing$, so $m=0$ and the squared correlation is zero directly.
\end{proof}

We now verify condition~\eqref{eq:vector_route_blind} for a one-layer
fixed-transition SSM, including its count-normalized
extension.

\begin{lemma}[Limited evidence assembly in short-convolution SSMs]
\label{lem:case2a_vector_short_conv_bridge}
Fix $K,d_e,d,H\geq1$, $1\leq q\leq K$, and $c_t\in\{1,t\}$.  Every
$g\in\mathcal F_{\mathrm{SSM}}^{\mathrm{cal}}
(1,d,H,q;c_t)$ satisfies the condition of
Lemma~\ref{lem:vector_bounded_bridge} with
\[
  \mathcal J_t
  =\{t-q+1,\ldots,t\}\cap\{K+1,\ldots,t\}.
\]
Thus $|\mathcal J_t|\leq q$, and
\begin{equation}
  \mathrm{Reg}_T
  \bigl(\mathcal F_{\mathrm{SSM}}^{\mathrm{cal}}
  (1,d,H,q;c_t)\bigr)
  =\Omega(T_\rho),
  \label{eq:case2a_vector_short_conv_ssm_floor}
\end{equation}
uniformly over $d_e,d,H$, $1\leq q\leq K$, all learned convolution taps,
and all fixed transition matrices.
\end{lemma}

\begin{proof}
Treat one head and unroll its fixed-transition state.  Its contribution at
time $t$ is a sum over write times of a current affine read multiplied by a
bilinear write; the transition changes only their deterministic
coefficients.  A width-$q$ convolved token at time $u$ can contain labels
from positions $u-q+1,\ldots,u$ and regressors from positions
$u-q+1,\ldots,u+1$, after clipping to available indices.

Hold out the fresh routed regressor $\widetilde x^\star:=\widetilde x_{t+1-K}$.  After
multiplying by $\widetilde x^\star$ and integrating it out to form the implicit
belief, Gaussian parity leaves two possible sources of an evidence atom
$\widetilde x_{s-K}y_s$.  First, if the read supplies $\widetilde x^\star$, both fields of
the atom must come from one write at some time $r$.  If that write contains
$y_s$, then
\[
  s-K\leq r-K\leq r-q,
\]
whereas its earliest available regressor is $\widetilde x_{r-q+1}$, so it cannot
also contain $\widetilde x_{s-K}$.  Second, if a write supplies $\widetilde x^\star$, one
remaining field comes from that write and one from the current read.  If the
read supplies $y_s$, then $s\in\{t-q+1,\ldots,t\}$.  If instead the write
supplies $y_s$, the read would have to supply $\widetilde x_{s-K}$; this is
impossible because
\[
  s-K\leq t-K<t-q+1.
\]
The fourth-moment terms and direct residual contribution are
deterministic and contain no evidence atom.

The implicit belief has degree at most two.  Conditional on the design,
its response-even part is orthogonal to the matched labels, while its
response-odd part is linear in those labels with affine design loadings.
Thus $\EE[w_{t,s}\widetilde x_{s-K}^\top]$ is
the coefficient of the matched monomial $\widetilde x_{s-K}y_s$, and it vanishes
outside $\mathcal J_t$, giving
\eqref{eq:vector_route_blind}.  Summing heads preserves the same set $\mathcal J_t$, and
the deterministic factor $1/c_t$ does not change its support.  Since
$|\mathcal J_t|\leq q\leq K$,
Lemma~\ref{lem:vector_bounded_bridge},
the Bayes energy bound following~\eqref{eq:routed_vector_bayes_energy}, and
Theorem~\ref{thm:routing_floor} prove
\eqref{eq:case2a_vector_short_conv_ssm_floor}.
\end{proof}

Taking $c_t\equiv1$ and $q=d_{\mathrm{conv}}$, together with
Corollary~\ref{cor:case2a_vector_one_layer_floor}, proves
Corollary~\ref{cor:case2a_floors}.

\subsubsection{Constructions and upper bounds for Proposition~\ref{prop:case2a_pair}}
\label{app:case2a_constructions}

The constructions first assemble evidence and address the required regressor,
then apply the Case~1 kernels for belief maintenance.
Proposition~\ref{prop:case2a_vector_attention_witness} implements this with
two attention layers and positional bias;
Proposition~\ref{prop:case2a_vector_ssm_witness} uses a one-layer SSM with a
width-$(K+1)$ convolution. We compute their regret using Lemma~\ref{lem:case2a_vector_kernel_risk}.

For the attention construction, we reuse the Case~1 window notation from
Lemma~\ref{lem:vector_uniform_body} with horizon $T_\rho$ and window
$1\leq L_0\leq T_\rho$. With $n$ indexing matched observations, write
\[
  \tau_n:=\min(n,L_0),
  \qquad N_j:=\sum_{n=1}^{T_\rho}\tau_n^j,\quad j=1,2.
\]
The optimal uniform kernel's cumulative belief-estimation error is then
$\mathcal E_{L_0,T_\rho}$ from~\eqref{eq:vector_equal_tap_cost}.

\begin{proposition}[Two-layer attention construction]
\label{prop:case2a_vector_attention_witness}
Let $L\geq K+1$ and put $L_0:=\min(L,T_\rho)$.  If the ambient
width satisfies $d_e\geq4p+2$, then for every $d,H\geq1$ with
$Hd\geq2p$, there is a two-layer linear-attention predictor
in $\mathcal F_{\mathrm{att}}(2,d,H,L)$ whose regret is
\begin{equation}
  \mathcal E_{L_0,T_\rho}-\mathcal E_{T_\rho}^{\mathrm{Bayes}}.
  \label{eq:case2a_vector_attention_exact}
\end{equation}
Let $g_{\mathrm{box}}$ denote the construction that weights every retained
evidence atom by $1/L$. When $L\leq T_\rho$, its regret is
\begin{align}
  \mathrm{Reg}_T(g_{\mathrm{box}})
  ={}&
  \Lambda\left(\frac{L}{3}-\frac12+\frac{1}{6L}\right)
  +V\left(\frac{T_\rho}{L}-\frac12+\frac{1}{2L}\right)
  -\mathcal E_{T_\rho}^{\mathrm{Bayes}}.
  \label{eq:case2a_vector_attention_mass}
\end{align}
When $L>T_\rho$, the same predictor satisfies
\begin{equation}
  \mathrm{Reg}_T(g_{\mathrm{box}})
  =\Lambda T_\rho\!\left(1-\frac{T_\rho+1}{L}\right)
  +\frac{T_\rho(T_\rho+1)}{6L^2}\bigl[\Lambda(2T_\rho+1)+3V\bigr]
  -\mathcal E_{T_\rho}^{\mathrm{Bayes}}.
  \label{eq:case2a_vector_attention_startup}
\end{equation}
\end{proposition}

\begin{proof}
Use four $p$-dimensional residual blocks and two scalar coordinates, setting
any remaining coordinates to zero.  The initial embedding places the two
whitened regressors and the label in the first two vector blocks and one
scalar coordinate; the other two vector blocks and the accumulator coordinate
are reserved for the two attention layers.

In layer one, set every data-dependent score to zero and spike the relative
positional bias at offset $K$.  Partition the $2p$ coordinates of the two
source vectors among the $H$ heads, with at most $d$ coordinates assigned to
each head.  This is possible because $Hd\geq2p$.  The value map of each
head extracts its assigned coordinates from $z_{s-K}$, and its output map
writes them into the corresponding reserved residual coordinates at position
$s$.  Summing the head outputs therefore copies both
$\widetilde x_{s+1-K}$ and $\widetilde x_{s-K}$ exactly.  Thus the first layer supplies the
next routed regressor and the regressor paired with $y_s$.

For layer two, partition the $p$ eigencoordinates among the heads, again with
at most $d$ coordinates per head.  On the coordinates assigned to head $b$,
choose its query and key maps so that its score is
\[
  \sum_{i\in I_b}\gamma_i
    (U^\top\widetilde x_{n+1})_i(U^\top\widetilde x_j)_i.
\]
Its value map reads $Y_j$ into one head-output coordinate, and its output map
writes the resulting scalar into the shared accumulator coordinate.  Summing
over heads gives
\[
  g_{K+n}
  =\widetilde x_{n+1}^\top U
    \operatorname{diag}(\gamma_1,\ldots,\gamma_p)U^\top
    \sum_{j=n-\tau_n+1}^nr_j.
\]
The final linear readout selects this scalar accumulator.  The residual-width
condition accounts for the original two vector blocks, the two routed blocks,
the label, and the accumulator, while the headwise partitions respect the
per-head query, key, value, and pre-projection output width $d$.

In the eigenbasis $U$, Lemma~\ref{lem:case2a_vector_kernel_risk} separates
the objective over coordinates. For coordinate $i$ it is
\[
  \lambda_i T_\rho-2\lambda_i\gamma_iN_1
  +\gamma_i^2(\lambda_iN_2+v_iN_1).
\]
Optimizing $\gamma_i$ proves~\eqref{eq:case2a_vector_attention_exact}.  If
$L\leq T_\rho$, taking $\gamma_i=1/L$ and using
\[
  \sum_{n=1}^{T_\rho}(1-\tau_n/L)^2
  =\frac{L}{3}-\frac12+\frac{1}{6L},
  \qquad
  \sum_{n=1}^{T_\rho}\tau_n/L^2
  =\frac{T_\rho}{L}-\frac12+\frac{1}{2L}
\]
gives~\eqref{eq:case2a_vector_attention_mass}.  If $L>T_\rho$, then $\tau_n=n$,
and the two sums are
\[
  \sum_{n=1}^{T_\rho}(1-n/L)^2
  =T_\rho-\frac{T_\rho(T_\rho+1)}L
    +\frac{T_\rho(T_\rho+1)(2T_\rho+1)}{6L^2},
  \qquad
  \sum_{n=1}^{T_\rho}\frac{n}{L^2}=\frac{T_\rho(T_\rho+1)}{2L^2}.
\]
Substitution proves~\eqref{eq:case2a_vector_attention_startup}.
\end{proof}

Applying Proposition~\ref{prop:case2a_vector_attention_witness} with the
following window choice gives the attention upper bound in
Proposition~\ref{prop:case2a_pair}. Section~\ref{app:attention_context_floor}
supplies the matching lower bound when $K=O(\sqrt{T_\rho})$.

Recall from~\eqref{eq:vector_stationary_scale}--
\eqref{eq:vector_common_window} the scales $\zeta,\bar\zeta$ and the
nearest-integer common window $L^\star(T_\rho)$.
To ensure positional reach, if
$K+1\leq T_\rho$, choose
\[
  L_0=\min\{T_\rho,\max\{K+1,L^\star(T_\rho)\}\},
  \qquad L=L_0.
\]
When $K+1\leq L^\star(T_\rho)$ and $L^\star(T_\rho)=o(T_\rho)$,
\[
  \mathrm{Reg}_T(g_{\mathrm{box}})
  =\frac{2}{\sqrt3}\bar\zeta\sqrt{T_\rho}+o(\sqrt{T_\rho}).
\]
If $L^\star(T_\rho)<K+1\leq T_\rho$, the same construction has
\[
  \mathrm{Reg}_T(g_{\mathrm{box}})
  \leq
  \frac{\Lambda(K+1)}{3}
  +\frac{VT_\rho}{K+1}-\mathcal E_{T_\rho}^{\mathrm{Bayes}}
  +O(1).
\]
Finally, if $K+1>T_\rho$, choose $L=K+1$.
Equation~\eqref{eq:case2a_vector_attention_startup} gives $O(K)$ regret,
completing the attention upper bound in Proposition~\ref{prop:case2a_pair}.

\begin{proposition}[One-layer SSM construction]
\label{prop:case2a_vector_ssm_witness}
If $Hd\geq p$ and the ambient width satisfies $d_e\geq2p+2$,
the one-layer class
$\mathcal F_{\mathrm{SSM}}
(1,d,H,K{+}1)$ contains a predictor \(g_{\exp}\) satisfying
\begin{equation}
  \mathrm{Reg}_T(g_{\exp})=\zeta\sqrt{T_\rho}+o(\sqrt{T_\rho}).
  \label{eq:case2a_vector_ssm_witness}
\end{equation}
The remainder is for fixed $p,\widetilde\Sigma_\theta,\sigma^2$.
\end{proposition}

\begin{proof}
The initial embedding exposes the two whitened-regressor blocks and the label in
$2p+1$ residual coordinates, leaving one clean output coordinate. At time
$s=K+j$, the width-$(K+1)$ coordinatewise convolution reads the
next-regressor block of $z_{s-K}$ as $\widetilde x_{j+1}$, its current-regressor block
as $\widetilde x_j$, and the current label as $y_s=Y_j$. Distribute the $p$
eigencoordinates among the $Hd$ available state coordinates. In direction $i$,
the bilinear write forms $(1-\alpha_i)(U^\top\widetilde x_j)_iY_j$,
the fixed transition uses a scalar decay $\alpha_i$, and the ordinary
token--state read multiplies the state by $(U^\top\widetilde x_{j+1})_i$.
This gives the whitened kernel
\[
  \begin{gathered}
    K_u=U\operatorname{diag}(\kappa_{u,1},\ldots,\kappa_{u,p})U^\top,\\
    \kappa_{u,i}=(1-\alpha_i)\alpha_i^u,
    \qquad
    \frac{1+\alpha_i}{1-\alpha_i}
    =2\sqrt{\frac{v_iT_\rho}{\lambda_i}}.
  \end{gathered}
\]
Lemma~\ref{lem:case2a_vector_kernel_risk} and the geometric-sum calculation
of Lemma~\ref{lem:vector_stationary_prefix}, applied coordinatewise, give
$\sqrt{\lambda_iv_iT_\rho}+o(\sqrt{T_\rho})$ before the common
$\mathcal E_{T_\rho}^{\mathrm{Bayes}}=O(\log T_\rho)$ subtraction.  Summing the fixed number of coordinates
proves the claim.
\end{proof}

This proves the SSM upper bound in Proposition~\ref{prop:case2a_pair}.
Section~\ref{app:case2a_ssm_full_class} supplies the matching lower bound and
shows that $Hd\geq p$ is necessary for sublinear regret throughout the
one-layer class, regardless of convolution width.

\subsubsection{Attention lower bounds for Propositions~\ref{prop:case2a_pair} and~\ref{prop:attention_context_requirement}}
\label{app:attention_context_floor}

Lemma~\ref{lem:attention_context_floor} derives a regret lower bound from the
limited number of matched observations available to fixed-depth attention.
It proves Proposition~\ref{prop:attention_context_requirement}
and, together with Proposition~\ref{prop:case2a_vector_attention_witness}, the
attention claims in Proposition~\ref{prop:case2a_pair}. The argument depends
only on which raw tokens can influence a prediction, so it applies to both
linear and softmax attention and their count-normalized extensions.

Recall the Bayes estimation error from~\eqref{eq:vector_bayes_baseline}:
\begin{equation}
  e_n^{\mathrm{Bayes}}
  =\EE\operatorname{tr}\left(
    \widetilde\Sigma_\theta^{-1}+\sigma^{-2}\sum_{j=1}^n\widetilde x_j\widetilde x_j^\top
  \right)^{-1},
  \label{eq:case2a_vector_posterior_risk}
\end{equation}
with an empty sum at $n=0$. This is the posterior error after $n$ matched
observations.

\begin{lemma}[Finite-context attention floor]
\label{lem:attention_context_floor}
Fix $1\leq K<T$ and recall $T_\rho=T-K$. Across $D$ layers, a prediction
depends on at most $L_D:=1+D(L-1)$ raw tokens.
For every ambient width $d_e\geq1$, every $D,d,H,L\geq1$, every admissible
count-normalization schedule $\mathbf c$, and every
$g\in\mathcal F_{\mathrm{att}}^{\mathrm{sm,cal}}
(D,d,H,L;\mathbf c)$,
\begin{equation}
  \mathrm{Reg}_T(g)
  \geq
  \sum_{j=L_D+1}^{T_\rho}
  \bigl[e_{L_D}^{\mathrm{Bayes}}-e_j^{\mathrm{Bayes}}\bigr].
  \label{eq:case2a_vector_context_exact}
\end{equation}
If $T_\rho\geq4(L_D+\sigma^2/\min_i\lambda_i+p+2)$, then
\begin{equation}
  \mathrm{Reg}_T(g)
  \geq
  \frac{p\sigma^2T_\rho}{4(L_D+\sigma^2/\min_i\lambda_i)}
  \geq
  \frac{p\sigma^2}{4(1+\sigma^2/\min_i\lambda_i)}\frac{T_\rho}{L_D}.
  \label{eq:case2a_vector_context_clean}
\end{equation}
Both bounds are uniform in $d_e,d,H$ and all attention parameters.  Taking
$\mathbf c=\mathbf1$ gives the same statement for the softmax-enabled class
$\mathcal F_{\mathrm{att}}^{\mathrm{sm}}(D,d,H,L)$, and hence for its
all-linear subclass $\mathcal F_{\mathrm{att}}(D,d,H,L)$.
Consequently, any fixed-depth attention sequence with
$O(\sqrt{T_\rho})$ regret must have $L=\Omega(\sqrt{T_\rho})$.
\end{lemma}

\begin{proof}
Let $I_{r,t}:=[\max\{1,t-r(L-1)\},t]$.  Induction over layers shows that the
representation at position $t$ after layer $r$ is measurable with respect to
$\sigma(z_j:j\in I_{r,t})$.  Indeed, layer $r$ accesses only the $L$ layer-
$(r-1)$ representations ending at $t$, and the union of their raw intervals
is $I_{r,t}$.  Scores, either attention operator, deterministic count normalization,
values, residual updates, and multi-head recombination are measurable
functions of those representations.  Hence the final output depends on at
most $L_D$ raw token positions.  Cached evaluation does not enlarge this
computation graph: a cached key or value is frozen when its position is
computed.

Consider matched step $n$, corresponding to prediction time $t=K+n$. Enlarge
the visible field by revealing every regressor in the entire history, while
retaining only the labels from $I_{D,t}$.  This genie field contains all
variables on which the attention output depends.  At most
$\min(n,L_D)$ of its labels are matched, and their routed regressors are
distinct standard Gaussian vectors.  All other revealed regressors are
independent of $\widetilde\theta$ and contribute no likelihood factor.  Gaussian
conjugacy therefore makes the minimum coefficient estimation error under the
genie field $e_{\min(n,L_D)}^{\mathrm{Bayes}}$, whereas the full-history
Bayes estimation error is $e_n^{\mathrm{Bayes}}$.  The next routed
regressor is fresh and isotropic, so this difference lower-bounds the excess
prediction loss. Conditional-expectation Pythagoras gives
\[
  \ell_{K+n}(g)-\ell_{K+n}(g^{\mathrm{Bayes}})
  \geq
  e_{\min(n,L_D)}^{\mathrm{Bayes}}-e_n^{\mathrm{Bayes}}.
\]
The terms with $n\leq L_D$ vanish.  Summing the remaining rows proves
\eqref{eq:case2a_vector_context_exact}.

For the explicit rate, operator Jensen and the Wishart inverse moment give
\begin{align}
  e_{L_D}^{\mathrm{Bayes}}
  &\geq
  \operatorname{tr}(\widetilde\Sigma_\theta^{-1}+L_D\sigma^{-2}I_p)^{-1}
  \geq\frac{p\sigma^2}{L_D+\sigma^2/\min_i\lambda_i},
  \label{eq:case2a_context_jensen}\\
  e_j^{\mathrm{Bayes}}
  &\leq
  \sigma^2\EE\operatorname{tr}
  \left(\sum_{i=1}^j\widetilde x_i\widetilde x_i^\top\right)^{-1}
  =\frac{p\sigma^2}{j-p-1},
  \qquad j>p+1.
  \label{eq:case2a_context_wishart}
\end{align}
Thus every
$j\geq\lceil2(L_D+\sigma^2/\min_i\lambda_i)+p+1\rceil$ contributes at least
$p\sigma^2/[2(L_D+\sigma^2/\min_i\lambda_i)]$ to
\eqref{eq:case2a_vector_context_exact}.  The stated condition on $T_\rho$ leaves
at least $T_\rho/2$ such indices, proving the first inequality in
\eqref{eq:case2a_vector_context_clean}; the second uses
$L_D+\sigma^2/\min_i\lambda_i\leq(1+\sigma^2/\min_i\lambda_i)L_D$.  Finally, if the displayed condition fails, then
$L_D=\Omega(T_\rho)$ already.  If it holds, the clean bound and
$O(\sqrt{T_\rho})$ regret imply $L_D=\Omega(\sqrt{T_\rho})$.  At fixed $D$,
$L_D=1+D(L-1)=\Theta(L)$, which proves the context requirement.
\end{proof}

Taking $\mathbf c=\mathbf1$ proves
Proposition~\ref{prop:attention_context_requirement}.
When $K=O(\sqrt{T_\rho})$, the window $L=\widetilde L^\star$ in
Proposition~\ref{prop:case2a_pair} satisfies $L=\Theta(\sqrt{T_\rho})$.
At depth two, \eqref{eq:case2a_vector_context_clean} therefore gives
$\Omega(\sqrt{T_\rho})$ regret. Together with the upper bound in
\S\ref{app:case2a_constructions}, this completes the attention claims in
Proposition~\ref{prop:case2a_pair}.

\subsubsection{SSM lower bounds for Proposition~\ref{prop:case2a_pair}}
\label{app:case2a_ssm_full_class}

The construction in Proposition~\ref{prop:case2a_vector_ssm_witness} gives
the upper bound.  For the lower bounds,
Lemma~\ref{lem:case2a_vector_full_ssm_projection} reduces the implicit belief
to a stationary kernel with a boundary term for the first atom.
Lemma~\ref{lem:case2a_full_ssm_low_state} proves linear regret when $Hd<p$,
at any convolution width.  Proposition~\ref{prop:case2a_ssm_general_d} combines
these lemmas with the scalar stationary-kernel bound of
Lemma~\ref{lem:vector_stationary_prefix} to prove both SSM claims in
Proposition~\ref{prop:case2a_pair}. We conclude with the memory comparison.

\begin{lemma}[SSM evidence projection with a boundary term]
\label{lem:case2a_vector_full_ssm_projection}
Fix $K,d_e,d,H\geq1$, and let
\[
  g\in\mathcal F_{\mathrm{SSM}}
  (1,d,H,K{+}1)
\]
use arbitrary learned head-specific convolution taps and arbitrary fixed
transition matrices.  There are deterministic matrices
$W_{\mathrm{read}}\in\RR^{p\times Hd}$,
$\Phi\in\RR^{Hd\times Hd}$, and
$W_{\mathrm{write}}\in\RR^{Hd\times p}$, together with kernel corrections
$\Delta K_0,\ldots,\Delta K_K,\Delta K_0^\partial,\ldots,\Delta K_K^\partial\in\RR^{p\times p}$,
such that at matched step $n$ its implicit belief satisfies
\begin{equation}
  \begin{aligned}
  \widetilde\theta_{K+n}(g)
  ={}&\sum_{j=1}^nW_{\mathrm{read}}\Phi^{n-j}W_{\mathrm{write}}r_j
  +\sum_{j=\max\{2,n-K\}}^n\Delta K_{n-j}r_j\\
  &+\mathbf 1_{\{n\leq K+1\}}\Delta K_{n-1}^\partial r_1+\eta_n,
  \end{aligned}
  \label{eq:case2a_full_ssm_tail_factorization}
\end{equation}
where
$\EE[\eta_n\widetilde\theta^\top]=0$ and
$\EE[\eta_n r_j^\top]=0$ for $1\leq j\leq n$.  Thus every routed kernel weight at
lag at least $K+1$ has range in the common subspace
$\operatorname{col}(W_{\mathrm{read}})$, whose dimension is at most $Hd$.
Equivalently, the same belief has the coarser decomposition
\begin{equation}
  \widetilde\theta_{K+n}(g)
  =A_n r_1+
   \sum_{j=2}^nK_{n-j}r_j+\eta_n,
  \label{eq:case2a_vector_full_ssm_projection}
\end{equation}
where $A_n\in\RR^{p\times p}$ may vary with $n$, whereas the matrices
$K_u\in\RR^{p\times p}$ depend only on the lag $u=n-j$.
The assertion is uniform in the ambient and state dimensions, parameter
magnitudes, and spectra of the fixed transitions.
\end{lemma}

\begin{proof}
It is enough to analyze one head and then sum the resulting projections.
Absorb the fixed whitening $x\mapsto\Sigma_x^{-1/2}x$
into the affine maps. For head $b$,
\[
  \widetilde e_s^{(b)}
  =\sum_{j=0}^{K}C_j^{(b)}\odot e_{s-j}^{(0)},
  \qquad
  u_s^{(b)}
  =(W_K^{(b)}\widetilde e_s^{(b)})
   \odot(W_V^{(b)}\widetilde e_s^{(b)}),
  \qquad
  h_t^{(b)}=\sum_{s\leq t}(\Phi^{(b)})^{t-s}u_s^{(b)}.
\]
The slice of the final linear map acting on this head turns its
token--state read into $q_t^{(b)\top}h_t^{(b)}$, where
$q_t^{(b)}$ is affine in $\widetilde e_t^{(b)}$.  Biases in the affine
factors are included by homogenizing with a constant coordinate.

\paragraph{Computing the implicit belief.}
No observed label uses the fresh regressor $\widetilde x_{n+1}$. After holding
out its two adjacent-token occurrences, expand, for $i,j,k\in\{1,\ldots,p\}$,
\begin{align*}
  q_t^{(b)}
  &=q_{t,0}^{(b)}+\sum_i \widetilde x_{n+1,i} q_{t,i}^{(b)},\\
  u_s^{(b)}
  &=u_{s,0}^{(b)}+\sum_j\widetilde x_{n+1,j} u_{s,j}^{(b)}
    +\sum_{j,k}\widetilde x_{n+1,j}\widetilde x_{n+1,k} u_{s,jk}^{(b)}.
\end{align*}
All displayed coefficients are free of $\widetilde x_{n+1}$.  The vectors
$q_{t,i}^{(b)}$ and $u_{s,jk}^{(b)}$ are deterministic: they are the
coefficients of $\widetilde x_{n+1}$ in an affine read and of
$\widetilde x_{n+1,j}\widetilde x_{n+1,k}$ in a product of two
affine write factors.  The remaining coefficients
$q_{t,0}^{(b)}$ and $u_{s,j}^{(b)}$ have degree at most one in the other
token fields, while $u_{s,0}^{(b)}$ has degree at most two.

Put $u:=t-s$.  Gaussian parity leaves only terms with one or three internal
fresh-regressor factors, giving
\begin{equation}
  \EE\!\left[
    \widetilde x_{n+1,i} q_t^{(b)\top}(\Phi^{(b)})^u u_s^{(b)}
    \mid\mathcal G_t^-
  \right]
  =q_{t,i}^{(b)\top}(\Phi^{(b)})^u u_{s,0}^{(b)}
   +q_{t,0}^{(b)\top}(\Phi^{(b)})^u u_{s,i}^{(b)}
   +c_{i,u}^{(b)},
  \label{eq:case2a_vector_ssm_contraction}
\end{equation}
where $c_{i,u}^{(b)}$ collects the three-factor terms and is deterministic
because their coefficients $q_{t,i}^{(b)}$ and $u_{s,jk}^{(b)}$ are
deterministic.
Thus every nonconstant term on the right of
\eqref{eq:case2a_vector_ssm_contraction} has total degree at most two in the
remaining design and response fields.

The proof of Lemma~\ref{lem:vector_degree_two_projection} now applies to the
matched pairs $(\widetilde x_j,Y_j)$, even though their two fields occupy
different raw tokens.  Partial injectivity makes the $\widetilde x_j$'s distinct.
The paired flip
$(\widetilde x_j,\epsilon_{K+j})\mapsto(-\widetilde x_j,-\epsilon_{K+j})$ fixes every other
routed evidence atom and shows that a cross-atom $\widetilde x_iY_j$, $i\ne j$, is orthogonal
to both $\widetilde\theta$ and the routed evidence space.  The global response flip treats
the zero- and two-response terms.  Consequently the projection of
\eqref{eq:case2a_vector_ssm_contraction} onto the routed evidence has
deterministic matrix-valued kernel weights, and every remaining term belongs to a
residual that is coordinatewise orthogonal to both $\widetilde\theta$ and that evidence.
Equivalently, after summing these terms into $\eta_n$,
$\EE[\eta_n\widetilde\theta^\top]=0$ and
$\EE[\eta_n r_j^\top]=0$ for every $j\leq n$.

\paragraph{Regrouping by label time.}
At write time $s$, a width-$(K+1)$ convolution can see regressors from
$s-K$ through $s+1$ and labels from $s-K$ through $s$, clipped to available
indices. Fixed coordinate mixing changes the matrix coefficient of a
monomial but not these time indices.

First consider the term in
\eqref{eq:case2a_vector_ssm_contraction} in which the read supplies the
fresh regressor and the $\widetilde x_{n+1}$-free write supplies a matched
pair. A label at time $r\leq s$ has its regressor available only if
$r-K\geq s-K$, forcing $r=s$. Hence the only possible direct routed write
is $\widetilde x_{s-K}y_s$. For head
$b$, let $W_{\mathrm{write}}^{(b)}\in\RR^{d\times p}$ be its effective
evidence-write map, and let
$W_{\mathrm{read}}^{(b)}\in\RR^{p\times d}$ collect the coefficients with
which the read supplies the fresh regressor.  The state-mediated kernel weight
on $r_j$ at routed lag $u=n-j$ is therefore
\[
  W_{\mathrm{read}}^{(b)}(\Phi^{(b)})^uW_{\mathrm{write}}^{(b)}.
\]

The other nonconstant term in
\eqref{eq:case2a_vector_ssm_contraction} takes one occurrence of $\widetilde x_{n+1}$
from a write at $s=t-u$ and one remaining field from the current read.  Such a write
exists only for $0\leq u\leq K$.  If the read supplies the label
$y_r=y_{t-v}$, then $0\leq v\leq K$, and the remaining write factor can be
the routed regressor $\widetilde x_{t-v-K}$ exactly when
\[
  0\leq v\leq u\leq K.
\]
Indeed, its regressor tap is $K+v-u$; when both adjacent-token occurrences
are available, the duplicate next-regressor path has tap $K+v-u+1$.
For fixed $v=t-r$, summing over these $u$ therefore gives a fixed bypass
matrix depending only on $v$.  If instead the read supplies the routed
regressor and the write supplies the label, the only possibility is
$v=u=0$, which adds another lag-zero bypass.  Thus all such bypasses act on
the $K+1$ most recent routed atoms.  Denote their head-$b$ matrices by
$\Delta K_0^{(b)},\ldots,\Delta K_K^{(b)}$.

The first routed regressor $\widetilde x_1$ lacks its duplicate next-regressor
occurrence because $e_0^{(0)}=W_{\mathrm{emb}}z_0=0$.  This can alter a
bypass correction while $r_1$ is
among the $K+1$ most recent atoms; denote its resulting bypass matrices by
$\Delta K_0^{\partial,(b)},\ldots,\Delta K_K^{\partial,(b)}$.  Its direct routed write is
unaffected, because it uses the current occurrence of $\widetilde x_1$ in $z_1$ at
tap $K$.  Consequently, once $n\geq K+2$, the kernel weight on $r_1$ is exactly
$W_{\mathrm{read}}^{(b)}(\Phi^{(b)})^{n-1}W_{\mathrm{write}}^{(b)}$.  These
cases exhaust the nonconstant terms in
\eqref{eq:case2a_vector_ssm_contraction}; the fourth-moment term
$c_{i,u}^{(b)}$ is deterministic and has zero projection onto routed evidence.

Finally, concatenate the heads:
\[
  W_{\mathrm{read}}
  :=[W_{\mathrm{read}}^{(1)}\ \cdots\ W_{\mathrm{read}}^{(H)}],
  \qquad
  \Phi:=\operatorname{diag}(\Phi^{(1)},\ldots,\Phi^{(H)}),
  \qquad
  W_{\mathrm{write}}
  :=\begin{bmatrix}
       W_{\mathrm{write}}^{(1)}\\[-1mm]
       \vdots\\[-1mm]
       W_{\mathrm{write}}^{(H)}
     \end{bmatrix}.
\]
Summing the bypass matrices and residuals over heads proves
\eqref{eq:case2a_full_ssm_tail_factorization}.  Defining
\[
  A_n:=W_{\mathrm{read}}\Phi^{n-1}W_{\mathrm{write}}
       +\mathbf 1_{\{n\leq K+1\}}\Delta K_{n-1}^\partial,
  \qquad
  K_u:=W_{\mathrm{read}}\Phi^uW_{\mathrm{write}}
       +\mathbf 1_{\{u\leq K\}}\Delta K_u
\]
then gives the coarser stationary-profile form
\eqref{eq:case2a_vector_full_ssm_projection}.
\end{proof}

The first evidence atom needs a separate kernel weight even for
$p=d=H=1$ and $K=1$. Choose convolved residual coordinates
\[
  \widetilde e_{s,1}:=(z_s)_1+(z_{s-1})_1,
  \qquad
  \widetilde e_{s,2}:=y_s,
\]
write $u_s=\widetilde e_{s,1}^2$, propagate with a nonzero scalar decay, and
read with $y_t h_t$.  The lag-zero kernel weight on the first routed atom can
then differ from the lag-zero kernel weight on every later routed atom because
$x_1$, unlike later regressors, has no occurrence in a preceding token:
position $0$ is zero-padded.

The same calculation of the implicit belief exposes a state-dimension bottleneck
that arbitrary convolution cannot remove: outside the at-most-$Hd$-
dimensional current-read image, matching an additional belief direction
through the remaining bilinear read--write term necessarily incurs label
noise.

\begin{lemma}[State-dimension floor at any convolution width]
\label{lem:case2a_full_ssm_low_state}
Fix $p$, $\widetilde\Sigma_\theta\succ0$, and $\sigma^2>0$.  For
$1\leq K<T$ and $d_{\mathrm{conv}},d_e,d,H\geq1$,
suppose $Hd<p$.  Every
$g\in\mathcal F_{\mathrm{SSM}}(1,d,H,d_{\mathrm{conv}})$ satisfies
\begin{equation}
  \mathrm{Reg}_T(g)
  \geq T_\rho\,
  \frac{\lambda_{\min}(\widetilde\Sigma_\theta)\sigma^2}
       {\sigma^2+2(p-1)\lambda_{\min}(\widetilde\Sigma_\theta)}
  -\mathcal E_{T_\rho}^{\mathrm{Bayes}}.
  \label{eq:case2a_full_ssm_low_state}
\end{equation}
Thus the class has $\Omega(T_\rho)$ regret, uniformly over all remaining
parameters, even when they depend on $T$.
\end{lemma}

\begin{proof}
Fix a matched step $n$, put $t=K+n$, and hold out the fresh regressor
$\widetilde x_{n+1}$. Concatenate the head states into
$h_t\in\RR^{Hd}$ and write the scalar output as
$g_t^{\mathrm{dir}}+q_t^\top h_t$, absorbing the output maps into the affine current
read.  The direct residual $g_t^{\mathrm{dir}}$ is affine, so its deterministic
held-out-regressor coefficient is included in $c_i$ below.  The
fresh-regressor calculation in the proof of
Lemma~\ref{lem:case2a_vector_full_ssm_projection}, before its
regrouping-by-label-time step, uses only the linearity of convolution and
therefore applies at any convolution width.  Thus
\[
  q_t=q_0+A\widetilde x_{n+1},
  \qquad
  h_t=h_0+\sum_i\widetilde x_{n+1,i}h_i
            +\sum_{i,j}\widetilde x_{n+1,i}\widetilde x_{n+1,j}h_{ij},
\]
where $A\in\RR^{Hd\times p}$ and the $h_{ij}$ are deterministic,
$q_0$ and the $h_i$ are affine in the remaining observed fields, and
$h_0$ has degree at most two. The conditional expectation defining the implicit belief gives
\begin{equation}
  \widetilde\theta_{K+n,i}(g)=A_{:i}^\top h_0+q_0^\top h_i+c_i
  \label{eq:case2a_arbitrary_conv_contraction}
\end{equation}
for deterministic $c_i$.

Choose a possibly row-dependent unit vector $a\in\ker A$ and set
$h_a:=\sum_i a_i h_i$.
Such a vector exists because $Hd<p$.  Equation
\eqref{eq:case2a_arbitrary_conv_contraction} gives
\[
  a^\top\widetilde\theta_{K+n}(g)=q_0^\top h_a+\text{constant}.
\]
Let $X$ stack each distinct remaining regressor once and put
$Y:=(Y_1,\ldots,Y_n)^\top$.  Since $q_0$ and $h_a$ are affine, write
\[
  q_0=b+BX+CY,
  \qquad
  h_a=e+EX+FY.
\]
The response-odd mixed component of their product is $X^\top MY$, where
\[
  M=B^\top F+E^\top C,
  \qquad \operatorname{rank}(M)\leq2Hd.
\]
The global response flip makes every response-even term orthogonal to both
the target and the response-odd terms.  Flipping all regressors and matched
noises while holding $\widetilde\theta$ fixed then makes the remaining label-linear
term odd, while $X^\top MY$ and $a^\top\widetilde\theta$ remain invariant.  Hence
\begin{equation}
  \EE\bigl(a^\top\widetilde\theta_{K+n}(g)-a^\top\widetilde\theta\bigr)^2
  \geq
  \EE\bigl(X^\top MY-a^\top\widetilde\theta\bigr)^2.
  \label{eq:case2a_arbitrary_conv_mixed_reduction}
\end{equation}

Let $\mu\in\RR^p$ be the sum of the matched blocks of $M$, pairing the
block for $\widetilde x_j$ with the column for $Y_j$.  Since
$Y_j=\widetilde x_j^\top\widetilde\theta+\epsilon_{K+j}$, retaining the complete matrix $M$
gives
\begin{equation}
  \EE\bigl(X^\top MY-a^\top\widetilde\theta\bigr)^2
  \geq
  (\mu-a)^\top\widetilde\Sigma_\theta(\mu-a)+\sigma^2\|M\|_F^2.
  \label{eq:case2a_arbitrary_conv_signal_noise}
\end{equation}
Indeed, the independent label noise contributes exactly
$\sigma^2\|M\|_F^2$, while Jensen's inequality applied to the remaining
quadratic-design term gives the first term.

We next use the rank bound on $M$ to relate the matched coefficient $\mu$
to the noise cost. For $v\in\RR^p$, define $E_v$ with the same dimensions
as $M$: column $j$ contains $v$ in the rows corresponding to
$\widetilde x_j$ and zeros elsewhere. Its inner product with $M$ therefore
selects exactly the matched blocks, giving
\[
  \langle M,E_v\rangle_F=v^\top\mu.
\]
Since the matched regressors are distinct, these columns are orthogonal,
each with norm $\|v\|_2$. Hence $\|E_v\|_{\mathrm{op}}=\|v\|_2$.
A singular-value decomposition and $\operatorname{rank}(M)\leq2Hd$ now give
\[
  |v^\top\mu|
  \leq\sqrt{\operatorname{rank}(M)}\,\|M\|_F\|E_v\|_{\mathrm{op}}
  \leq\sqrt{2Hd}\,\|M\|_F\|v\|_2,
  \qquad
  \|\mu\|_2^2\leq2Hd\|M\|_F^2.
\]
Thus the noise term is at least $\sigma^2\|\mu\|_2^2/(2Hd)$.
Combining this with
\eqref{eq:case2a_arbitrary_conv_mixed_reduction}--
\eqref{eq:case2a_arbitrary_conv_signal_noise} and completing the square gives
\begin{align*}
  \EE\bigl(a^\top\widetilde\theta_{K+n}(g)-a^\top\widetilde\theta\bigr)^2
  &\geq
  \inf_{\mu\in\RR^p}
  \left\{(\mu-a)^\top\widetilde\Sigma_\theta(\mu-a)
  +\frac{\sigma^2}{2Hd}\|\mu\|_2^2\right\}\\
  &=a^\top\left(\widetilde\Sigma_\theta^{-1}
       +\frac{2Hd}{\sigma^2}I_p\right)^{-1}a\\
  &\geq
  \frac{\lambda_{\min}(\widetilde\Sigma_\theta)\sigma^2}
       {\sigma^2+2(p-1)\lambda_{\min}(\widetilde\Sigma_\theta)}.
\end{align*}
The coefficient estimation error dominates this scalar projection on every
matched step. Summing the routed-regressor and posterior Pythagorean identities
proves~\eqref{eq:case2a_full_ssm_low_state}.
The same proof allows wider bilinear write features linearly combined into
the state, and reads pairing affine token features with linear state
projections. These preserve the quadratic state and affine $q_t$, so
\eqref{eq:case2a_arbitrary_conv_contraction} and the rank bound $2Hd$ are
unchanged, regardless of the intermediate feature widths.
\end{proof}

We now combine the preceding lemmas with the Case~1 stationary-kernel bound
and the earlier construction to establish the SSM class regret.

\begin{proposition}[SSM class regret]
\label{prop:case2a_ssm_general_d}
Fix $p$, $\widetilde\Sigma_\theta\succ0$, and $\sigma^2>0$, and let $1\leq K<T$ with
$T_\rho=T-K\to\infty$; $K$ may depend on $T$.  For every
$d_e,d,H\geq1$,
\begin{equation}
  \mathrm{Reg}_T
  \bigl(\mathcal F_{\mathrm{SSM}}
  (1,d,H,K{+}1)\bigr)
  \geq
  \sqrt{T_\rho-1}\sum_{i=1}^p\sqrt{d_iv_i}
  -\mathcal E_{T_\rho}^{\mathrm{Bayes}}+o(\sqrt{T_\rho}).
  \label{eq:case2a_vector_full_ssm_lower}
\end{equation}
If in addition $Hd\geq p$ and $d_e\geq2p+2$, then
\begin{equation}
  \mathrm{Reg}_T
  \bigl(\mathcal F_{\mathrm{SSM}}
  (1,d,H,K{+}1)\bigr)
  \leq\zeta\sqrt{T_\rho}+o(\sqrt{T_\rho}),
  \label{eq:case2a_vector_full_ssm_upper}
\end{equation}
and hence the class regret is $\Theta(\sqrt{T_\rho})$.
The bounds establish the rate but need not have sharp constants.
If instead $Hd<p$, then for every
$d_{\mathrm{conv}}\geq1$,
\begin{equation}
  \mathrm{Reg}_T
  \bigl(\mathcal F_{\mathrm{SSM}}
  (1,d,H,d_{\mathrm{conv}})\bigr)
  =\Omega(T_\rho).
  \label{eq:case2a_vector_full_ssm_low_state_rate}
\end{equation}
\end{proposition}

\begin{proof}
Apply Lemma~\ref{lem:case2a_vector_full_ssm_projection} and discard the
orthogonal residual.  Let $M_n:=\sum_{u<n-1}K_u$.  Before Bayes subtraction,
the row estimation error involving the unrestricted kernel weight on the first atom is
\[
  \operatorname{tr}\!\left[
    (A_n+M_n-I_p)\widetilde\Sigma_\theta(A_n+M_n-I_p)^\top
  \right]
  +\operatorname{tr}(A_n\Gamma A_n^\top)
  +\sum_{u<n-1}\operatorname{tr}(K_u\Gamma K_u^\top).
\]
Its exact matrix minimum over $A_n$ is obtained at
\[
  A_n=-(M_n-I_p)\widetilde\Sigma_\theta(\widetilde\Sigma_\theta+\Gamma)^{-1}
\]
and, using the eigenbasis $U$ of $\widetilde\Sigma_\theta$, equals
\[
  \operatorname{tr}\!\left[
    (M_n-I_p)UD_dU^\top(M_n-I_p)^\top
  \right]
  +\sum_{u<n-1}\operatorname{tr}(K_u\Gamma K_u^\top),
  \qquad D_d:=\operatorname{diag}(d_1,\ldots,d_p).
\]
For $n=1$, $M_1=0$ and the minimized row is
$\operatorname{tr}(D_d)\geq0$, so we may discard it.  Reindexing the remaining
rows $n=2,\ldots,T_\rho$ by $n-1$ gives the stationary-prefix functional at
horizon $T_\rho-1$.
In this basis, the bias and variance matrices are $D_d$ and
$U^\top\Gamma U=\operatorname{diag}(v_1,\ldots,v_p)$.
Every off-diagonal entry of $U^\top K_uU$ therefore contributes only a
nonnegative square.  Discarding those entries leaves, in direction $i$, the
stationary-prefix functional with
bias coefficient $d_i$, variance coefficient $v_i$, and horizon $T_\rho-1$.
Lemma~\ref{lem:vector_stationary_prefix} yields
$\sqrt{d_iv_i(T_\rho-1)}+o(\sqrt{T_\rho})$, uniformly over that profile.  Summing the
fixed number of coordinates and subtracting the cumulative Bayes estimation error
$\mathcal E_{T_\rho}^{\mathrm{Bayes}}$ proves~\eqref{eq:case2a_vector_full_ssm_lower}.  The construction in
Proposition~\ref{prop:case2a_vector_ssm_witness} lies in the class whenever
$Hd\geq p$ and $d_e\geq2p+2$, and therefore
proves~\eqref{eq:case2a_vector_full_ssm_upper}.  Finally,
$\mathcal E_{T_\rho}^{\mathrm{Bayes}}=O(\log T_\rho)=o(\sqrt{T_\rho})$ and all $d_i,v_i$ are positive, which gives the
$\sqrt{T_\rho}$ rate.  When $Hd<p$,
Lemma~\ref{lem:case2a_full_ssm_low_state} gives
\eqref{eq:case2a_vector_full_ssm_low_state_rate}.
\end{proof}

This completes the SSM claims in Proposition~\ref{prop:case2a_pair}.
Convolution supplies positional reach, while the recurrent state must
still span at least $p$ dimensions.  With $d_e\geq2p+2$,
$d_{\mathrm{conv}}=K+1$ and $Hd\geq p$ suffice for $\Theta(\sqrt{T_\rho})$ regret.

\paragraph{Memory comparison.}
\label{app:case2a_memory_comparison}
For the SSM construction in Proposition~\ref{prop:case2a_vector_ssm_witness},
take $D=H=1$, $d=p$, and $d_e=2p+2$. Its learned parameters, recurrent
state, and convolution buffers together occupy $O(p^2+p(K+1))$ scalar
entries, hence $O(K+1)$ for fixed $p$.
By Lemma~\ref{lem:attention_context_floor}, any fixed-depth attention
achieving $O(\sqrt{T_\rho})$ regret requires $L=\Omega(\sqrt{T_\rho})$.
With standard per-position caching of token representations or keys and
values, this entails $\Omega(\sqrt{T_\rho})$ scalar entries.
Thus when $K=o(\sqrt{T_\rho})$, the SSM attains the same regret rate
with asymptotically less total storage.

\subsection{Case~2.b: content routing}
\label{app:case2b_capacity}

We first show how limited addressing capacity forces linear regret
(\S\ref{app:case2b_capacity_regret}), then develop geometric bounds on capacity
(\S\ref{app:case2b_capacity_geometry}) and apply them to the architecture
classes (\S\ref{app:case2b_capacity_architectures}).
Section~\ref{app:case2b_capacity_theorem} combines these ingredients to prove
Theorem~\ref{thm:case2b_floors}. Section~\ref{app:case2b_vector_witnesses}
then gives the attention and selective-SSM constructions for
Proposition~\ref{prop:case2b_pair}.

\subsubsection{Addressing capacity and regret}
\label{app:case2b_capacity_regret}

We first establish how addressing capacity controls regret. To do so, we
isolate each candidate regressor's component along the Bayes-belief direction.

We express the whitened variables from \eqref{eq:vector_whitening}
in the eigenbasis $U$ of the whitened prior covariance
$\widetilde\Sigma_\theta$.
Fix $p\in\mathbb N$, $\Sigma_x\succ0$,
$\Sigma_\theta\succ0$, and $\sigma^2>0$,
and write
\begin{equation}
  \widetilde\Sigma_\theta:=\Sigma_x^{1/2}\Sigma_\theta\Sigma_x^{1/2}
    =U\operatorname{diag}(\lambda_1,\ldots,\lambda_p)U^\top,
  \quad
  \xi_s:=U^\top\widetilde x_s,
  \quad
  \beta:=U^\top\widetilde\theta.
  \label{eq:case2b_whitening}
\end{equation}
Thus $\xi_s\stackrel{\mathrm{iid}}\sim\mathcal N(0,I_p)$,
$\beta\sim\mathcal N(0,
  \operatorname{diag}(\lambda_1,\ldots,\lambda_p))$, and a routed label is
$y_s=\xi_{\rho(s)}^\top\beta+\epsilon_s$.  Fixed linear maps in either
architecture implement the change of coordinates in
\eqref{eq:case2b_whitening}.

Fix a prediction time $t\geq R$ in the content-routing model.  Let $\mathcal U_t$ be
the set of key indices not yet visited in the cycle containing $t+1$, and set
$N_t:=|\mathcal U_t|$.  The upcoming key index is uniform on $\mathcal U_t$.  For
$q\in\mathcal U_t$, let $X_q\in\RR^p$ be the transformed regressor $\xi_s$ at the
most recent occurrence of key $q$.

To isolate the candidate vectors, we condition on the remaining history.
Let $\mathcal C_t^-$ be the sigma-field generated by $z_{1:t}$ after
removing the upcoming key $\psi_{t+1}$ and every occurrence of the
candidate vectors $(X_q)_{q\in\mathcal U_t}$. Conditional on $\mathcal C_t^-$,
\begin{equation}
  (X_q)_{q\in\mathcal U_t}\stackrel{\mathrm{iid}}\sim\mathcal N(0,I_p),
  \label{eq:case2b_candidate_vectors}
\end{equation}
independently of the upcoming key.  The Bayes belief
$\hat\beta_t^\mathrm{Bayes}:=\EE[\beta\mid z_{1:t}]$ is
$\mathcal C_t^-$-measurable. The predictor may still depend arbitrarily on
the candidate vectors through their earlier occurrences in the input.
The arguments below preserve this dependence.

On $\{\|\hat\beta_t^\mathrm{Bayes}\|>0\}$ define
$\bar\beta_t^\mathrm{Bayes}:=\hat\beta_t^\mathrm{Bayes}/
\|\hat\beta_t^\mathrm{Bayes}\|$; when the norm vanishes, choose an arbitrary
deterministic unit vector.  Put
\begin{equation}
  Z_q:=(\bar\beta_t^\mathrm{Bayes})^\top X_q,
  \qquad
  X_q^\perp:=
    (I_p-\bar\beta_t^\mathrm{Bayes}
      (\bar\beta_t^\mathrm{Bayes})^\top)X_q,
  \qquad
  \mathcal C_t:=\mathcal C_t^-\vee
    \sigma(X_q^\perp:q\in\mathcal U_t).
  \label{eq:case2b_scalarization}
\end{equation}
Conditional on $\mathcal C_t$,
\begin{equation}
  (Z_q)_{q\in\mathcal U_t}\stackrel{\mathrm{iid}}\sim\mathcal N(0,1),
  \qquad
  (\hat\beta_t^\mathrm{Bayes})^\top X_q
  =\|\hat\beta_t^\mathrm{Bayes}\|Z_q.
  \label{eq:case2b_scalar_candidates}
\end{equation}
The second identity also holds when $\hat\beta_t^\mathrm{Bayes}=0$ because both sides then
vanish. For $q\in\mathcal U_t$, let $g_{t,q}$ denote the predictor's output with
the upcoming key fixed to $q$ and all other input coordinates unchanged.

For conditionally square-integrable $V,W$, set
\[
  \operatorname{corr}_{\mathcal C_t}^2(V,W)
  :=
  \frac{\EE[VW\mid\mathcal C_t]^2}
       {\EE[V^2\mid\mathcal C_t]\,\EE[W^2\mid\mathcal C_t]}.
\]
The addressing capacity is
\begin{equation}
  \mathrm{Cap}_t(g)
  :=
  \sum_{q\in\mathcal U_t}
  \operatorname{corr}_{\mathcal C_t}^2(g_{t,q},Z_q).
  \label{eq:cosine_addressing_capacity}
\end{equation}
For $p=1$, there are no orthogonal components to condition on.

For comparison with Theorem~\ref{thm:routing_floor}, fix $q\in\mathcal U_t$ and
let
\[
  \mathcal G_{t,q}^-
  :=
  \mathcal C_t\vee\sigma(Z_r:r\in\mathcal U_t,\ r\neq q),
  \qquad
  \hat a_{t,q}(g)
  :=
  \EE[Z_qg_{t,q}\mid\mathcal G_{t,q}^-].
\]
Because $Z_q$ is conditionally standard Gaussian and independent of
$\mathcal G_{t,q}^-$, conditional orthogonal projection gives
$g_{t,q}=\hat a_{t,q}(g)Z_q+\varepsilon_{t,q}$, where
$\varepsilon_{t,q}$ is orthogonal to every $cZ_q$ with
$c\in L^2(\mathcal G_{t,q}^-)$.  Therefore,
\begin{equation}
  \operatorname{corr}_{\mathcal C_t}^2
    (g_{t,q},\|\hat\beta_t^\mathrm{Bayes}\|Z_q)
  =
  \underbrace{
    \operatorname{corr}_{\mathcal C_t}^2
      (\hat a_{t,q}(g),\|\hat\beta_t^\mathrm{Bayes}\|)
  }_{\text{conditional belief alignment}}
  \underbrace{
    \frac{\EE[\hat a_{t,q}(g)^2\mid\mathcal C_t]}
         {\EE[\hat a_{t,q}(g)^2+\varepsilon_{t,q}^2
              \mid\mathcal C_t]}
  }_{\text{conditional addressing fidelity}}.
  \label{eq:conditional_two_factor_product}
\end{equation}

The next two lemmas connect addressing capacity to regret:
Lemma~\ref{lem:cosine_capacity_two_factor} bounds the average two-factor
product by $\mathrm{Cap}_t(g)/N_t$, and
Lemma~\ref{lem:cosine_capacity_to_regret} shows that sublinear capacity
forces linear regret. This reduces the regret floors in
Theorem~\ref{thm:case2b_floors} to bounding each architecture's addressing
capacity.

\begin{lemma}[Capacity controls the two-factor product]
\label{lem:cosine_capacity_two_factor}
Conditionally on $\mathcal C_t$, the product in
\eqref{eq:conditional_two_factor_product}, averaged over the upcoming key,
is at most $\mathrm{Cap}_t(g)/N_t$.  Moreover,
\begin{equation}
  \frac1{N_t}\sum_{q\in\mathcal U_t}
  \EE[(g_{t,q}-(\hat\beta_t^\mathrm{Bayes})^\top X_q)^2
    \mid\mathcal C_t]
  \geq
  \|\hat\beta_t^\mathrm{Bayes}\|^2
  \left(1-\frac{\mathrm{Cap}_t(g)}{N_t}\right).
  \label{eq:cosine_capacity_conditional_loss}
\end{equation}
\end{lemma}

\begin{proof}
When $\|\hat\beta_t^\mathrm{Bayes}\|>0$, multiplying $Z_q$ by the
$\mathcal C_t$-measurable nonzero scalar
$\|\hat\beta_t^\mathrm{Bayes}\|$ does not change conditional squared
correlation.  When $\|\hat\beta_t^\mathrm{Bayes}\|=0$, the two-factor
product is zero.  Thus each product is at most
$\operatorname{corr}_{\mathcal C_t}^2(g_{t,q},Z_q)$, and averaging proves
the first claim.

For the second claim, the conditional distance from the target
$\|\hat\beta_t^\mathrm{Bayes}\|Z_q$ to the line spanned by $g_{t,q}$ gives
\[
  \EE[(g_{t,q}-\|\hat\beta_t^\mathrm{Bayes}\|Z_q)^2\mid\mathcal C_t]
  \geq
  \|\hat\beta_t^\mathrm{Bayes}\|^2
  \left(1-
    \operatorname{corr}_{\mathcal C_t}^2(g_{t,q},Z_q)
  \right).
\]
Use~\eqref{eq:case2b_scalar_candidates} and average over $q$.
\end{proof}

\begin{lemma}[Sublinear capacity forces linear regret]
\label{lem:cosine_capacity_to_regret}
Recall $T_\rho=T-R$. Suppose there is a deterministic set of
$\Omega(T_\rho)$ prediction times at which $N_t\geq R/2$ and
$\mathrm{Cap}_t(g)\leq\varepsilon_R R$ almost surely for every $g$ in a
predictor class, where $\varepsilon_R\to0$ is deterministic and uniform
over these times and predictors. For fixed $p$ and fixed
$\widetilde\Sigma_\theta\succ0$, the class has $\Omega(T_\rho)$ Bayes regret.
\end{lemma}

\begin{proof}
Bayes orthogonality and the uniform reveal give
\[
  \ell_t(g)-\ell_t(g^{\mathrm{Bayes}})
  =
  \EE\!\left[
    \frac1{N_t}\sum_{q\in\mathcal U_t}
    \EE[(g_{t,q}-(\hat\beta_t^\mathrm{Bayes})^\top X_q)^2
      \mid\mathcal C_t]
  \right].
\]
At the designated times, Lemma~\ref{lem:cosine_capacity_two_factor} yields
\begin{equation}
  \ell_t(g)-\ell_t(g^{\mathrm{Bayes}})
  \geq (1-o(1))\EE\|\hat\beta_t^\mathrm{Bayes}\|^2.
  \label{eq:case2b_energy_step}
\end{equation}

Recall the expected Bayes estimation error $e_j^{\mathrm{Bayes}}$
from~\eqref{eq:vector_bayes_baseline}.  Total variance gives
\begin{equation}
  \EE\|\hat\beta_t^\mathrm{Bayes}\|^2
  =\Lambda-e_{N_t^\rho}^{\mathrm{Bayes}},
  \qquad
  \Lambda=\operatorname{tr}\widetilde\Sigma_\theta>0.
  \label{eq:case2b_belief_energy}
\end{equation}
The Bayes error estimate~\eqref{eq:vector_bayes_baseline_fixed_p} yields
\begin{equation}
  \sum_{j=0}^{T_\rho-1}e_j^{\mathrm{Bayes}}
  =p\sigma^2\log T_\rho+O(1).
  \label{eq:case2b_posterior_trace}
\end{equation}
The sum over any subset of matched times is no larger.  Summing
\eqref{eq:case2b_energy_step} over $\Omega(T_\rho)$ designated times therefore
gives $\Lambda\Omega(T_\rho)-O(\log T_\rho)=\Omega(T_\rho)$. Notice that no factor $pR$
appears: each key presents only one target direction, the
projection of its candidate vector along the current Bayes belief.
\end{proof}

\subsubsection{Bounds on addressing capacity}
\label{app:case2b_capacity_geometry}

To apply Lemma~\ref{lem:cosine_capacity_to_regret}, we develop capacity bounds
for the architectural restrictions established in
\S\ref{app:case2b_capacity_architectures}.
Lemma~\ref{lem:cosine_capacity_euclidean_certificate} provides the common
geometric reduction, used for linear spaces in
Corollary~\ref{cor:cosine_capacity_subspace_certificate} and polynomial
images in Lemma~\ref{lem:cosine_capacity_fixed_degree_certificate}.
Lemma~\ref{lem:cosine_capacity_conditional_fixed_degree_certificate}
supplies the conditional extension needed for selective SSMs.

\begin{lemma}[Geometric representation of capacity]
\label{lem:cosine_capacity_euclidean_certificate}
Let $G_t:=(g_{t,q})_{q\in\mathcal U_t}$ be the prediction vector.
For each $q\in\mathcal U_t$, define the optimal conditional least-squares
rescaling of $g_{t,q}$ toward $Z_q$ by
\[
  \omega_{t,q}^\star
  :=
  \frac{\EE[Z_qg_{t,q}\mid\mathcal C_t]}
       {\EE[g_{t,q}^2\mid\mathcal C_t]}.
\]
Set $\widetilde G_t:=(\omega_{t,q}^\star g_{t,q})_{q\in\mathcal U_t}$,
the rescaled prediction vector, and
$Z:=(Z_q)_{q\in\mathcal U_t}$.  Then
\begin{equation}
  \mathrm{Cap}_t(g)
  =
  N_t-\EE[\|Z-\widetilde G_t\|^2\mid\mathcal C_t].
  \label{eq:cosine_capacity_projection_identity}
\end{equation}
Consequently, if $\widetilde G_t$ lies almost surely in a
$\mathcal C_t$-measurable set $\mathcal I_t\subseteq\RR^{N_t}$ containing
$0$, then
\begin{equation}
  \mathrm{Cap}_t(g)
  \leq
  N_t-\EE[\operatorname{dist}^2(Z,\mathcal I_t)\mid\mathcal C_t].
  \label{eq:cosine_capacity_distance_certificate}
\end{equation}
\end{lemma}

\begin{proof}
Conditional on $\mathcal C_t$, each $Z_q$ has unit second moment and
$\omega_{t,q}^\star g_{t,q}$ is its $L^2$ projection onto the line spanned
by $g_{t,q}$.  Hence
\[
  \EE[(Z_q-\omega_{t,q}^\star g_{t,q})^2\mid\mathcal C_t]
  =1-\operatorname{corr}_{\mathcal C_t}^2(g_{t,q},Z_q).
\]
Summing proves~\eqref{eq:cosine_capacity_projection_identity}.  The distance
inequality follows pointwise from $\widetilde G_t\in\mathcal I_t$.
\end{proof}

\begin{corollary}[Linear-subspace capacity bound]
\label{cor:cosine_capacity_subspace_certificate}
Suppose $G_t$ lies almost surely in a
$\mathcal C_t$-measurable linear subspace $V_t\subseteq\RR^{N_t}$ of
dimension at most $K$.  Then $\mathrm{Cap}_t(g)\leq K$.
\end{corollary}

\begin{proof}
Let $\Omega_t:=\operatorname{diag}(\omega_{t,q}^\star:q\in\mathcal U_t)$.  Then
$\widetilde G_t=\Omega_tG_t$ lies in the subspace $\Omega_tV_t$, whose
dimension is at most $K$.  For a standard Gaussian $Z$ and a fixed
subspace $V$,
\[
  N_t-\EE[\operatorname{dist}^2(Z,V)\mid\mathcal C_t]=\dim V.
\]
Apply Lemma~\ref{lem:cosine_capacity_euclidean_certificate}.
\end{proof}

\begin{lemma}[Fixed-degree polynomial-image capacity bound]
\label{lem:cosine_capacity_fixed_degree_certificate}
Fix an integer $k\geq1$.  Suppose $G_t$ lies almost surely in the image of a
$\mathcal C_t$-measurable polynomial map
$Q_t:\RR^m\to\RR^{N_t}$ of degree at most $k$.  Then
\[
  \mathrm{Cap}_t(g)
  \leq C_k(m+1)\log(eN_t),
\]
where $C_k$ depends only on $k$.
\end{lemma}

\begin{proof}
Coordinatewise multiplication by $\Omega_t$ preserves both degree and
parameter dimension, so $\widetilde G_t$ lies in the image of the
polynomial map $\Omega_t Q_t$ of degree at most $k$. By the projection identity in
Lemma~\ref{lem:cosine_capacity_euclidean_certificate}, it suffices to prove that every
polynomial map $Q:\RR^m\to\RR^N$ of degree at most $k$ satisfies
\begin{equation}
  N-\EE\operatorname{dist}^2(Z,Q(\RR^m))
  \leq C_k(m+1)\log(eN),
  \qquad Z\sim\mathcal N(0,I_N).
  \label{eq:fixed_degree_image_distance_bound}
\end{equation}

Write $Q_i(x)=\sum_{|\alpha|\leq k}c_{i,\alpha}x^\alpha$ and define
\[
  \widetilde Q_i(s,x)
  :=\sum_{|\alpha|\leq k}
    c_{i,\alpha}s^{k-|\alpha|}x^\alpha.
\]
Then $\widetilde Q:\RR^{m+1}\to\RR^N$ is homogeneous of degree $k$ and
$\widetilde Q(1,x)=Q(x)$.  Let
$\mathcal D:=\overline{\widetilde Q(\RR^{m+1})}$.  This is a closed cone
under nonnegative scaling because
$c\widetilde Q(y)=\widetilde Q(c^{1/k}y)$ for $c\geq0$, and it contains
$Q(\RR^m)$.  If
$\mathcal D=\{0\}$, the claim is immediate.  Otherwise, with
$\mathcal S:=\mathcal D\cap\mathbb{S}^{N-1}$, optimization along each ray gives
\[
  N-\EE\operatorname{dist}^2(Z,\mathcal D)
  =\EE\sup_{u\in\mathcal S}\langle Z,u\rangle_+^2.
\]
Let $\pi(x):=x/\|x\|$ for $x\neq0$ and put
\[
  \mathcal V
  :=\pi\!\left(\widetilde Q(\RR^{m+1})\setminus\{0\}\right).
\]
If $u\in\mathcal S$, take $y_j$ with $\widetilde Q(y_j)\to u$.
Homogeneity gives
\[
  \widetilde Q\!\left(
    \frac{y_j}{\|\widetilde Q(y_j)\|^{1/k}}
  \right)
  =\frac{\widetilde Q(y_j)}{\|\widetilde Q(y_j)\|}
  \longrightarrow u.
\]
Conversely $\mathcal V\subseteq\mathcal S$, so
$\mathcal S=\overline{\mathcal V}$.

If $m+1>N$, then the desired estimate follows from the trivial bound
\[
  N-\EE\operatorname{dist}^2(Z,Q(\RR^m))
  \leq N
  \leq(m+1)\log(eN).
\]
Hence assume $m+1\leq N$.  By projectivized polynomial-image regularity
\citep[Lemma~2.15(c)]{zhangCoveringNumberReal2025}, $\mathcal V$ is
$((4k+1)^{m+2},m+1)$-regular.  Since
$\mathcal V\subseteq\mathbb{S}^{N-1}\subseteq[-1,1]^N$, the regular-set covering
theorem \citep[Theorem~2.6]{zhangCoveringNumberReal2025} gives
\[
  \log\mathcal N(\mathcal V,\varepsilon)
  \leq C(m+1)
    \log\!\left(\frac{C(4k+1)N}{\varepsilon}\right)
\]
for $0<\varepsilon\leq\operatorname{diam}(\mathcal V)$; larger radii are
trivial.  Here and below $C$ is universal.  Passing to the closure changes
the covering radius by at most a constant factor, and adjoining the negative
set multiplies the covering number by at most two.  Thus, for
$\mathcal S_\pm:=\mathcal S\cup(-\mathcal S)$,
\[
  \log\mathcal N(\mathcal S_\pm,\varepsilon)
  \leq C(m+1)
    \log\!\left(\frac{C(4k+1)N}{\varepsilon}\right).
\]
Dudley's entropy integral therefore gives
\[
  \EE\sup_{u\in\mathcal S_\pm}\langle Z,u\rangle
  \leq C_k\sqrt{(m+1)\log(eN)}.
\]
The supremum is a $1$-Lipschitz function of $Z$, so Gaussian concentration
upgrades this to the second-moment bound
\[
  \EE\!\left[
    \left(\sup_{u\in\mathcal S_\pm}\langle Z,u\rangle\right)^2
  \right]
  \leq C_k(m+1)\log(eN).
\]
This bounds the preceding positive-part supremum.  Enlarging
$Q(\RR^m)$ to $\mathcal D$ can only decrease distance, proving
\eqref{eq:fixed_degree_image_distance_bound}.
\end{proof}

The convolutional input to a selective transition can contain a
linear summary of the candidate coordinates.  The next lemma accounts for
those directions before applying the polynomial-image bound to the remaining
Gaussian coordinates.

\begin{lemma}[Polynomial-image capacity bound with linear side information]
\label{lem:cosine_capacity_conditional_fixed_degree_certificate}
Fix an integer $k\geq1$.  Conditionally on $\mathcal C_t$, let
$S_t=W_tZ$, where the matrix $W_t$ is $\mathcal C_t$-measurable and has rank
at most $r_{\mathrm{side}}$.  Suppose that, conditional on
$(\mathcal C_t,S_t)$, $G_t$ lies in the image of a
$(\mathcal C_t,S_t)$-measurable polynomial map
$Q_{t,S_t}:\RR^m\to\RR^{N_t}$ of degree at most $k$.  Then
\[
  \mathrm{Cap}_t(g)
  \leq r_{\mathrm{side}}+C_k(m+1)\log(eN_t).
\]
\end{lemma}

\begin{proof}
Let $r_t^{\mathrm{side}}:=\operatorname{rank}(W_t)$.  If
$r_t^{\mathrm{side}}=N_t$ the claim is immediate.  Otherwise, conditional
on $\mathcal C_t$, apply a fixed Borel-measurable orthonormal-frame rule on
each rank stratum to choose a $\mathcal C_t$-measurable isometry
$V_t:\RR^{N_t-r_t^{\mathrm{side}}}\to\ker W_t$.  Then
$Z_t^\circ:=V_t^\top Z\sim\mathcal N(0,I_{N_t-r_t^{\mathrm{side}}})$
independently of $S_t$.  The conditional moments defining $\Omega_t$ are
$\mathcal C_t$-measurable, so for every realized $S_t=s$, the map
$V_t^\top\Omega_t Q_{t,s}$ has degree at most $k$.  Hence
\begin{align*}
  \|Z-\Omega_tG_t\|^2
  &\geq\|V_t^\top(Z-\Omega_tG_t)\|^2\\
  &\geq
  \operatorname{dist}^2\!\left(
    Z_t^\circ,V_t^\top\Omega_t Q_{t,S_t}(\RR^m)
  \right).
\end{align*}
Condition on $S_t$, apply~\eqref{eq:fixed_degree_image_distance_bound}, and
then average over $S_t$.  This gives
\[
  \EE[\|Z-\Omega_tG_t\|^2\mid\mathcal C_t]
  \geq
  N_t-r_t^{\mathrm{side}}-C_k(m+1)\log(eN_t).
\]
Use~\eqref{eq:cosine_capacity_projection_identity} and
$r_t^{\mathrm{side}}\leq r_{\mathrm{side}}$.
\end{proof}

\subsubsection{Architecture-specific geometry}
\label{app:case2b_capacity_architectures}

We now establish the geometric restrictions needed to apply the bounds in
\S\ref{app:case2b_capacity_geometry} to our architecture classes
(\S\ref{sec:architectures}).
Lemma~\ref{lem:case2b_architecture_output_geometry} places the predictions of
linear attention and fixed-transition or affine-gated SSMs in low-dimensional
subspaces. For one-layer selective SSMs, the same lemma places the prediction vector
in the image of a polynomial map after conditioning on a linear summary
of the candidates, allowing us to apply
Lemma~\ref{lem:cosine_capacity_conditional_fixed_degree_certificate}.
Lemma~\ref{lem:case2b_selective_fixed_prediction_image} extends the selective
argument to our fixed-depth class.
Our bounds also cover the count-normalized extensions of
Appendix~\ref{app:count_normalized_extensions}, denoted by the superscript
$\mathrm{cal}$.

\Needspace{25\baselineskip}
\begin{lemma}[Architecture-specific geometry]
\label{lem:case2b_architecture_output_geometry}
Fix $t\geq R$ and condition on $\mathcal C_t$. The following statements
hold for every predictor in the indicated classes, including their
count-normalized extensions with any admissible schedule.
\begin{enumerate}[label=\textnormal{(\alph*)},ref=\alph*,leftmargin=*,itemsep=.5em]
  \item \textnormal{Linear attention and fixed-transition SSMs.}
  \label{item:case2b_geometry_linear}
  For a predictor in $\mathcal F_{\mathrm{att}}(D,d,H,L)$ or
  $\mathcal F_{\mathrm{SSM}}(D,d,H,\dconv)$, the prediction vector $G_t$ lies in a
  $\mathcal C_t$-measurable linear subspace of dimension at most
  \[
    \binom{d_{\mathrm{key}}+3^D}{3^D}.
  \]
  \item \textnormal{Affine-gated SSMs.}
  \label{item:case2b_geometry_affine}
  For a predictor in $\mathcal F_{\mathrm{SSM}}^{\mathrm{aff}}(D,d,H,\dconv)$,
  the prediction vector $G_t$ lies in a
  $\mathcal C_t$-measurable linear subspace of dimension at most
  \[
    \binom{d_{\mathrm{key}}+4^D}{4^D}.
  \]
  \item \textnormal{One-layer selective SSMs.}
  \label{item:case2b_geometry_selective}
  For a predictor in $\mathcal F_{\mathrm{SSM}}^{\mathrm{sel}}(1,d,H,\dconv)$,
  there is a linear summary of the candidate coordinates, $S_t=W_tZ$,
  with $W_t$ measurable with respect to $\mathcal C_t$.
  Conditional on $(\mathcal C_t,S_t)$, the prediction vector $G_t$ lies
  in the image of a $(\mathcal C_t,S_t)$-measurable polynomial map
  $Q_t:\RR^m\to\RR^{N_t}$ of degree at most four, with
  \[
    \operatorname{rank}W_t\leq Hd+1,
    \qquad m\leq6Hd.
  \]
\end{enumerate}
The subspace dimension bounds in
\textnormal{(\ref{item:case2b_geometry_linear})--(\ref{item:case2b_geometry_affine})}
depend only on $D$ and $d_{\mathrm{key}}$; the rank and dimension bounds in
\textnormal{(\ref{item:case2b_geometry_selective})} are
independent of $d_e$, $\dconv$, and $p$.
\end{lemma}

\begin{proof}
For attention, replace only the upcoming-key block at the current position by
a variable $\chi\in\RR^{d_{\mathrm{key}}}$.  Causality makes every earlier
position independent of $\chi$.  If the current representation has degree
$\delta_r$ in $\chi$ after layer $r$, a linear-attention cross term has degree at
most $\delta_r$, while the current-position self term --- a bilinear score
times a linear value --- has degree at most $3\delta_r$.  Residuals, head
recombination, and the final linear readout do not increase degree.  Since
$\delta_0=1$, induction gives $\delta_D\leq3^D$.  Evaluating the polynomial on the
$N_t$ public codewords places $G_t$ in the fixed span of their
degree-$\leq3^D$ monomial-evaluation vectors.  Candidate vectors, including
the state they induced earlier, affect only the coefficients of this
polynomial and not that span.

For the fixed-transition and affine-gated SSM bounds, make the same substitution in the
current token.  Every earlier residual token and pre-update state is
independent of $\chi$ by causality, although it may depend arbitrarily on
$Z$.  Suppose a layer input has coordinatewise degree at most $m$ in $\chi$.
Its convolution has the same degree and its bilinear write has degree at most
$2m$.  With a fixed transition, the updated state has degree at most $2m$,
so the ordinary read has degree at most $3m$.  In the affine-gate ablation,
the delta term $\beta_tk_t(k_t^\top h_{t-1})$ can have degree $3m$; uniformly
over all four transition forms, the ordinary read therefore has degree at
most $4m$.  Head recombination and residual addition do not
  increase degree.  Starting from degree one and evaluating on the codebook
  proves the SSM claims in parts~(\ref{item:case2b_geometry_linear})
  and~(\ref{item:case2b_geometry_affine}).

The deterministic count-normalization factors are fixed scalars after conditioning
on the prediction time.  They do not change any of these degree bounds, for
attention or for the SSM.

It remains to track only the candidate projections exposed to nonlinear
gates.  Conditional on $\mathcal C_t$, each head-specific first-layer
convolved token has a decomposition
\[
  \widetilde e_{t,q}^{(1,b)}
  =\mu_{t,q}^{(b)}+B_t^{(b)}Z,
\]
where $\mu_{t,q}^{(b)}$ and $B_t^{(b)}$ are
$\mathcal C_t$-measurable.  The rank of the raw map $B_t^{(b)}$ can grow with
the ambient and convolution widths, so we do not condition on the complete
convolved token.  Instead, for each head retain only the candidate-dependent
part of the affine preactivation to its sigmoid gate.  This has dimension
zero for a fixed transition, one for a scalar gate or delta strength, and at
most $d$ for a diagonal gate or generalized delta transition.  Stacking the
heads gives a linear summary of $Z$ of rank at most $Hd$.

If $t+1$ begins a cycle, let $q_t$ be the current key and append $Z_{q_t}$ to
this summary.  Its candidate predecessor is the current regressor, so this
additional coordinate fixes the candidate scalar carried directly by the
current residual token.  Away from a cycle boundary no candidate vector
occurs in that token.  Thus the total side-information rank is at most
$Hd+1$.

Conditional on this side information, every first-layer sigmoid value is a
fixed coefficient for each counterfactual key.  Treat as free variables the
$Hd$ pre-update state coordinates and the candidate-dependent parts of the
three affine write/read projections, contributing $3Hd$ coordinates.  A
delta transition additionally requires the $Hd$ candidate-dependent
direction coordinates, while a generalized delta transition requires at
most $2Hd$ coordinates for its two affine directions.  Hence at most $6Hd$
free coordinates suffice uniformly over the permitted transition menu.
The bilinear write has degree at most two.  Fixed, scalar-gated, and
diagonal-gated transitions preserve this degree before the ordinary read,
while the delta and generalized delta terms have degree at most three in the
free coordinates.  The ordinary token--state read raises the uniform degree
to at most four.  Residual addition, head recombination, deterministic count
normalization, and final prediction do not increase it.  Collecting the
coordinate polynomials over $q\in\mathcal U_t$ proves
part~(\ref{item:case2b_geometry_selective}).
\end{proof}

The one-layer argument conditions only on the candidate projections that
feed sigmoid gates, rather than on the entire convolved token.  When
selectivity appears only in the first layer and every later transition is
fixed, the same argument propagates through the entire stack:

\begin{lemma}[Geometry of selective SSM stacks]
\label{lem:case2b_selective_fixed_prediction_image}
Fix $D\geq1$ and consider
$\mathcal F_{\mathrm{SSM}}^{\mathrm{sel}}(D,d,H,\dconv)$,
for arbitrary ambient width $d_e$ and maximum convolution width $\dconv$,
including its count-normalized extensions with any admissible schedule.
The first layer
uses an arbitrary permitted convolution width and any permitted transition,
while every later layer uses an arbitrary permitted convolution width and a
fixed transition.  Its writes are bilinear and its reads are
ordinary token--state reads.  Its $D$ layers contain $DHd$ recurrent
coordinates. There is a linear summary of the candidate coordinates,
$S_t=W_tZ$, with $W_t$ measurable with respect to $\mathcal C_t$ and
\[
  \operatorname{rank}W_t
  \leq Hd+1
\]
such that, conditionally on $(\mathcal C_t,S_t)$, $G_t$ lies in the image of
a degree-$\leq4\cdot3^{D-1}$ polynomial map
$Q_t:\RR^m\to\RR^{N_t}$ with $m\leq(4D+2)Hd$.
\end{lemma}

\begin{proof}
Use the projection-specific side information constructed in
Lemma~\ref{lem:case2b_architecture_output_geometry}(\ref{item:case2b_geometry_selective})
for the first layer.
Conditional on $(\mathcal C_t,S_t)$, all of its sigmoid values are fixed
coefficients for each counterfactual key, and the direct current-token
residual is fixed even at a cycle boundary.

Retain as free variables all $DHd$ pre-update state coordinates.  From the
candidate-dependent part of the first-layer convolved tokens, retain the
three $Hd$-dimensional projections entering the bilinear write and ordinary
read.  For delta transitions also retain the $Hd$ projected direction
coordinates; for generalized delta transitions retain the at most $2Hd$
coordinates entering their two affine directions.  The scalar and diagonal
transitions require no direction variables.  Thus the first layer contributes
at most $5Hd$ projected coordinates beyond its pre-update state.

For each later-layer head $(r,b)$, let
\[
  \eta_t^{(r,b)}
  :=\sum_{j=1}^{\dconv-1}
    C_j^{(r,b)}\odot e_{t-j}^{(r-1)},
  \qquad r=2,\ldots,D,
\]
where an empty sum is zero.  This is the entire lagged part of the
convolution and is shared across the counterfactual upcoming keys.  Retain as
free parameters
\[
  \bar W_K^{(r,b)}\eta_t^{(r,b)},
  \qquad
  \bar W_V^{(r,b)}\eta_t^{(r,b)},
  \qquad
  W_Q^{(r,b)}\eta_t^{(r,b)},
  \qquad r=2,\ldots,D,
\]
where bars denote the linear parts of the affine write maps.  Together with
the preceding first-layer projections and all pre-update states, these form a
vector $\vartheta_t$ of dimension at most
\[
  DHd+5Hd+3(D-1)Hd=(4D+2)Hd.
\]
Its realized value may depend
arbitrarily and nonlinearly on all candidate vectors; treating it as free
only enlarges the set of possible prediction vectors.

The first-layer bilinear write has degree at most two in $\vartheta_t$.  A
fixed, scalar-gated, or diagonal-gated transition leaves the updated state at
degree at most two, whereas a delta or generalized delta term has degree at
most three.  Multiplication by the affine ordinary-read selector therefore
gives $e_{t,q}^{(1)}$ degree at most four, uniformly over the permitted
transition menu.  Put $\delta_1:=4$ and suppose the current representation
after layer $r-1$ has degree at most $\delta_{r-1}$, where $r\geq2$.
The current-token part of each layer-$r$ convolution is linear in
$e_{t,q}^{(r-1)}$, while its lagged part enters the write factors and read
selector only through the three retained projections above.  Each such factor
is therefore a polynomial of degree at most $\delta_{r-1}$.  The bilinear
write has degree at most $2\delta_{r-1}$, as does the updated state under the
fixed transition.  The ordinary token--state read consequently has degree at
most $3\delta_{r-1}$, and residual addition and head recombination do not
increase it.  Hence
$\delta_r\leq3\delta_{r-1}\leq4\cdot3^{r-1}$.
Deterministic count normalization and the final prediction map also preserve
degree.  Collecting the coordinate polynomials over $q\in\mathcal U_t$ proves the
claim.  Later convolution widths affect the realized lag aggregates, but
neither their three projected dimensions per head nor the degree bound.
\end{proof}

\paragraph{Scope of the selective-SSM bound.} Our class $\mathcal F_{\mathrm{SSM}}^{\mathrm{sel}}$ permits selectivity
only in the first layer. This lets us fix sigmoid gate outputs by
conditioning on linear candidate projections, preserving Gaussianity
in the remaining directions. The inputs to deeper sigmoid gates generally depend nonlinearly
on the candidates: conditioning on them need not preserve Gaussianity, while
leaving the gates unconditioned generally breaks the polynomial structure
used above. Extending the bound to fully selective stacks is therefore
beyond the present argument.

\subsubsection{Proof of Theorem~\ref{thm:case2b_floors}}
\label{app:case2b_capacity_theorem}

We now combine the preceding bounds to prove Theorem~\ref{thm:case2b_floors}.
Lemma~\ref{lem:case2b_architecture_output_geometry} and
Corollary~\ref{cor:cosine_capacity_subspace_certificate} yield the attention
and affine-gate bounds in (a)--(b), while
Lemmas~\ref{lem:cosine_capacity_conditional_fixed_degree_certificate} and
\ref{lem:case2b_selective_fixed_prediction_image} yield the selective-SSM
bound in (c). Lemma~\ref{lem:cosine_capacity_to_regret} then gives the regret
floors.
We use the main-text class notation below; all bounds also hold
uniformly for the count-normalized extensions of
Appendix~\ref{app:count_normalized_extensions}, under every admissible schedule.

\begin{lemma}[Fixed-depth addressing floors]
\label{lem:case2b_fixed_depth_addressing_floors}
Let $d_{\mathrm{key}}=\Theta(\log R)$ and fix $D$ independently of $R$.
For every choice of $d_e$, $d$, $H$, $L$, $\dconv$, and $t\geq R$, every
predictor $g$ in $\mathcal F_{\mathrm{att}}(D,d,H,L)$ or
$\mathcal F_{\mathrm{SSM}}(D,d,H,\dconv)$ obeys
\begin{equation}
  \mathrm{Cap}_t(g)
  \leq\min\left\{N_t,
    \binom{d_{\mathrm{key}}+3^D}{3^D}\right\}.
  \label{eq:case2b_fixed_ssm_capacity_bound}
\end{equation}
Every predictor $g\in\mathcal F_{\mathrm{SSM}}^{\mathrm{aff}}(D,d,H,\dconv)$
with the same dimensions obeys
\begin{equation}
  \mathrm{Cap}_t(g)
  \leq\min\left\{N_t,
    \binom{d_{\mathrm{key}}+4^D}{4^D}\right\}.
  \label{eq:case2b_affine_gate_capacity_bound}
\end{equation}
Consequently, in the main-text regime $R\to\infty$,
$T_\rho\to\infty$, and $R=o(\sqrt{T_\rho})$, every such fixed-depth
predictor has $\Omega(T_\rho)$ regret, uniformly over $d_e$, $d$, $H$,
$L$, $\dconv$, and parameter choices.
\end{lemma}

\begin{proof}
Lemma~\ref{lem:case2b_architecture_output_geometry} and
Corollary~\ref{cor:cosine_capacity_subspace_certificate} give the displayed
bounds.  Each degree is constant at fixed $D$, so every binomial term is
$\operatorname{polylog}(R)=o(R)$.  The first half of each complete
post-initial cycle has $N_t\geq R/2$; these account for
$T_\rho/2-O(R)=\Omega(T_\rho)$ prediction times. Lemma
\ref{lem:cosine_capacity_to_regret} gives the regret floor.
\end{proof}

Before proving the fixed-depth state requirement in
Lemma~\ref{lem:case2b_selective_fixed_state_floor}, we illustrate the
argument for one-layer selective SSMs.
Lemma~\ref{lem:case2b_architecture_output_geometry}(\ref{item:case2b_geometry_selective}) and
Lemma~\ref{lem:cosine_capacity_conditional_fixed_degree_certificate} with
$k=4$ give
\begin{equation}
  \mathrm{Cap}_t(g)
  \leq Hd+1+C_4(6Hd+1)\log(eN_t)
  \leq C(Hd+1)\log(eN_t).
  \label{eq:case2b_selective_capacity_bound}
\end{equation}
Here $C$ is universal, and the additive one is needed only at cycle
boundaries. The bound holds uniformly over the suppressed ambient width $d_e$.
The next lemma extends this argument to our fixed-depth selective
SSM class, expressing the bound in terms of the total recurrent-state
dimension $DHd$.

\begin{lemma}[Fixed-depth selective SSM state floor]
\label{lem:case2b_selective_fixed_state_floor}
Fix $D$ independently of $R$ and consider the depth-$D$ class in
Lemma~\ref{lem:case2b_selective_fixed_prediction_image}.  Uniformly over the
public codebook, ambient width $d_e$,
maximum convolution width $\dconv$, all other resources, and all parameter
choices,
\begin{equation}
  \mathrm{Cap}_t(g)
  \leq C_D\,DHd\log(eN_t),
  \qquad t\geq R,
  \label{eq:case2b_selective_fixed_capacity_bound}
\end{equation}
where $C_D$ depends only on $D$.  In the regime $R\to\infty$,
$T_\rho\to\infty$, and $R=o(\sqrt{T_\rho})$,
\begin{equation}
  DHd=o(R/\log R)
  \quad\Longrightarrow\quad
  \mathrm{Reg}_T(g)=\Omega(T_\rho).
  \label{eq:case2b_selective_fixed_state_floor}
\end{equation}
Consequently, every predictor sequence in this fixed-depth class with
$o(T_\rho)$ regret satisfies $DHd=\Omega(R/\log R)$.
\end{lemma}

\begin{proof}
Lemma~\ref{lem:case2b_selective_fixed_prediction_image} gives degree at most
$4\cdot3^{D-1}$, parameter dimension at most $(4D+2)Hd$, and linear
side-information rank at most $Hd+1$.  Lemma
\ref{lem:cosine_capacity_conditional_fixed_degree_certificate} with
$k=4\cdot3^{D-1}$ therefore gives
\[
  \mathrm{Cap}_t(g)
  \leq Hd+1
    +C_{4\cdot3^{D-1}}\bigl((4D+2)Hd+1\bigr)\log(eN_t)
  \leq C_D\,DHd\log(eN_t),
\]
which proves~\eqref{eq:case2b_selective_fixed_capacity_bound}.

Now restrict to times for which $t+1$ lies in positions
$2,\ldots,\lfloor R/2\rfloor$ of a complete post-initial cycle.  There are
$\Omega(T_\rho)$ such times, none is a cycle boundary, and $N_t\geq R/2$ on all
of them.  If $DHd=o(R/\log R)$, then the capacity bound is
$o(R)$ uniformly on these times.  Lemma
\ref{lem:cosine_capacity_to_regret} proves
\eqref{eq:case2b_selective_fixed_state_floor}.

Finally, if an $o(T_\rho)$-regret sequence failed the stated
$\Omega(R/\log R)$ necessity, some subsequence would satisfy
$DHd=o(R/\log R)$, contradicting
\eqref{eq:case2b_selective_fixed_state_floor}.
\end{proof}

\begin{proof}[Proof of Theorem~\ref{thm:case2b_floors}]
At fixed depth, the capacity bounds in
Lemma~\ref{lem:case2b_fixed_depth_addressing_floors} are $o(R)$,
proving parts (a)--(b). Since $N_t\leq R$,
Lemma~\ref{lem:case2b_selective_fixed_state_floor} gives part (c).
Their regret conclusions complete the proof.
\end{proof}

\subsubsection{Constructions for Proposition~\ref{prop:case2b_pair}}
\label{app:case2b_vector_witnesses}

This section proves Proposition~\ref{prop:case2b_pair}. Each construction
first retrieves the required regressors, then applies its Case~1 kernel for
belief maintenance. Proposition~\ref{prop:case2b_vector_witnesses} gives
their leading regret constants and sufficient dimensions.

Recall the Case~1 notation from
\eqref{eq:vector_covariance_scales}--\eqref{eq:vector_common_window}:
\[
  \begin{gathered}
    v_i=\lambda_i+\operatorname{tr}\widetilde\Sigma_\theta+\sigma^2,
    \qquad \Lambda=\sum_i\lambda_i,
    \qquad V=\sum_i v_i,\\
    \zeta=\sum_i\sqrt{\lambda_iv_i},
    \qquad \bar\zeta=\sqrt{\Lambda V}.
  \end{gathered}
\]
We round $L^\star(T_\rho)$ to the nearest integer when choosing the window.

\begin{proposition}[Content-routing pair]
\label{prop:case2b_vector_witnesses}
Assume $R=o(\sqrt{T_\rho})$ and keep
$p$, $\widetilde\Sigma_\theta\succ0$, and $\sigma^2>0$ fixed. The following
depth-two predictors can be realized with ambient width
$d_e=2d_{\mathrm{key}}+4p+2$.
\begin{enumerate}[label=\textnormal{(\alph*)},ref=\alph*,leftmargin=*,itemsep=.5em]
  \item \textnormal{Attention.}
  \label{item:case2b_construction_attention}
  A predictor $g_{\mathrm{att}}\in\mathcal F_{\mathrm{att}}^{\mathrm{sm}}
  (2,d_{\mathrm{key}}\vee p,2,L)$ with
  $L=L^\star(T_\rho)+O(1)$ uses softmax in its first layer and linear
  attention in its second, and satisfies
  \[
    \mathrm{Reg}_T(g_{\mathrm{att}})
    =\frac{2}{\sqrt3}\bar\zeta\sqrt{T_\rho}+o(\sqrt{T_\rho}).
  \]
  \item \textnormal{Selective SSM.}
  \label{item:case2b_construction_ssm}
  A predictor $g_{\mathrm{SSM}}\in\mathcal F_{\mathrm{SSM}}^{\mathrm{sel}}
  (2,2R,p,2)$ uses $4pR$ recurrent-state coordinates and satisfies
  \[
    \mathrm{Reg}_T(g_{\mathrm{SSM}})
    =\zeta\sqrt{T_\rho}+o(\sqrt{T_\rho}).
  \]
\end{enumerate}
\end{proposition}

We establish the attention and selective-SSM claims in turn.

\paragraph{Attention construction.}

The first layer uses two softmax heads to retrieve the regressors for
evidence assembly and addressing. The second layer uses one active
linear-attention head to apply the Case~1 uniform kernel; its other head
is zero. We bound the error from finite softmax scores before deriving the regret.

\begin{proof}[Proof of Proposition~\ref{prop:case2b_vector_witnesses}(\ref{item:case2b_construction_attention})]
At position $i$, the first layer's atom head uses
\begin{equation}
  s^{\mathrm a}_{ij}
  :=\beta_{\mathrm{sm}}\langle\psi_i,\psi_j\rangle-\gamma(i-j)
    +p^{\mathrm a}_{i-j},
  \qquad
  p^{\mathrm a}_0=-M_0,
  \quad p^{\mathrm a}_u=0\ (u>0),
  \label{eq:case2b_atom_logits}
\end{equation}
and copies $\xi_j$. The lag penalty $-\gamma(i-j)$ favors more recent
occurrences of the same key. The finite lag-zero penalty prevents
the current token, which has the same key and zero recency cost, from winning.
The query head uses
\begin{equation}
  s^{\mathrm q}_{ij}
  :=\beta_{\mathrm{sm}}\langle\psi_{i+1},\psi_j\rangle-\gamma(i-j)
  \label{eq:case2b_query_logits}
\end{equation}
and also copies $\xi_j$. Denote the routed residual
coordinates by $\widetilde u_i$ and $\widetilde q_i$, respectively.

Let
$\eta_0:=\max_{r\neq s}|\langle\psi^{(r)},\psi^{(s)}\rangle|<1/3$.
For any fixed $a>0$, choose finite parameters satisfying
\begin{equation}
  \gamma\geq(1+a)\log T,
  \qquad
  \beta_{\mathrm{sm}}(1-\eta_0)-2\gamma R\geq(1+a)\log(LT),
  \qquad
  M_0-2\gamma R\geq(1+a)\log(LT).
  \label{eq:case2b_softmax_margins}
\end{equation}
The desired previous occurrence is at lag at most $2R-1$, so it lies in the
window whenever $L\geq2R$.  If $\delta_i$ is the mass outside the desired
source in either head, the same-key recency gap and wrong-key content gap
give uniformly
\begin{equation}
  \delta_i
  \leq
  \bar\delta_T
  :=
  \frac{e^{-\gamma}}{1-e^{-\gamma}}
  +L\exp\!\left[-\{\beta_{\mathrm{sm}}(1-\eta_0)-2\gamma R\}\right]
  +\exp[-(M_0-2\gamma R)].
  \label{eq:case2b_softmax_leakage}
\end{equation}
The parameters in~\eqref{eq:case2b_softmax_margins} may be increased so that
$T^m\bar\delta_T\to0$ for any prescribed fixed $m$.

The second layer uses linear attention with zero recency slope and its own
parameters. Its active head uses score and value
\begin{equation}
  s^{(2)}_{ts}
  :=\frac1L\widetilde q_t^\top\widetilde u_s,
  \qquad
  W_V^{(2)}e_s^{(1)}=y_s e_{\mathrm{acc}}.
  \label{eq:case2b_vector_contraction}
\end{equation}
Here $e_{\mathrm{acc}}\in\RR^d$ is a standard basis vector in the head's
value space. The output projection $W_O^{(2)}$ maps it to the residual
accumulator coordinate, which is zero after the first layer and selected
by $W_{\mathrm{pred}}$.
Under exact routing, its scalar output is
\begin{equation}
  g_t^{\mathrm{exact}}
  =\xi_{\rho(t+1)}^\top
    \left(\frac1L\sum_{s\in I_t}\xi_{\rho(s)}y_s\right).
  \label{eq:case2b_exact_attention_witness}
\end{equation}
Here
$I_t:=J_t\cap\{R+1,\ldots,t\}$ is the matched portion of the physical
length-$L$ attention window, truncated during warm-up.
The ambient width $d_e=2d_{\mathrm{key}}+4p+2$ supplies two key
blocks, the current and next transformed regressors, two routed vector blocks,
and label/accumulator coordinates.
The first-layer matching scores require $d_{\mathrm{key}}$ query--key
coordinates, while copying a regressor and computing the second-layer
inner product each require $p$ value or query--key coordinates.  Because these
maps are independent, the sufficient common per-head width is
$d_{\mathrm{key}}\vee p$.

It remains to control soft routing.  Conditional on the key stream, a
retrieved vector has the form
\[
  \widetilde X=(1-\delta)X_*+e_{\mathrm{off}},
  \qquad
  e_{\mathrm{off}}=\sum_{j\neq *}\alpha_jX_j.
\]
The off-target regressors are centered and independent of $X_*$, so
\[
  \EE[e_{\mathrm{off}}X_*^\top\mid\text{keys}]=0,
  \qquad
  \EE[\|e_{\mathrm{off}}\|^4\mid\text{keys}]
  \leq C_p\delta^4.
\]
Likewise, for a routed atom,
$y\widetilde X=(1-\delta)yX_*+ye_{\mathrm{off}}$ with
$\EE[ye_{\mathrm{off}}\mid\beta,\text{keys}]=0$.  Fixed $p,\widetilde\Sigma_\theta,\sigma^2$ give uniform
fourth moments. The uniform-kernel weights in~\eqref{eq:case2b_vector_contraction}
have total absolute mass at most one, so Minkowski's and H\"older's
inequalities bound the increase in cumulative squared loss by
$O(T\bar\delta_T)=o(\sqrt{T_\rho})$.

For $t\leq R$, the attention output and the Bayes predictor are both zero.
At $t=R+j$, $1\leq j\leq T_\rho$, the window $I_t$ contains $\min(j,L)$
matched atoms.  Conditional on the keys, their routed regressors and the
next query have exactly the Case~1 joint law.  Thus the exact-routed
regret is the Case~1 uniform-kernel regret over $T_\rho$ steps. The
kernel error identity gives the cost before subtracting the Bayes baseline:
\begin{equation}
  \frac{\Lambda L}{3}+\frac{VT_\rho}{L}+O(1).
  \label{eq:case2b_attention_leading_cost}
\end{equation}
The posterior subtraction is $O(\log T_\rho)=o(\sqrt{T_\rho})$.
Taking $L$ as the nearest integer to $L^\star(T_\rho)$ therefore gives
$\tfrac{2}{\sqrt3}\bar\zeta\sqrt{T_\rho}+o(\sqrt{T_\rho})$.
Since $R=o(\sqrt{T_\rho})$ and $L^\star(T_\rho)=\Theta(\sqrt{T_\rho})$, the reach
condition $L\geq2R$ is eventually satisfied.  Together with the
soft-routing error bound, this proves
Proposition~\ref{prop:case2b_vector_witnesses}(\ref{item:case2b_construction_attention}).
\end{proof}

\paragraph{Selective-SSM construction.}

The first layer caches regressors in two state blocks per eigencoordinate and reads
the regressor named by the upcoming key into the residual stream. The
second layer uses a width-two convolution to pair the preceding cache
read with the current label, then applies the Case~1 exponential kernel.
We verify the cache before implementing the layers and bounding regret.

The first layer uses $H=p$ heads.  Head $i$ handles the scalar
eigencoordinate $\xi_{i,t}:=e_i^\top\xi_t$ and stores
\begin{equation}
  \bar h_{i,t}
  =(C_{i,t}^{(1)},C_{i,t}^{(2)})\in\RR^{2R},
  \label{eq:case2b_per_head_state}
\end{equation}
where the two state blocks each have one coordinate per key.  The first block keeps the latest
regressor for each key relative to a common reference.  The second uses the
same writes but also resets the upcoming key, so their difference isolates
the regressor needed for prediction.  Let $\iota_t\in[R]$ index the key
$\psi_t=\psi^{(\iota_t)}$.  The codewords enter the affine gate maps as
parameters, not recurrent-state coordinates.

In the hard-gate limit, the two state blocks use the following gates, where
$1$ retains the previous coordinate value and $0$ clears it before the new write:
\begin{equation}
  \lambda^{(1)}_{t,r}=\mathbf1\{r\neq\iota_t\},
  \qquad
  \lambda^{(2)}_{t,r}=\mathbf1\{r\notin\{\iota_t,\iota_{t+1}\}\}.
  \label{eq:case2b_cache_gates}
\end{equation}
Both are limits of sigmoid gates in our class. At steepness $s$, take
\begin{equation}
  \lambda^{(1,s)}_{t,r}
  :=\sigma\!\left(s\left\{
    \frac23-\langle\psi^{(r)},\psi^{(\iota_t)}\rangle
  \right\}\right),
  \qquad
  \lambda^{(2,s)}_{t,r}
  :=\sigma\!\left(s\left\{
    \frac23-\langle\psi^{(r)},
      \psi^{(\iota_t)}+\psi^{(\iota_{t+1})}\rangle
  \right\}\right).
  \label{eq:case2b_finite_cache_gates}
\end{equation}
Each gate is the sigmoid of an affine function of the token's two key blocks.
The coherence bound gives a margin of at least
$\kappa_R:=1/3-\eta_0>0$ on either side of $2/3$, including when the two
keys coincide.  The bilinear write, with one constant affine factor,
broadcasts $u_{i,t}:=\xi_{i,t}-\xi_{i,t+1}$ to both state blocks:
\begin{align}
  C_{i,t,r}^{(1)}&=\lambda^{(1)}_{t,r}C_{i,t-1,r}^{(1)}+u_{i,t},
  \label{eq:case2b_cache_A}\\
  C_{i,t,r}^{(2)}&=\lambda^{(2)}_{t,r}C_{i,t-1,r}^{(2)}+u_{i,t},
  \label{eq:case2b_cache_B}
\end{align}
With hard gates, if key $r$ last occurred at $\ell\leq t$, telescoping gives
$C_{i,t,r}^{(1)}=\xi_{i,\ell}-\xi_{i,t+1}$.
The two state blocks agree outside the upcoming key: both reset the current key,
erasing their previous difference.  If $\iota_t\neq\iota_{t+1}$, the
upcoming slot is $\xi_{i,\rho(t+1)}-\xi_{i,t+1}$ in the first block and
$\xi_{i,t}-\xi_{i,t+1}$ in the second.  If the keys coincide, both blocks
reset the same slot and $\rho(t+1)=t$.  Thus, with
$\bar c_i^{\mathrm{cache}}:=(\mathbf1_R,-\mathbf1_R)$, the residual addition
of $\xi_{i,t}$ gives
\begin{equation}
  \chi_{i,t}:=\xi_{i,t}+\bar c_i^{\mathrm{cache}\top}\bar h_{i,t}
  =\xi_{i,\rho(t+1)}\qquad(t\geq R).
  \label{eq:case2b_cache_identity}
\end{equation}
This includes equal consecutive keys across cycle boundaries.

For finite sigmoids, steepness $C\kappa_R^{-1}\log(RT)$ makes the uniform
gate error $\delta_T$ an arbitrarily small fixed inverse power of $RT$.
Let superscripts ``soft'' and ``hard'' denote the corresponding trajectories.
With hard gates, entries in each state block are differences of two Gaussian coordinates, including
before a key's first occurrence.  Induction in
\eqref{eq:case2b_cache_A}--\eqref{eq:case2b_cache_B}, using gate values in
$[0,1]$, therefore gives uniformly for $t\leq T$
\[
  \max_r\bigl\|C_{i,t,r}^{(1),\mathrm{soft}}
                  -C_{i,t,r}^{(1),\mathrm{hard}}\bigr\|_8
  +\max_r\bigl\|C_{i,t,r}^{(2),\mathrm{soft}}
                  -C_{i,t,r}^{(2),\mathrm{hard}}\bigr\|_8
  \leq C_{p,\widetilde\Sigma_\theta,\sigma}T\delta_T.
\]
Summing these errors over the $R$ keys bounds each routed-coordinate error by
$C_{p,\widetilde\Sigma_\theta,\sigma}RT\delta_T$ in $L^8$.

\begin{proof}[Proof of Proposition~\ref{prop:case2b_vector_witnesses}(\ref{item:case2b_construction_ssm})]
Work in the whitened eigencoordinates of~\eqref{eq:case2b_whitening} and use
the linear embedding
\[
  e_t^{(0)}=(\psi_{t+1},\psi_t,\xi_{t+1},\xi_t,y_t,\xi_t,\xi_t,0)
  \in\RR^{2d_{\mathrm{key}}+4p+2}.
\]
Layer~1 uses identity convolution and the coordinatewise sigmoid transitions
of the cache recurrence above. A constant read selector and
linear output projection sum the coordinatewise differences between each
head's two state blocks and add the result to the corresponding
coordinate in the two extra $\xi_t$ blocks, leaving all other blocks unchanged.
These residual blocks therefore contain two copies of
$\chi_t=(\chi_{1,t},\ldots,\chi_{p,t})$. The copies occupy residual-stream
coordinates and require no additional recurrent state.

The width-two convolution in layer~2 exposes $y_t$, $\chi_t$, and
$\chi_{t-1}$ in separate coordinates.  Head $i$ uses the bilinear
write $y_t\chi_{i,t-1}$ and the fixed scalar recurrence
\[
  a_{i,t}=\alpha_i a_{i,t-1}+y_t\chi_{i,t-1}.
\]
Each layer-$2$ head has $2R$ state coordinates, of which only one is
active; the others remain zero.
For $t\geq R+1$, the cache identity gives
$\chi_{i,t-1}=\xi_{i,\rho(t)}$.  Reading the updated state with
$\chi_{i,t}$ and recombining the heads with coefficients $1-\alpha_i$
therefore yields
\begin{equation}
  g_t
  =\sum_{i=1}^p(1-\alpha_i)a_{i,t}\chi_{i,t}.
  \label{eq:case2b_selective_output}
\end{equation}
The second-layer output projection writes only to the initially zero
prediction coordinate, which $W_{\mathrm{pred}}$ selects.
In the hard-gate limit this is the Case~1 exponential predictor in each direction
over the matched steps. The first cycle's zero labels keep the accumulator and
prediction zero, even with finite gates.  At $t=R$, the Bayes belief is still
zero; at the first matched write $t=R+1$, the identity for $\chi_{i,R}$ applies.
The update-then-read convention incorporates $y_t$ before predicting $y_{t+1}$.

Choose the decays by
\begin{equation}
  W_i:=\frac{1+\alpha_i}{1-\alpha_i}
  =2\sqrt{\frac{v_iT_\rho}{\lambda_i}}.
  \label{eq:case2b_direction_width}
\end{equation}
The Case~1 error identity then gives, in direction $i$,
\[
  \frac{\lambda_iW_i}{4}+\frac{v_iT_\rho}{W_i}+o(\sqrt{T_\rho})
  =\sqrt{\lambda_iv_iT_\rho}+o(\sqrt{T_\rho}).
\]
Summing the fixed number of directions and subtracting the
$O(\log T_\rho)$ posterior baseline gives $\zeta\sqrt{T_\rho}+o(\sqrt{T_\rho})$ for the hard-gate
construction.

For finite gates, the layer-1 cache follows the soft trajectory above.  The
routed-coordinate error is at most $C_{p,\widetilde\Sigma_\theta,\sigma}RT\delta_T$ in $L^8$, and
the labels have uniformly bounded $L^8$ norm, so H\"older's inequality gives
the same order in $L^4$ for the layer-2 evidence-write error.  Iterating the
accumulator for at most $T$ steps bounds its soft--hard error by
$C_{p,\widetilde\Sigma_\theta,\sigma}RT^2\delta_T$ in $L^4$.  The hard accumulator has $L^4$ norm at
most $C_{p,\widetilde\Sigma_\theta,\sigma}T$ under a crude finite-horizon bound.  A final H\"older
bound in~\eqref{eq:case2b_selective_output}, followed by summation over time,
bounds the increase in cumulative loss by $C_{p,\widetilde\Sigma_\theta,\sigma}RT^5\delta_T$. Taking
$\delta_T\leq(RT)^{-8}$ makes this $o(\sqrt{T_\rho})$, and the strict coherence
margin attains this error with finite sigmoid slopes.

Both layers use bilinear writes and ordinary token--state reads, without
count normalization ($\mathbf c=(1,1)$ in the extended class).
Finally, both layers have $p$ heads of width $2R$, totaling $4pR$ recurrent
coordinates.  Their convolution widths $(1,2)$ require $d_e$ buffer
coordinates, and the embedding above gives the stated ambient width.
\end{proof}

Specializing Lemma~\ref{lem:case2b_selective_fixed_state_floor} to
$D=2$ pairs this construction with a nearly matching lower bound. Within the
two-layer selective SSM class, sublinear regret requires
$\Omega(R/\log R)$ recurrent
coordinates, while the construction above uses
$4pR=O(R)$ recurrent coordinates for fixed $p$ and attains
$\zeta\sqrt{T_\rho}+o(\sqrt{T_\rho})$ regret.  It is therefore optimal in
total recurrent state up to a logarithmic factor.
 
\fi

\end{document}